\documentclass[journal]{IEEEtran} 
\usepackage{amsmath, amsfonts, amsthm, amssymb}

\theoremstyle{plain}
\newtheorem{theorem}{Theorem}
\newtheorem{lemma}{Lemma}

\newtheorem{definition}{Definition}
\newtheorem{assumption}{Assumption}
\usepackage{algorithmic}
\usepackage{array}
\usepackage{textcomp}
\usepackage{stfloats}
\usepackage{url}
\usepackage{verbatim}
\usepackage{graphicx}
\usepackage{cite}
\usepackage{duckuments}
\usepackage{booktabs}
\usepackage{enumitem}
\usepackage[font=footnotesize]{caption}
\usepackage{placeins}
\usepackage{xcolor}
\usepackage[colorlinks=true, linkcolor=blue, citecolor=blue, urlcolor=blue]{hyperref}
\usepackage{balance}
\usepackage{subcaption}
\usepackage[ruled,vlined]{algorithm2e}
\usepackage{titlesec}

\titlespacing*{\subsection}{0pt}{0.6\baselineskip}{0.1\baselineskip}

\DeclareMathOperator{\sech}{sech}
\DeclareMathOperator{\sat}{sat}
\DeclareMathOperator{\diag}{diag}
\DeclareMathOperator{\sgn}{sgn}
\DeclareMathOperator{\tr}{tr}
\begin{document}
% \title{Robust and Agile Flight Control of Quadrotors with Adaptive Unwinding-Free Quaternion-Based Sliding Mode Control}

% \title{{\color{softblue}Adaptive} Unwinding-Free Quaternion-Based\\ Sliding Mode Control of Quadrotors}

\title{Robust and Agile Quadrotor Flight via Adaptive Unwinding-Free Quaternion Sliding Mode Control}

\author{Amin Yazdanshenas, and Reza Faieghi
        % <-this % stops a space
\thanks{The authors are with the Autonomous Vehicles Laboratory, Department of Aerospace Engineering, Toronto Metropolitan University, Toronto, Canada{\tt\footnotesize \{amin.yazdanshenas,reza.faieghi\}@torontomu.ca}.}}

\markboth{}%
{Yazdanshenas and Faieghi: Robust and Agile Quadrotor Flight via Adaptive Unwinding-Free Quaternion Sliding Mode Control}


\maketitle
% \bstctlcite{IEEEexample:BSTcontrol}
% @IEEEtranBSTCTL{IEEEexample:BSTcontrol,
% 	CTLuse_article_number     = "yes",
% 	CTLuse_paper              = "yes",
% 	CTLuse_forced_etal        = "yes",
% 	CTLmax_names_forced_etal  = "2",
% 	CTLnames_show_etal        = "1",
% 	CTLuse_alt_spacing        = "yes",
% 	CTLalt_stretch_factor     = "4",
% 	CTLdash_repeated_names    = "no",
% 	CTLname_format_string     = "{f.~}{vv~}{ll}{, jj}",
% 	CTLname_latex_cmd         = ""
% };

% why what was done was done
% what was done
% what was found
% what was concluded
\begin{abstract}
This paper presents a new adaptive sliding mode control (SMC) framework for quadrotors that achieves robust and agile flight under tight computational constraints. The proposed controller addresses key limitations of prior SMC formulations, including (i) the slow convergence and almost-global stability of $\mathrm{SO(3)}$-based methods, (ii) the oversimplification of rotational dynamics in Euler-based controllers, (iii) the unwinding phenomenon in quaternion-based formulations, and (iv) the gain overgrowth problem in adaptive SMC schemes. Leveraging nonsmooth stability analysis, we provide rigorous global stability proofs for both the nonsmooth attitude sliding dynamics defined on $\mathbb{S}^3$ and the position sliding dynamics. Our controller is computationally efficient and runs reliably on a resource-constrained nano quadrotor, achieving 250 Hz and 500 Hz refresh rates for position and attitude control, respectively. In an extensive set of hardware experiments with over 130 flight trials, the proposed controller consistently outperforms three benchmark methods, demonstrating superior trajectory tracking accuracy and robustness with relatively low control effort. The controller enables aggressive maneuvers such as dynamic throw launches, flip maneuvers, and accelerations exceeding 3g, which is remarkable for a 32-gram nano quadrotor. These results highlight promising potential for real-world applications, particularly in scenarios requiring robust, high-performance flight control under significant external disturbances and tight computational constraints.

\end{abstract}

% \input{100Template}
% \input{Amin_intro}
% \input{01Introduction}
% \vspace{-2.5\baselineskip} 
\section*{Multimedia Material}\label{se:multimedia}
The experimental codes and videos related to this paper are accessible at the following links:\\
Code:~\href{https://github.com/AVL-TMU/CrazyAQSMC}{https://github.com/AVL-TMU/CrazyAQSMC}\\
Video:~\href{https://youtu.be/yhFYXKonTRk}{https://youtu.be/yhFYXKonTRk}
\vspace{-0.25\baselineskip} 
\section{Introduction}\label{se:intro}
\subsection{Motivation}
Quadrotors require robust control to maintain stability and precise maneuverability under disturbances and uncertainties.
One widely studied method in this context is sliding mode control (SMC).
However, despite several decades of research, state-of-the-art SMC approaches for quadrotor control still face major limitations.

One key challenge involves attitude control. As discussed in Section \ref{se:related_work}, coordinate-free methods exhibit slow convergence and provide only almost global stability. Meanwhile, Euler-angle-based strategies rely on simplified rotational dynamics, which degrade performance at large roll or pitch angles. Quaternion-based methods also face the unwinding issue, which can cause unnecessarily prolonged rotations.

A second challenge is the need to know the upper bounds of uncertainties. Adaptive switching gains eliminate the need for prior knowledge of these bounds. However, as explained in Section \ref{se:related_work}, the existing adaptive gains grow excessively over time, leading to actuator saturation, significant fluctuations in the control signal, and instability.

SMC, and nonlinear controllers in general, rely heavily on empirical tuning. Most studies claiming the superiority of one method over another base their conclusions on a single or limited set of parameter values, making the results highly dependent on the tuning process. Sensitivity analysis can mitigate this confounding factor by examining how performance varies across different tunings, enabling more robust and reliable conclusions about controller effectiveness.

\subsection{Contributions}
We present a new sliding mode controller that addresses \textit{all major limitations} of prior SMC approaches. 

Our controller demonstrates robust and agile quadrotor control at a level \textit{rarely achieved} by classical nonlinear methods. 
It consistently outperforms benchmark methods, not just under a single tuning, but systematically across multiple scenarios through a \textit{sensitivity analysis involving randomly varied tunings over 130 flight trials}.

Our adaptive gain mechanism responds in real time to varying levels of disturbances and tracking errors. It enables the system to \textit{withstand strong wind disturbances, execute flip maneuvers, and handle accelerations exceeding $3g$ on a nano quadrotor weighing only 32 $[\text{g}]$}, vulnerable to aggressive maneuvers due to its low mass.

The controller's low computational load enables implementation on a limited micro-controller of nano quadrotors, reliably running position and attitude control loops in 250 Hz, and 500 Hz, respectively, even with six adaptive switching gains that update in every computing frame.

From theorerical perspective, the controller presents the following advantages:
\begin{enumerate}[label=\roman*.]
    \item Provides global stability compared to the almost global stability of coordinate-free approaches.
    \item Eliminates the need for simplification of rotational dynamics in Euler-based approaches.
    \item Resolves the unwinding issue of existing quaternion-based SMC techniques.
    \item Incorporates adaptive gains that can increase/decrease over time, preventing their excessive growth in the previous method, which can cause instability.
\end{enumerate}

% \textcolor{softblue}{In addition to these results, we utilize sensitivity analysis to ensure a fair and systematic comparison of various benchmark control strategies, highlighting the practical trade-offs of quaternion-, Euler-, and SO(3)-based representations. These contributions are validated through detailed experimental studies, including attitude control using a custom-built gimbal, free-flight trajectory tracking, and aggressive throw launches.}

% This paper is an \textit{evolved} version of our earlier work \cite{yazdanshenas2024quaternion}. In that initial study, we introduced an unwinding-free quaternion-based sliding mode controller without integrating adaptive gains, providing a formal stability analysis, or performing hardware experimentation.
% Here, we extend those results by incorporating adaptive gains and providing a Lyapunov-based stability analysis. We also present comprehensive hardware implementation results, evaluating both the original SMC and the newly developed adaptive SMC in real flight scenarios, and perform extensive sensitivity analyses against benchmark methods using hardware experiments.
This paper is an evolved version of our earlier work \cite{yazdanshenas2024quaternion}, which introduced an unwinding-free quaternion-based sliding mode controller without adaptive gains, formal stability analysis, or hardware validation. Here, we extend that work by incorporating adaptive gains, providing stability analysis, and presenting comprehensive hardware results. 
% \vspace{-0.7\baselineskip} 
\section{Related Work}\label{se:related_work}
% A vast body of literature exists on quadrotor control, with recent comprehensive surveys available in \cite{kim2019comprehensive} and \cite{khalid2023control}. In this paper, we focus on the latest nonlinear control techniques and discuss how they relate to our proposed controller.
\subsection{Optimization- and learning-based methods}
Recently, optimization-based methods, particularly nonlinear model predictive control (NMPC)  and its variants \cite{nan2022nonlinear, romero2022model, izadi2024multi}, have proven effective in quadrotor control, especially for agile maneuvers \cite{sun2022comparative}.
NMPC relies on an accurate system model. Disturbances and uncertainties can degrade its performance, leading to sub-optimal or unstable control. The integration of NMPC with SMC \cite{bhattacharjee2020robust} and disturbance observers \cite{li2022enhanced}, and robust formulations like tube-based NMPC \cite{michel2019design} and stochastic NNMPC \cite{xue2024output} have proven effective for robust control.

However, the computational load of NMPC remains a challenge for hardware implementation, especially with long prediction horizons. Most hardware implementations of NMPC have used it for either position control \cite{sun2022comparative, nan2022nonlinear} or attitude control \cite{kamel2015fast}, while relying on simpler controllers, e.g. incremental nonlinear dynamic inversion \cite{tal2020accurate}, dynamic inversion \cite{achtelik2013inversion}, or linear quadratic regulators \cite{bouabdallah2007full}, for the other control loop.

Machine learning (ML) has also led to effective control solutions for quadrotors.
Several studies have explored ML models, particularly reinforcement learning agents, for end-to-end learning of control policies \cite{song2023reaching}, achieving impressive performance in aggressive maneuvers.
Other studies have explored integrating ML models with NMPC \cite{romero2024actor} and classical methods \cite{yogi2024neural}. 
ML techniques can improve the accuracy of vehicle models used in controllers \cite{salzmann2023real} and/or optimize control design parameters \cite{annaswamy2023integration}.

While optimization- and learning-based methods hold strong potential for quadrotor control, classical nonlinear methods, including SMC, remain important for at least two reasons:
(i) Classical approaches are often simpler and more computationally efficient. Despite advances in flight control hardware, quadrotors, especially in the nano and micro classes, still face limited onboard computing.
(ii) Hybrid approaches that integrate optimization- and learning-based methods with classical control combine the strengths of both paradigms \cite{bhattacharjee2020robust, annaswamy2023integration, yogi2024neural}. Continued advances in classical nonlinear control will directly benefit these hybrid frameworks.

\subsection{Classical Nonlinear Control Methods}
Classical nonlinear control methods have a long heritage in quadrotor control. Examples include feedback linearization \cite{voos2009nonlinear}, SMC \cite{xu2006sliding}, backstepping \cite{madani2006backstepping}, adaptive \cite{nicol2011robust}, and passivity-based designs \cite{souza2014passivity}.  Arguably, the most common technique in this category is the geometric tracking controller (GTC) introduced by Lee et al. \cite{lee2010geometric}, which is a cascaded Lyapunov-based controller design on $\mathrm{SE(3)}$. GTC has been the basis for several quadrotor control techniques that we summarize under geometric methods.

\subsubsection{Geometric Methods} 
This class of controllers govern the vehicle’s dynamics on the $\mathrm{SE(3)}$ manifold. They employ a coordinate-free representation of quadrotor attitude using the rotation matrix $\mathbf{R}$ in the form of $\mathbf{R}_e = \mathbf{R}_d^\top  \mathbf{R}$ where the indices $e$ and $d$ represent error and desired values, respectively. This framework has underpinned several quadrotor controllers using SMC \cite{garcia2020robust}, backstepping \cite{lee2013backstepping}, adaptive control \cite{lee2012robust}, and other Lyapunov-based robust designs \cite{lee2013nonlinear, parra2012toward}. It has also been fundamental to SMC-based design for sapcecraft\cite{gong2020adaptive,ren2023adaptive}, and attitude control of rigid bodies, in general \cite{lee2012exponential}. The two challenges with GTC, which also apply to its SMC extension \cite{garcia2020robust}, include the slow convergence of attitude error, as pointed out in \cite{teng2022lie, lopez2020sliding} and its almost global stability \cite{bhat2000topological}. However, it continues to remain a popular controller, implemented on common platforms like PX4 Autopilot \cite{d2024efficient} and Crazyflie \cite{crazyflie_lee_controller}.

\subsubsection{SMC with Euler angle-based rotational dynamics}
One of the earliest sliding mode controllers for quadrotors is \cite{xu2006sliding}, which is based on the Euler angles representation of rotational dynamics. Numerous studies have since extended this controller in various directions, as explained below.

(i) The first direction focuses on enhancing trajectory tracking performance through variations of SMC. For example, \cite{zheng2014second} develops integral SMC for quadrotors by adding an integral term to the sliding surface, accelerating the initial convergence of vehicle states to the sliding surface. \cite{xiong2017global} introduces terminal SMC (TSMC), incorporating a nonlinear term to achieve a faster convergence compared to conventional SMC. However, TSMC relies on non-Lipschitz functions, which can cause singularities, motivating the development of non-singular TSMC (NTSMC) \cite{hou2020nonsingular, hassani2021robust}. Recent advancements include fixed-time NTSMC \cite{mechali2021observer, wang2020fixed, ai2019fixed, yu2021novel}, which guarantees convergence within a predefined time, regardless of initial conditions.

(ii) The second direction focuses on reducing chattering that arises from the discontinuity of the sign function $\sgn(\cdot)$ used in standard SMC law. The simplest solution is to replace $\sgn(\cdot)$ with a saturation function $\operatorname{sat}(\cdot)$, which mitigates chattering by introducing a thin boundary layer, but at the cost of small tracking errors, as shown in \cite{xu2006sliding}. Another alternative is the $\tanh(\cdot)$, as explored in \cite{noordin2021sliding, nguyen2021adaptive, noordin2022position}. Compared to the $\operatorname{sat}(\cdot)$, $\tanh(\cdot)$ avoids abrupt changes near the boundary layer and preserves the differentiability of the control input without significantly increasing control complexity or computational overhead. Other strategies include using fuzzy logic to create smooth switching functions \cite{zare2022quadrotor, jing2019quadrotor} and higher-order SMC (HOSMC) \cite{munoz2017second, chandra2022higher}, which extends control to the derivative of the sliding surface; however, these methods introduce greater complexity and increase computational load.

(iii) The third direction focuses on the integration of SMC with adaptive control. 
Some studies employ adaptation laws to estimate vehicle parameters in real-time, updating the control law accordingly \cite{bouadi2011adaptive, huang2019robust, mofid2018adaptive}. A different line of research leverages universal approximators, such as neural networks and fuzzy systems \cite{yogi2024neural, nekoukar2021robust}, to approximate unknown system dynamics and improve control accuracy. The above approaches lead to complex controllers.
An alternative is to use the partial knowledge of the quadrotor model in the control law and employ adaptive switching gains that can dynamically adjust to the upper bounds of model uncertainties and external disturbances. This approach has been applied to boundary-layer SMC \cite{nadda2018adaptive}, TSMC \cite{lian2021adaptive}, and second-order SMC \cite{thanh2018quadcopter}. However, in all these studies, the adaptation laws have a structure similar to
\begin{equation}\label{eq:gain_saturation}
\dot{k} = \gamma \lvert s \rvert
\end{equation}
where $k$ is the switching gain, $\gamma > 0$ is the learning rate, and $s$ is the sliding surface. The issue with this formulation is that for all $s \neq 0$, $\dot{k} > 0$, which leads to the excessive growth of $k$ over time, making the control law impractical.

It is important to note that all the aforementioned SMC studies rely on a simplified model of rotational dynamics that is only valid for small roll and pitch angles. 
To elaborate, let ${\boldsymbol{\eta}}=\left[\varphi, \theta,\psi\right]^\top $, where $ - \pi  < \varphi  \le \pi $, $ - \frac{\pi }{2} \le \theta  \le \frac{\pi }{2}$, and $ - \pi  < \psi  \le \pi $ be the Euler angles representing pitch, roll, and yaw in the yaw-pitch-roll sequence,  ${\boldsymbol{\omega}}$ the angular velocity vector, ${\rm{\mathbf{J}}}=\diag\left(J_x, J_y, J_z \right)$ the inertia matrix, and $\boldsymbol{\tau}$ the moments around the principal axes.
Then, the quadrotor attitude dynamics takes the following form
\begin{equation}\label{eq:EulerKinematicalEquation}
    \dot {\boldsymbol{\eta}} = {\bf{H}}\left({\boldsymbol{\eta}}\right) {\boldsymbol{\omega}},
\end{equation}
where  
\begin{equation}\label{eq:H}
{\bf{H}}\left( {\boldsymbol{\eta }} \right) = \left[ {\begin{array}{*{20}{c}}
1&{\sin \varphi \tan \theta }&{\cos \varphi \tan \theta }\\
0&{\cos \varphi }&{ - \sin \varphi }\\
0&{{{\sin \varphi } \mathord{\left/
 {\vphantom {{\sin \varphi } {\cos \theta }}} \right.
 \kern-\nulldelimiterspace} {\cos \theta }}}&{{{\cos \varphi } \mathord{\left/
 {\vphantom {{\cos \varphi } {\cos \theta }}} \right.
 \kern-\nulldelimiterspace} {\cos \theta }}}
\end{array}} \right],
\end{equation}
and
\begin{equation}\label{eq:rotDyn}
    \dot{\boldsymbol{\omega}} = {\bf{J}}^{ - 1}\left( -{\boldsymbol{\omega}} \times {\bf{J}}{\boldsymbol{\omega}}+ \boldsymbol{\tau} \right),
\end{equation}
where $\times$ indicates the cross product.
Now, if $\varphi$ and $\theta$ are small, $\bf{H}$ can be greatly simplified such that $\dot{\boldsymbol{\eta}} \approx {\boldsymbol{\omega}}$. Subsequently, \eqref{eq:EulerKinematicalEquation}-\eqref{eq:rotDyn} can be simplified to a set of second-order differential equations
\begin{equation}\label{eq:simplified_rot_dyn}
\left\{ {\begin{array}{*{20}{l}}
{\ddot \varphi  = \dot \theta \dot \psi (\frac{{{J_y} - {J_z}}}{{{J_x}}})   + \frac{1}{{{J_x}}}{\tau_1}},\\
{\ddot \theta  = \dot \varphi \dot \psi (\frac{{{J_z} - {J_x}}}{{{J_y}}})   + \frac{1}{{{J_y}}}{\tau_2}},\\
{\ddot \psi  = \dot \varphi \dot \theta (\frac{{{J_x} - {J_y}}}{{{J_z}}}) + \frac{1}{{{J_z}}}{\tau_3}}.\\
\end{array}} \right.
\end{equation}
The above equations have been the basis of all the aforementioned SMC papers and many others in the existing literature.
While the particular second-order structure of \eqref{eq:simplified_rot_dyn} lends well to the SMC formulations, it is only valid for small $\varphi$ and $\theta$ values. 
Thus, as we will show in real flight experiments, the controller performance degrades with the increase in $\varphi$ or $\theta$.

Many of the above studies have only demonstrated simulation results, and those that explored hardware implementations \cite{munoz2017second, yu2021novel, rios2018continuous, shao2021adaptive, lian2021adaptive, yogi2024neural, nekoukar2021robust} do not address this simplification in their experiments. 
In this paper, we present attitude control experiments on a gimbal that will highlight this limitation of Euler-based approaches.

\subsubsection{SMC with quaternion-based rotational dynamics}
While quaternion-based controllers for quadrotors are abundant, sliding mode controllers in this framework remain relatively under-explored.
Relevant studies include integral SMC \cite{sanchez2013time}, TSMC \cite{serrano2023terminal}, and HOSMC \cite{arellano2015quaternion}.
Additionally, \cite{abaunza2019quadrotor} presents a quaternion-based SMC law for attitude stabilization of quadrotors during throw launches.

There exist several gaps in quaternion-based SMC approaches:
(i) As will be explained in Section \ref{se:control_design}, the existing controllers are prone to unwinding, which arises from the ambiguity of quaternions and can lead to longer-than-necessary rotations.
(ii) No adaptive gain mechanisms have been developed for these methods, meaning they require prior knowledge of the upper bounds of uncertainties and disturbances.
This complicates real-world implementation and tuning, especially in dynamic environments with significant disturbances.
(iii) There is a lack of comprehensive hardware experimentation and validation for these controllers. Of the mentioned studies, only \cite{abaunza2019quadrotor} includes hardware implementations, and those experiments focus solely on attitude stabilization.

It becomes evident that there is a need for an adaptive, unwinding-free quaternion-based sliding mode controller for quadrotors, whose effectiveness is experimentally demonstrated and verified. Our contributions directly address this gap. Focusing on simplicity and practicality, we develop a new SMC framework in its most basic form, which can be implemented in real-time even on the limited computing hardware of nano quadrotors. Moreover, this controller can serve as a foundation for more advanced SMC formulations, e.g. TSMC, or be integrated with optimization- and learning-based methods, e.g. learning-based SMC.
% \vspace{-1.0\baselineskip} 
\section{Preliminaries}
\subsection{Notations} 
Throughout this paper, unless stated otherwise, we use the following notation standards.

Italic letters indicate scalars, lowercase bold letters represent vectors, and uppercase bold letters denote matrices. 
For a vector $\mathbf{v}$, $v_i$ refers to its $i$-th element.
For a matrix $\mathbf{M}$, $M_{ij}$ denotes the element in the $i$-th row and $j$-th column. 
$\lambda_{\max}(\mathbf{M})$ denotes the largest eigenvalue of $\mathbf{M}$.
$\tr(\mathbf{M})$ denotes the trace of matrix $\mathbf{M}$, defined as the sum of its diagonal elements.

We denote the 2-norm of a vector $\mathbf{v}$ as $\|\mathbf{v}\|$, the absolute value of $\mathbf{v}$ as $\lvert\mathbf{v}\rvert$, and the $s$-th Hamard power of $\mathbf{v}$ as $\mathbf{v}^{\circ s}$.
We represent the dot product between two vectors $\mathbf{u}$ and $\mathbf{v}$ as $\mathbf{u} \cdot \mathbf{v}$ or $\langle \mathbf{u}, \mathbf{v} \rangle$, their cross product as $\mathbf{u} \times \mathbf{v}$, their element-wise product as $\mathbf{u} \circ \mathbf{v}$, and their element-wise division as $\mathbf{u} \oslash \mathbf{v}$.

The two coordinate frames used in this paper include the inertial frame 
$\mathcal{F}_e = \left\{\mathbf{e}_1, \mathbf{e}_2, \mathbf{e}_3\right\}$, and the body frame $\mathcal{F}_b = \left\{\mathbf{b}_1, \mathbf{b}_2, \mathbf{b}_3 \right\}$, illustrated in Fig. \ref{fig:quadrotor_model}. 
The rotation matrix $\mathbf{R} \in \mathrm{SO}(3)$ describes the rotation from $\mathcal{F}_b$ to $\mathcal{F}_e$.
$\mathrm{SO}(3)$ refers to the special orthogonal group in three dimensions, with its associated Lie algebra $\mathfrak{so}(3)$. 
The tangent space at a point $\mathbf{R} \in \mathrm{SO}(3)$ is denoted by $T_{\mathbf{R}} \mathrm{SO}(3)$, and consists of matrices of the form $\mathbf{R} \boldsymbol{\omega}^\times$ for any $\boldsymbol{\omega} \in \mathbb{R}^3$, where $(\cdot)^\times : \mathbb{R}^3 \to \mathfrak{so}(3)$ maps a vector to a skew-symmetric matrix.
The inverse mapping $(\cdot)^\vee : \mathfrak{so}(3) \to \mathbb{R}^3$ retrieves the vector representation from a skew-symmetric matrix.

% The operator $\cdot ^\vee$ maps elements from $\mathfrak{so}(3)$ to $\mathbb{R}^3$, while $\cdot^\times$ maps elements from $\mathbb{R}^3$ to $\mathfrak{so}(3)$.

We represent unit quaternions as $\mathbf{q} = \left[q_w, \vec{\mathbf{q}}^\top \right]^\top  \in \mathbb{S}^3$, where $\mathbb{S}^3$ refers to 3-dimensional unit sphere.
We denote the conjugate of unit quaternion $\mathbf{q}$ by
$\mathbf{q}^*$, and the multiplication of two unit quaternions $\mathbf{q}$ and $\mathbf{p}$ by $\mathbf{q} \otimes \mathbf{p}$. 
For a unit quaternion $\mathbf{q}$, we define an operator $
L_\mathbf{q}:\mathbb{R}^3 \rightarrow \mathbb{R}^3$ as follows 
\begin{equation}\label{eq:L}
    L_\mathbf{q}(\mathbf{v}) = (q_w^2-\|\vec{\mathbf{q}}\|^2)\mathbf{v} + 2(\vec{\mathbf{q}}\cdot\mathbf{v})\vec{\mathbf{q}}+2 q_w (\vec{\mathbf{q}} \times \mathbf{v}).
\end{equation}
It is easy to verify that $L_\mathbf{q}$ is homogeneous, i.e. $L_\mathbf{q}(s\mathbf{v}) = sL_\mathbf{q}(\mathbf{v})$, and isometric, i.e. $\|L_\mathbf{q}(\mathbf{v})\|=\|\mathbf{v}\|$.

We occasionally use Euler angles in this paper, denoted as ${\boldsymbol{\eta}}=\left[\varphi, \theta,\psi\right]^\top $, where $ - \pi  < \varphi  \le \pi $, $ - \frac{\pi }{2} \le \theta  \le \frac{\pi }{2}$, and $ - \pi  < \psi  \le \pi $. These angles represent roll $\varphi$, pitch $\theta$, and yaw $\psi$ in the yaw-pitch-roll sequence.

The notation $f \in C^k$ denotes that function $f$ is $k$ times continuously differentiable.

We use the subscripts $d$ to indicate desired values, and the subscript $e$ to indicate errors. The hat symbol $\hat{\cdot}$ identifies nominal values, and the tilde symbol $\tilde{\cdot}$ identifies uncertain terms.
The bar and underline symbols, $\bar{\cdot}$ and $\underline{\cdot}$ represent the upper and lower bounds of uncertain terms.

\subsection{Quadrotor dynamics}\label{sec:Dynamics}
Let us consider the quadrotor model and coordinate frames $\mathcal{F}_e$ and $\mathcal{F}_b$ presented in Fig. \ref{fig:quadrotor_model}.
Let $\boldsymbol{\xi} = \left[x, y, z\right]^\top $ represent the position, $\boldsymbol{\nu} = \left[u,v,w\right]$ the velocity, $\mathbf{q} = \left[q_{w}, \Vec{\mathbf{q}}^\top  \right]^\top $ the attitude, and $\boldsymbol{\omega} = \left[P,Q,R\right]^\top $ the angular velocity. 
Then, the quadrotor flight dynamics takes the following form 
\begin{equation}\label{eq:quad_dyn}
    \left\{
\begin{array}{l}
    \dot {\boldsymbol{\xi}} = \boldsymbol{\nu} \\
    \dot{\boldsymbol{\nu}}= -g\mathbf{e}_{3} + \frac{1}{m} L_\mathbf{q}(f \mathbf{e}_{3})+ \mathbf{d}_a \\
    \dot{\mathbf{q}} = \frac{1}{2}\mathbf{q} \otimes \left[0, \boldsymbol{\omega}^\top \right]^\top \\
    \dot{\boldsymbol{\omega}} = \mathbf{J}^{-1} \left(-\boldsymbol{\omega} \times \mathbf{J}\boldsymbol{\omega}  + \boldsymbol{\tau} \right) + \mathbf{d}_{\alpha}
\end{array}
\right.
\end{equation}
where $g$ is the gravity, $m$ is the mass, $\mathbf{J}$ is the inertia matrix, $f$ is the thrust, $\boldsymbol{\tau}$ is the moments, and $\mathbf{d}_a$ and $\mathbf{d}_\alpha$ represent external disturbances and unmodeled dynamics.

Let $\mathbf{u} \in \mathbb{R}^4$ be the 
control input expressed by
\begin{equation}
    \mathbf{u} = c_t \boldsymbol{\varpi}^{\circ2},
\end{equation}
where $\boldsymbol{\varpi}$ is the vector encompassing the angular speed of rotors, and $c_t$ is the rotors' thrust coefficient.
Then, the control input maps to the force and moments applied to the quadrotor through
\begin{equation}\label{eq:G_1}
    \begin{bmatrix}
        f\\ \boldsymbol{\tau}
    \end{bmatrix}
    = \mathbf{G}\mathbf{u}
    % +\mathbf{G}_2\dot{\boldsymbol{\varpi}}+\mathbf{G}_3\boldsymbol{\varpi},
\end{equation}
where
\begin{equation}
    \mathbf{G} = \begin{bmatrix}
        1&1&1&1\\
        l\;\sin(\beta) & -l\;\sin(\beta)&-l\;\sin(\beta) & l\;\sin(\beta)\\
        -l\;\cos(\beta) & l\;\cos(\beta)&-l\;\cos(\beta) & l\;\sin(\beta)\\
        -c_q/c_t &-c_q/c_t &c_q/c_t &c_q/c_t 
    \end{bmatrix}.
\end{equation}
% \begin{equation}
%     \mathbf{G}_2 = \begin{bmatrix}
%         0&0&0&0\\
%         0&0&0&0\\
%         0&0&0&0\\
%         -J_r &-J_r &J_r &J_r 
%     \end{bmatrix},
% \end{equation}
% \begin{equation}
%     \mathbf{G}_3 = \begin{bmatrix}
%         0&0&0&0\\
%         -J_r Q &-J_rQ &J_rQ &J_rQ \\
%         J_r P &J_rP &-J_rP &-J_rP \\
%         0&0&0&0 
%     \end{bmatrix},
% \end{equation}
with $c_q$ being the rotors' torque coefficient, and $\beta$ and $l$  geometric parameters defined in Fig. \ref{fig:quadrotor_model}.
% , $J_r$ the inertia of the rotors along their thrust directions,

\begin{figure}
    \centering
    \includegraphics[trim=0cm 1cm 12cm 0cm,clip,width=0.7\linewidth]{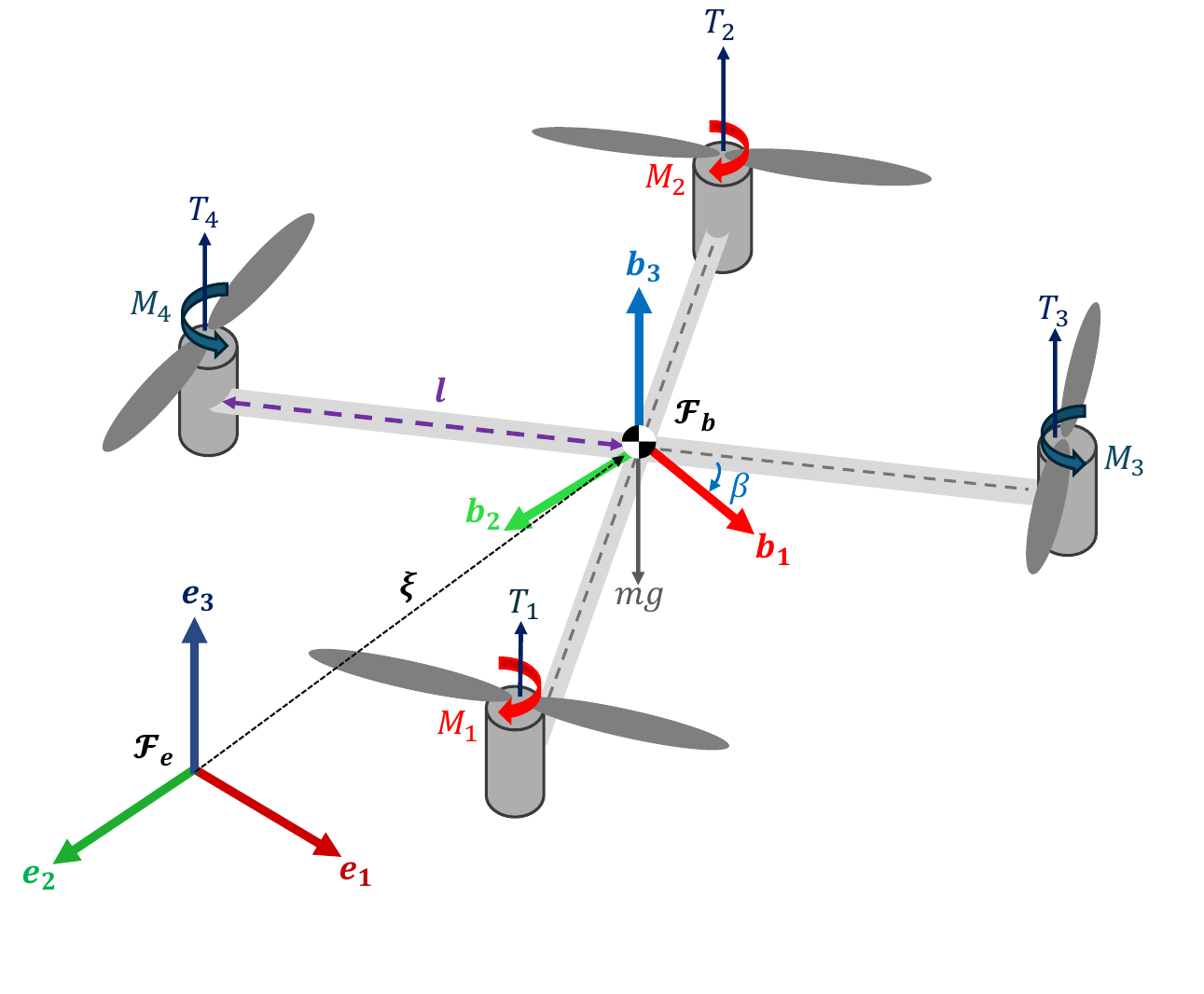}
    \caption{Quadrotor model and coordinate frames.}
    \label{fig:quadrotor_model}
\end{figure}

\subsection{Nonsmooth Stability Analysis}\label{se:nonsmooth_preliminaries}
To achieve unwinding-free performance, we rely on a sliding surface that is non-differentiable on a subset of $\mathbb{S}^3$, specifically where the scalar part 
of the error quaternion $q_{w_e}$ is zero.
Consequently, classical tools of Lyapunov stability theory, which rely on smoothness assumptions, are not directly applicable. 
Therefore, we adopt tools from nonsmooth analysis as explained below. 
% To achieve unwinding-free performance, we employ a particular sliding surface that is non-differentiable on a subset of $\mathbb{S}^3$. Consequently, classical tools of Lyapunov stability theory, which rely on smoothness assumptions, are not directly applicable. 
% To rigorously analyze stability under these conditions, we adopt tools from nonsmooth analysis,
% including Clarke's generalized gradient and a nonsmooth extension of Lyapunov stability theory~\cite{clarke2008nonsmooth}.

\begin{definition}[Measure-zero set \cite{clarke2008nonsmooth}]
A set $ \mathcal{A} \subset \mathbb{R}^n $ is said to have measure zero if, for every $ \varepsilon > 0 $, there exists a countable collection of open balls $ \{\mathcal{B}_i\} $ such that
\begin{equation}
 \mathcal{A} \subset \bigcup_i \mathcal{B}_i \quad \text{and} \quad \sum_i \mathrm{vol}(\mathcal{B}_i) < \varepsilon,   
\end{equation}
where $\mathrm{vol}(\mathcal{B}_i)$ denotes the volume of ball $\mathcal{B}_i$.
\end{definition}

% \begin{definition}[Clarke generalized gradient \cite{clarke2008nonsmooth}]
% Let $ \mathbf{f}: \mathbb{R}^n \to \mathbb{R}^m $ be locally Lipschitz function. The Clarke generalized gradient of $ \mathbf{f} $ at a point $ \mathbf{x} \in \mathbb{R}^n $, denoted $ \partial_C \mathbf{f}(\mathbf{x}) $, is defined as
% \begin{equation}
% \partial_C \mathbf{f}(\mathbf{x}) = \operatorname{co}\left\{ \lim_{i \to \infty} \nabla \mathbf{f}(\mathbf{x}_i) \;\middle|\; \mathbf{x}_i \to \mathbf{x},\; \mathbf{x}_i \notin \mathcal{N} \right\},
% \end{equation}
% where $\operatorname{co}$ denotes the convex hull, and \( \mathcal{N} \subset \mathbb{R}^n \) is any set of measure zero.
% \end{definition}
\begin{definition}[Clarke generalized gradient \cite{clarke2008nonsmooth}]
Let \( \mathbf{f}: \mathbb{R}^n \to \mathbb{R}^m \) be a locally Lipschitz function. The Clarke generalized gradient of \( \mathbf{f} \) at a point \( \mathbf{x} \in \mathbb{R}^n \), denoted \( \partial_C \mathbf{f}(\mathbf{x}) \), is defined as
\begin{equation}
\partial_C \mathbf{f}(\mathbf{x}) = \operatorname{co}\left\{ \lim_{i \to \infty} \nabla \mathbf{f}(\mathbf{x}_i) \;\middle|\; \mathbf{x}_i \to \mathbf{x},\; \mathbf{x}_i \notin \mathcal{N} \right\},
\end{equation}
where \( \operatorname{co} \) denotes the convex hull, and \( \mathcal{N} \subset \mathbb{R}^n \) is any measure-zero set. The index \( i \) traces a sequence \( \{\mathbf{x}_i\} \) of differentiability points of \( \mathbf{f} \), approaching \( \mathbf{x} \), with the limit taken over gradients \( \nabla \mathbf{f}(\mathbf{x}_i) \) at those nearby smooth points.
\end{definition}
Intuitively, if $\mathbf{f}$ is differentiable at $\mathbf{x}$ and continuously differentiable in a neighborhood of $\mathbf{x}$, then the Clarke generalized gradient reduces to the singleton $\partial_C \mathbf{f}(\mathbf{x}) = \{ \nabla \mathbf{f}(\mathbf{x}) \}$.
If it is not differentiable at $\mathbf{x}$, then $\partial_C \mathbf{f}(\mathbf{x})$ contains all limits of gradients at nearby differentiable points. The convex hull ensures that all such limiting gradients and their convex combinations are included, capturing all possible directional derivatives $\mathbf{f}$ can exhibit at $\mathbf{x}$.

The Clarke generalized gradient retains several properties of the classical gradient, such as linearity, the sum rule, and a generalized chain rule, which is particularly useful for Lyapunov stability analysis and is presented below.

\begin{lemma}[Clarke gradient chain rule \cite{clarke2008nonsmooth}]
Let $ f: \mathbb{R}^m \to \mathbb{R} $ be continuously differentiable and $ \mathbf{g}: \mathbb{R}^n \to \mathbb{R}^m $ be locally Lipschitz. Then the Clarke generalized gradient of the composition $ f \circ \mathbf{g} $ at a point $ \mathbf{x} \in \mathbb{R}^n $ satisfies
\begin{equation}\label{eq:c_chain_rule}
\partial_C (f \circ \mathbf{g})(\mathbf{x}) \subseteq \left\{ \langle \nabla  f(\mathbf{g}(\mathbf{x})), \boldsymbol{\iota} \rangle  \;\middle|\; \boldsymbol{\iota} \in \partial_C \mathbf{g}(\mathbf{x}) \right\}
\end{equation}
\end{lemma}

To interpret \eqref{eq:c_chain_rule}, recall the classic chain rule of smooth functions
$\nabla(f \circ \mathbf{g})(\mathbf{x}) = \nabla \mathbf{g}(\mathbf{x})^\top \nabla f(\mathbf{g}(\mathbf{x})) = \left\langle \nabla f(\mathbf{g}(\mathbf{x})), \nabla \mathbf{g}(\mathbf{x}) \right\rangle$.
When $\mathbf{g}$ is nonsmooth, the same principle applies, but since the directional derivative may not be unique, we account for all possible limiting gradients, leading to a set-valued generalized gradient, where the inner product is computed between $\nabla f(\mathbf{g}(\mathbf{x}))$ and each element of $\partial_C \mathbf{g}(\mathbf{x})$. This property becomes useful in our analysis where we have a Lyapunov function of a nonsmooth sliding surface.

\begin{theorem} [Nonsmooth Lyapunov stability \cite{shevitz1994lyapunov}]
Consider the system \( \dot{\mathbf{x}} = \mathbf{f}(\mathbf{x}) \), where \( \mathbf{f}: \mathbb{R}^n \to \mathbb{R}^n \) is locally Lipschitz and \( \mathbf{f}(\mathbf{x}^\star) = \mathbf{0} \). Let \( V: \mathbb{R}^n \to \mathbb{R}_{\ge0} \) be a locally Lipschitz, positive definite function. Define the generalized directional derivative of \( V \) as
\begin{equation}\label{eq:nonsmooth_vdot_preliminaries}
    \dot{V}^{\circ}(\mathbf{x}) = \max \{   \langle \boldsymbol{\varsigma}, \mathbf{f}(\mathbf{x}) \rangle :{\boldsymbol{\varsigma} \in \partial_C V(\mathbf{x})}\}.
\end{equation}
If \( \dot{V}^{\circ}(\mathbf{x}) \leq 0 \) for all \( \mathbf{x} \), then $\mathbf{x}^\star$ is Lyapunov stable. If \( \dot{V}^{\circ}(\mathbf{x}) < 0 \) for all \( \mathbf{x} \neq 0 \), then $\mathbf{x}^\star$ is asymptotically stable.
\end{theorem}

The intuition behind Theorem 1 is that when the Lyapunov function \( V \) is nonsmooth, its rate of change along system trajectories is not uniquely defined. Therefore, we compute the Clarke generalized gradient, which captures all possible limiting gradients. The generalized time derivative \( \dot{V}^{\circ}(\mathbf{x}) \) is then defined as the maximum of inner products \( \langle \boldsymbol{\varsigma}, f(\mathbf{x}) \rangle \) over all \( \boldsymbol{\varsigma} \in \partial_C V(\mathbf{x}) \). If even the worst-case rate of change is nonpositive (or strictly negative), \( V \) is guaranteed to decrease (or strictly decrease), implying Lyapunov (or asymptotic) stability of the equilibrium at the origin.

There is a direct analogue between stability results for smooth and nonsmooth systems \cite{shevitz1994lyapunov, bacciotti1999stability}. As stated in \cite{shevitz1994lyapunov}, the proofs are identical to their smooth counterparts except that some relations hold \textit{almost everywhere}, due to measure-zero sets, instead of everywhere.

\section{Problem Statement}
Let us begin with defining the tracking error terms as follows
% \begin{equation}\label{eq:error_definition}
% \begin{cases}
%    \boldsymbol{\xi}_e = \boldsymbol{\xi} -\boldsymbol{\xi}_d, \\
%    \boldsymbol{\nu}_e = \boldsymbol{\nu} -\boldsymbol{\nu}_d, \\
%    \mathbf{q}_e = \mathbf{q}_d^* \otimes \mathbf{q},\\
%     \boldsymbol{\omega}_e = \boldsymbol{\omega} - \boldsymbol{\omega}_d.
% \end{cases}
% \end{equation}
\begin{equation}\label{eq:error_definition}
\begin{aligned}
\boldsymbol{\xi}_e &= \boldsymbol{\xi} - \boldsymbol{\xi}_d, \quad
\boldsymbol{\nu}_e = \boldsymbol{\nu} - \boldsymbol{\nu}_d, \\
\mathbf{q}_e &= \mathbf{q}_d^* \otimes \mathbf{q}, \quad
\boldsymbol{\omega}_e = \boldsymbol{\omega} - \boldsymbol{\omega}_d.
\end{aligned}
\end{equation}

Using the quadrotor dynamics \eqref{eq:quad_dyn}, the error dynamics becomes
\begin{equation}\label{eq:err_dyn}
\begin{cases}
   \dot{\boldsymbol{\xi}}_e = \boldsymbol{\nu} - \dot{\boldsymbol{\xi}}_d, \\
   \dot{\boldsymbol{\nu}}_e     =  -g\mathbf{e}_{3} + \frac{1}{m} L_\mathbf{q}(f \mathbf{e}_{3}) + {\mathbf{d}}_a -\ddot{\boldsymbol{\xi}}_d, \\
      \dot{\mathbf{q}}_e =  \frac{1}{2}\left[
            -\Vec{\mathbf{q}}_e \cdot \boldsymbol{\omega}_e, \left(q_{w_e}\boldsymbol{\omega}_e+\Vec{\mathbf{q}}_e \times \boldsymbol{\omega}_e\right)^\top 
        \right]^\top ,
   \\
    \dot{\boldsymbol{\omega}}_e = \mathbf{J}^{-1} \left(-\boldsymbol{\omega} \times \mathbf{J}\boldsymbol{\omega}  + \boldsymbol{\tau} \right)+ {\mathbf{d}}_{\alpha}  - \dot{\boldsymbol{\omega}}_d.
\end{cases}
\end{equation}

\noindent \textbf{Problem Statement:} \textit{Determine $f$ and $\boldsymbol{\tau}$ such that the tracking errors in~\eqref{eq:error_definition} are uniformly ultimately bounded within an arbitrarily small neighborhood of the origin. The control law must ensure unwinding-free attitude tracking and incorporate adaptive gains that can both increase and decrease in response to the system’s behavior, given the following assumptions.}

\begin{assumption}\label{as:trajectory}
The desired trajectory is given by the desired position $\boldsymbol{\xi}_d(t) \in C^4$ and the desired yaw angle $\psi_d(t) \in C^2$. That is, there exist positive constants $B_{\xi_i}$ and $B_{\psi_i}$ such that
\begin{equation}
\begin{aligned}
    \big\| \boldsymbol{\xi}_d^{(i)}(t) \big\| &\le B_{\xi_i}, \quad i = 0,\dots,4, \\
    \big| \psi_d^{(i)}(t) \big| &\le B_{\psi_i}, \quad i = 0,1,2.
\end{aligned}
\end{equation}

% The desired trajectory is given by the desired position $\boldsymbol{\xi}_d(t) $ and the desired heading angle $\psi_d(t)$.
% All derivatives up to order four for $\boldsymbol{\xi}_d$ and up to order three for $\psi_d$ are uniformly bounded.
\end{assumption}

\begin{assumption}\label{as:model}
$ m $ and $ \mathbf{J} $ are partially known, i.e., $ m = \hat{m} + \tilde{m} $ and $ \mathbf{J} = \hat{\mathbf{J}} + \tilde{\mathbf{J}} $, where $ \hat{m} $ and $ \hat{\mathbf{J}} $ are known nominal values, and $\tilde{m}$ and $\tilde{\mathbf{J}}$ are unknown but bounded.
\end{assumption}

\begin{assumption}\label{as:disturbance}
$ \mathbf{d}_a $ and $ \mathbf{d}_\alpha$ are unknown but bounded. That is, there exist two vectors $\bar{\mathbf{d}}_a$ and $\bar{\mathbf{d}}_\alpha$ with positive elements such that $ |\mathbf{d}_a| \le \bar{\mathbf{d}}_a $ and $ |\mathbf{d}_\alpha| \le \bar{\mathbf{d}}_\alpha $. The absolute value $ |\cdot| $ is interpreted element-wise.
\end{assumption}

\begin{assumption}\label{as:known_bounds}
    $ \bar{\mathbf{d}}_a $ and $ \bar{\mathbf{d}}_\alpha $ are known. (This assumption is not required for our adaptive control results.)
\end{assumption}
We break down the above problem into three parts, and present three theorems in the subsequent sections to establish our results, as follows:
\begin{enumerate}[label=\roman*.]
    \item First, we deal with attitude dynamics, presenting the first unwinding-free quaternion-based SMC law whose stability is established via nonsmooth Lyapunov stability analysis, extending the results of \cite{lopez2020sliding}.
    \item Second, we expand our results to the 6-DOF flight, introducing the first quaternion-based SMC framework that simultaneously applies SMC to both translational and rotational dynamics, with a formal guarantee of stability.
    \item Third, we augment our 6-DOF flight controller with new adaptation laws that enable automatic adjustment, both increase and decrease, of the switching gains in response to system behavior and disturbance levels, with global uniform ultimate boundedness results. 
\end{enumerate}
Hereafter, we refer to our quaternion-based SMC law as QSMC, and its adaptive variant as AQSMC.
\section{Attitude Control}\label{se:control_design}
To design our QSMC law for attitude control, we adopt the sliding surface introduced in \cite{lopez2020sliding}; however, we extend the results of \cite{lopez2020sliding} in several key directions:
\begin{enumerate}[label=\roman*.]
    \item As will be detailed shortly, the sliding surface is nondifferentiable on
    \begin{equation}\label{eq:N_s_q}
    \mathcal{N}_{\mathbf{s}_{q}} = \left\{ \mathbf{q}_e \in \mathbb{S}^3 \;\middle|\; q_{w_e} = 0,\; \|\Vec{\mathbf{q}}_e\| = 1 \right\}.
    \end{equation}
    Proposition 1 in \cite{lopez2020sliding} shows that $\mathcal{N}_{\mathbf{s}_{q}}$ is not invariant and poses no stability threat. However, the Lyapunov candidates used to prove stability in Theorems 1 -- 3 of \cite{lopez2020sliding} remain nondifferentiable on $\mathcal{N}_{\mathbf{s}_{q}}$, so their classical time derivatives are not defined there. We leverage nonsmooth stability analysis to rigorously address this issue for the first time.
    % \item To improve the smoothness of the control signal for practical implementations, we replace the $\sat(\cdot)$ function in the control law of \cite{lopez2020sliding} with the $\tanh(\cdot)$ function. 
    % However, unlike prior work in hyperbolic SMC  \cite{noordin2021sliding, nguyen2021adaptive, noordin2022position}, we introduce a tunable parameter that controls the width of the boundary layer, which proved to be critical factor in our hardware tests to reduce tracking accuracy.
    \item We introduce novel adaptation mechanisms for the switching gains of control law. These adaptive gains proved to be instrumental in achieving significantly higher agility and disturbance rejection, while also reducing control effort. They eliminate the need to know the upper bounds of uncertainties, greatly simplifying controller calibration and practical implementation.
    \item The results in \cite{lopez2020sliding} are limited to attitude control. Integrating quaternion-based sliding mode attitude control with sliding mode position control for quadrotors is nontrivial and has not been addressed in prior work \cite{sanchez2013time, serrano2023terminal, arellano2015quaternion, abaunza2019quadrotor}. We address this integration challenge in the following sections.
    \item Finally, we implement the proposed control strategy on a quadrotor and validate its practical effectiveness and robustness through hardware experiments, for the first time.
\end{enumerate}

To begin with the controller design, we adopt the following sliding surface
\begin{equation}\label{eq:s_att}
    \mathbf{s}_q = \boldsymbol{\omega}_e + \boldsymbol{\Lambda}_q \operatorname{sgn}_{+}(q_{w_e}) \Vec{\mathbf{q}}_e,
\end{equation}
where $\boldsymbol{\Lambda}_q = \mathrm{diag}({\Lambda}_{q_{ii}})$ is a design parameter and
\begin{equation}
    \operatorname{sgn}_{+}(\cdot) = 
    % 2 H(\cdot) - 1 = 
    \begin{cases}
        1 & \text{if } \cdot \ge 0, \\
        -1 & \text{if } \cdot < 0.
    \end{cases}
\end{equation}
% where $H(\cdot)$ is the Heaviside step function.
The attitude control law built upon \eqref{eq:s_att} takes the following form
\begin{equation}\label{eq:att_ctrl_law}
\begin{aligned}
    \boldsymbol{\tau} =\;& \hat{\mathbf{J}} \dot{\boldsymbol{\omega}}_d + \boldsymbol{\omega} \times \hat{\mathbf{J}} \boldsymbol{\omega} - \hat{\mathbf{J}} \boldsymbol{\Lambda}_q \operatorname{sgn}_{+}(q_{w_e}) \dot{\Vec{\mathbf{q}}}_e \\
    & - \hat{\mathbf{J}} \mathbf{K}_q \tanh\left( \mathbf{s}_q \oslash \boldsymbol{\phi}_q \right),
\end{aligned}
\end{equation}
where $\mathbf{K}_q = \mathrm{diag}(K_{q_{ii}})$ and $\boldsymbol{\phi}_q$ are design parameters. 

The term $\sgn_{+}(q_{w_e})$ is the key to handling the unwinding problem. 
As the quaternions $\mathbf{q}$ and $-\mathbf{q}$ represent the same physical orientation, $\operatorname{sgn_{+}}(q_{w_e})$ ensures the consistency of the unit quaternion sign in the control law.
To clarify further, let us assume the current vehicle attitude is $\mathbf{q} = [0.998, 0.02, 0, 0]^\top $ and the desired attitude is $\mathbf{q}_d = [0.999, 0.01, 0, 0]^\top $, indicating the need for a small correction in the vehicle orientation. Using the attitude error definition \eqref{eq:err_dyn} gives $\mathbf{q}_e = \mathbf{q}_d^\star \otimes \mathbf{q} = [-0.999, -0.001, 0, 0]^\top $, which suggests an unnecessarily large rotation.
Applying $\operatorname{sgn_{+}}(q_{w_e})$ results in $\sgn_{+}(q_{w_e}) \mathbf{q}_e = [0.999, 0.001, 0, 0]^\top $, which accurately reflects the intended small adjustment, thus resolving the unwinding issue.

% Furthermore, despite the discontinuity of $\sgn_{+}(\cdot)$, $\mathbf{s}_q$ remains continuous. 
% Again, since $\mathbf{q}$ and $-\mathbf{q}$ represent the same physical orientation, when $q_{w_e}$ crosses zero, both the scalar and vector parts of the quaternion simultaneously change sign. The multiplication by $\sgn_{+}(q_{w_e})$ compensates for this sign flip, ensuring the continuity of $\mathbf{s}_q$.

However, $\mathbf{s}_q$ is not differentiable at $q_{w_e} = 0$. 
Taking the derivative of \eqref{eq:s_att} leads to
\begin{equation}\label{eq:s_q_dot_with_dirac}
    \dot{\mathbf{s}}_q = \dot{\boldsymbol{\omega}}_{e} + \boldsymbol{\Lambda}_{q} \sgn_+(q_{w_e})\dot{\Vec{\mathbf{q}}}_e + \boldsymbol{\Lambda}_{q} 2\delta(q_{w_e})\dot{q}_{w_e}\Vec{\mathbf{q}}_e,
\end{equation}
where $\delta(\cdot)$ is the unit impulse which captures the jump discontinuity at $q_{w_e} = 0$ and reflects the fact that $\mathbf{s}_q$ is not classically differentiable on the measure-zero set \eqref{eq:N_s_q}.

Handling such discontinuities is challenging within the framework of classical Lyapunov theory. Moreover, the set $\mathcal{N}_{\mathbf{s}_{q}}$ corresponds to non-equilibrium points, preventing the direct application of tools such as LaSalle’s invariance principle \cite{khalil2002nonlinear}. This subtlety is not addressed in \cite{lopez2020sliding}. We close this gap by presenting the following theorem based on nonsmooth stability analysis given in Section \ref{se:nonsmooth_preliminaries}.

\begin{theorem}[QSMC attitude control]
Consider $\mathbf{q}_e$ and $\boldsymbol{\omega}_e$ defined in \eqref{eq:error_definition}, with dynamics governed by \eqref{eq:err_dyn}, under Assumptions \ref{as:trajectory}--\ref{as:known_bounds}. Set the sliding surface as \eqref{eq:s_att}, and control law as \eqref{eq:att_ctrl_law}.
Define the lumped uncertainty bound $\bar{\boldsymbol{\delta}}_q$ as
\begin{equation}\label{eq:delta_q_bar}
    \bar{\boldsymbol{\delta}}_q = |\Tilde{\mathbf{J}} \dot{\boldsymbol{\omega}}_d| + |\boldsymbol{\omega} \times \Tilde{\mathbf{J}} \boldsymbol{\omega}| + |\Tilde{\mathbf{J}} \boldsymbol{\Lambda}_q \dot{\Vec{\mathbf{q}}}_e|.
\end{equation}
Suppose the following conditions are satisfied for all $i = 1,2,3$, 
\begin{equation}\label{eq:gain_condition}
\begin{aligned}
&\Lambda_{q_{ii}} > 0,\quad \phi_{q_i} > 0,\quad \pi_{q_i} > 0, \\
&K_{q_{ii}} = \frac{1}{\hat{J}_{ii}} \left( \bar{\delta}_{q_i} + \bar{d}_{\alpha_i} + \pi_{q_i} \right),
\end{aligned}
\end{equation}
where $\pi_{q_i}$ is a design parameter.
Then, $\mathbf{q}_e$ and $\boldsymbol{\omega}_e$ are both globally uniformly ultimately bounded.
\end{theorem}

\begin{proof}
Consider the Lyapunov candidate function
\begin{equation}\label{eq:att_lyapunov_fcn}
    V_q = \frac{1}{2}\mathbf{s}_q^T \mathbf{J} \mathbf{s}_q
\end{equation}
We aim to compute $\dot{V}_q^{\circ}$ defined in \eqref{eq:nonsmooth_vdot_preliminaries}.
There exist two different cases to evaluate, depending on the value of $q_{w_e}$.

\textit{Case (i): $q_{w_e} \neq 0$.} In this case, $\mathbf{s}_q$ is continuously differentiable, and the right hand side of \eqref{eq:s_q_dot_with_dirac} reduces to $\dot{\boldsymbol{\omega}}_{e} + \boldsymbol{\Lambda}_{q} \sgn_+(q_{w_e})\dot{\Vec{\mathbf{q}}}_e$ as $\delta(q_{w_e})=0$.
Therefore, \eqref{eq:att_lyapunov_fcn} is continuously differentiable for all $q_{w_e}\neq0$. It follows from Definition 2 that $\partial_C V_q = \{\nabla V_q \}$. 
Consequently,
\begin{equation}\label{eq:theorem2_vdot_1}
    \dot{V}_q^{\circ} \big|_{q_{w_e} \neq 0} = \mathbf{s}_q^\top J (\dot{\boldsymbol{\omega}}_{e} + \boldsymbol{\Lambda}_{q} \sgn_+(q_{w_e})\dot{\Vec{\mathbf{q}}}_e). 
\end{equation}
Substituting \eqref{eq:err_dyn} and \eqref{eq:att_ctrl_law} into \eqref{eq:theorem2_vdot_1} yields
\begin{equation}\label{eq:theorem2_vdot_2}
\begin{aligned}
    \dot{V}_q^{\circ} \big|_{q_{w_e} \neq 0} = &\; \mathbf{s}_q^\top \big( 
    -\hat{\mathbf{J}}  \mathbf{K}_q \tanh\left(\mathbf{s}_q \oslash\boldsymbol{\phi}_q \right)
     -\Tilde{\mathbf{J}}\dot{\boldsymbol{\omega}}_d
     -\boldsymbol{\omega} \times \Tilde{\mathbf{J}}\boldsymbol{\omega} \\ &
    + \Tilde{\mathbf{J}} \boldsymbol{\Lambda}_{q} \operatorname{sgn_{+}}(q_{w_e})\dot{\Vec{\mathbf{q}}}_e 
     + \mathbf{J}\mathbf{d}_{{\alpha}}\big).
\end{aligned}
\end{equation}
The uncertain terms $-\Tilde{\mathbf{J}}\dot{\boldsymbol{\omega}}_d-\boldsymbol{\omega} \times \Tilde{\mathbf{J}}\boldsymbol{\omega} + \Tilde{\mathbf{J}} \boldsymbol{\Lambda}_{q} \operatorname{sgn_{+}}(q_{w_e})\dot{\Vec{\mathbf{q}}}_e$ have an upper bound given by \eqref{eq:delta_q_bar}, as does $\mathbf{d}_\alpha$, according to Assumption \ref{as:disturbance}. 
Therefore,
\begin{equation}\label{eq:v_dot_3}
\begin{aligned}
\dot{V}_q^{\circ}\big|_{q_{w_e} \neq 0} \leq\;& \sum_{i=1}^{3} \lvert s_{q_i} \rvert \left( \bar{\delta}_{q_i} + \bar{d}_{\alpha_i} \right) 
- s_{q_i} \hat{J}_{ii}K_{q_{ii}}\tanh{(\frac{s_{q_i}}{{\phi}_{q_i}})}.
\end{aligned}
\end{equation}
Next, substitute \eqref{eq:gain_condition} in \eqref{eq:v_dot_3}. To simplify notations, we set $\Delta_{q_i}=\bar{\delta}_{q_i} + \bar{d}_{\alpha_i}$ and $r_{q_i}=\frac{\Delta_{q_i}}{\Delta_{q_i}+\pi_{q_i}}$, which satisfies $0< r_{q_i}<1$ from \eqref{eq:gain_condition} and Assumption \ref{as:disturbance}.
Also, note that $\tanh(\cdot)$ is an odd function; therefore, $s_{q_i}\tanh(\frac{s_{q_i}}{\phi_i})=\lvert s_{q_i} \rvert \tanh(\frac{\lvert s_{q_i} \rvert }{\phi_i})$. Subsequently,
\begin{equation}\label{eq:v_dot_4}
\dot{V}_q^{\circ}\big|_{q_{w_e} \neq 0} \leq -\sum_{i=1}^{3} \left( \Delta_{q_i} + \pi_{q_i} \right)\lvert s_{q_i} \rvert \left( \tanh{\left(\frac{\lvert s_{q_i} \rvert}{{\phi}_{q_i}} \right)} - r_{q_i} \right).
\end{equation}
% Using $\tanh^{-1}(\cdot)=\frac{1}{2}\ln\left(\frac{1+\cdot}{1-\cdot}\right)$,
Let $s_{q_i}^\star= \phi_{q_i} \tanh^{-1}(r_{q_i}) =\frac{\phi_{q_i}}{2}  \ln(\frac{1+r_{q_i}}{1-r_{q_i}}) $.
For $\lvert s_{q_i} \rvert \ge s_{q_i}^\star$, 
$\tanh{(\frac{s_{q_i}}{{\phi}_{q_i}})} - r_{q_i} \ge 0$. Thus, $\dot{V}_q^{\circ}\big|_{q_{w_e} \neq 0} \leq 0$ with the following upper bound
\begin{equation}\label{eq:v_dot_5}
\dot{V}_q^{\circ}\big|_{q_{w_e} \neq 0} \leq - \underbrace{\min_{i=1,2,3}\{(\Delta_{q_i}+\pi_{q_i})(1-r_{q_i}) \} \frac{\sqrt{2}}{\sqrt{\lambda_{\max}(\mathbf{J})}}}_{c_1} \sqrt{V_q}.
\end{equation}
% \textcolor{red}{
% using the Rayleigh quotient for $ \mathbf{J} \succ 0 $, we have $ V_q = \frac{1}{2} \mathbf{s}_q^\top \mathbf{J} \mathbf{s}_q \leq \frac{1}{2} \lambda_{\max}(\mathbf{J}) \| \mathbf{s}_q \|^2 $, which yields $ \| \mathbf{s}_q \| \geq \sqrt{2V_q / \lambda_{\max}(\mathbf{J})}$, where $ \lambda_{\max}(\mathbf{J}) $ is the largest eigenvalue of $ \mathbf{J} $.
% }
However, for $\lvert s_{q_i} \rvert < s_{q_i}^\star$, 
$\tanh{(\frac{s_{q_i}}{{\phi}_{q_i}})} - r_{q_i} < 0$, and the right hand side of \eqref{eq:v_dot_4} is bounded by a positive value as follows
\begin{equation}\label{eq:v_dot_6}
\dot{V}_q^{\circ}\big|_{q_{w_e} \neq 0} \leq \underbrace{\sum_{i=1}^{3} 
s_{q_i}^\star
\left( \Delta_{q_i}+\pi_{q_i} \right)}_{c_2}.
\end{equation}
To account for both regimes defined in \eqref{eq:v_dot_5} and \eqref{eq:v_dot_6}, we derive a conservative upper bound as follows
% Combining \eqref{eq:v_dot_5} and \eqref{eq:v_dot_6} leads to
\begin{equation}\label{eq:v_dot_7}
\dot{V}_q^{\circ}\big|_{q_{w_e} \neq 0} \leq -c_1\sqrt{V_q}+c_2.
\end{equation}

\textit{Case (ii): $q_{w_e} = 0$.} 
In this case, \eqref{eq:att_lyapunov_fcn} is not continuously differentiable; therefore, we start by calculating the Clarke generalized gradient of $\mathbf{s}_q$.
It follows from Definition 2 that
\begin{equation}
\begin{aligned}
\partial_C \mathbf{s}_q \big|_{q_{w_e} = 0} =\;&\operatorname{co} \left\{\left[ \frac{\partial \mathbf{s}_q}{\partial q_{w_e}} \; \frac{\partial \mathbf{s}_q}{\partial \vec{\mathbf{q}}_e}\; \frac{\partial \mathbf{s}_q}{\partial \boldsymbol{\omega}_e} \right]\right\}\\ =\;&\operatorname{co} \left\{
    \left[ \mathbf{0}_{3 \times 1} \; \Lambda_q\; \mathbf{I}_{3 \times 3} \right], 
    \left[ \mathbf{0}_{3 \times 1} \; -\Lambda_q\; \mathbf{I}_{3 \times 3} \right]
\right\},
\end{aligned}
\end{equation}
which can be rewritten as
\begin{equation}\label{eq:CG_sq}
\partial_C \mathbf{s}_q \big|_{q_{w_e} = 0} = 
\left\{ 
\left[
\mathbf{0}_{3 \times 1} \; \varrho \Lambda_q \; \mathbf{I}_{3 \times 3}
\right] \;\middle|\; \varrho \in [-1, 1]
\right\}.
\end{equation}
Using \eqref{eq:nonsmooth_vdot_preliminaries} and \eqref{eq:CG_sq} leads to
\begin{equation}
\partial_C V_q \big|_{q_{w_e} = 0} = \left\{ \mathbf{s}_q^\top \mathbf{J} \left[ \mathbf{0}_{3 \times 1} \; \varrho \Lambda_q \; \mathbf{I}_{3 \times 3} \right] \;\middle|\; \varrho \in [-1,1] \right\}.
\end{equation}
Therefore,
\begin{equation}
\begin{aligned}
\dot{V}_q^{\circ}\big|_{q_{w_e} = 0}
=& \max_{\varrho \in [-1,1]} \biggl\{ \left\langle 
\left[ \mathbf{0} \; \varrho \mathbf{s}_q^\top \mathbf{J} \Lambda_q \; \mathbf{s}_q^\top \mathbf{J} \right]^\top,
\begin{bmatrix}
\dot{q}^w_e \\
\dot{\vec{\mathbf{q}}}_e \\
\dot{\boldsymbol{\omega}}_e
\end{bmatrix}
\right\rangle \biggr\}\\
 =& \max_{\varrho \in [-1,1]} \left\{ \varrho \mathbf{s}_q^\top \mathbf{J} \Lambda_q \dot{\vec{\mathbf{q}}}_e + \mathbf{s}_q^\top \mathbf{J} \dot{\boldsymbol{\omega}}_e \right\}.
\end{aligned}
\end{equation}
We now follow a process similar to \textit{Case (i)}, that is, substituting \eqref{eq:err_dyn} and \eqref{eq:att_ctrl_law}, followed by applying the upper bounds of the uncertain terms to arrive at
\begin{equation}\label{eq:vdot_maximization}
\begin{split}
    \dot{V}_q^{\circ}\big|_{q_{w_e} = 0}
    = \max_{\varrho \in [-1,1]} \biggl\{\sum_{i=1}^{3} \biggl(
    \lvert s_{q_i} \rvert \left( \bar{\delta}_{q_i} + \bar{d}_{\alpha_i} \right) 
    \\
    - s_{q_i} \hat{J}_{ii}K_{q_{ii}}\tanh({\frac{s_{q_i}}{{\phi}_{q_i}}) \biggr)} 
    \biggr\}.
\end{split}
\end{equation}
Since the maximand is independent of $\varrho$, the maximization simplifies to an identity operation.
Consequently, 
\begin{equation}\label{eq:vdot_rhs}
\dot{V}_q^{\circ}\big|_{q_{w_e} = 0} \leq \sum_{i=1}^{3}
    \lvert s_{q_i} \rvert \left( \bar{\delta}_{q_i} + \bar{d}_{\alpha_i} \right) 
    - s_{q_i} \hat{J}_{ii}K_{q_{ii}}\tanh{(\frac{s_{q_i}}{{\phi}_{q_i}})}.
\end{equation}
The right hand side of \eqref{eq:vdot_rhs} is similar to that of \eqref{eq:v_dot_3}. Following the same approach as in \eqref{eq:v_dot_3}--\eqref{eq:v_dot_7}, we obtain
\begin{equation}\label{eq:final_vdot_att_th2_part2}
    \dot{V}_q^{\circ}\big|_{q_{w_e} = 0} \leq - c_1 \sqrt{V_q}+c_2.
\end{equation}
Combining \eqref{eq:v_dot_7} and \eqref{eq:final_vdot_att_th2_part2} and integrating \eqref{eq:final_vdot_att_th2_part2} yields
\begin{equation}\label{eq:v_q_bound_times}
    V_q(t)\le
\begin{cases}
\bigl(\sqrt{V_q(0)}-\tfrac{c_1}{2}\,t\bigr)^2 & \text{for }t\le t^\star,\\[4pt]
\displaystyle\Bigl(\frac{c_2}{c_1}\Bigr)^{\!2} & \text{for }t\ge t^\star,
\end{cases}
\end{equation}
where $t^\star= \max\{0,\dfrac{2}{c_1}\bigl(\sqrt{V_q(0)}-\tfrac{c_2}{c_1}\bigr)\}$.
This implies that $V_q(t)$ enters a ball of radius $\left(\frac{c_2}{c_1}\right)^2$ in finite time $t^\star$, and remains within that ball thereafter. Since $V_q$ is positive definite in $(\mathbf{q}_e, \boldsymbol{\omega}_e)$, it follows that $\mathbf{q}_e$ and $\boldsymbol{\omega}_e$ are globally uniformly ultimately bounded.
\end{proof}

% {\color{red}Let us now discuss the results.
% The size of boundary layer depends on $s^\star$ which can be fine-tuned by design parameters.
% However, in practice, it might be difficult to find good bounds to $K$. We will develop adaptive laws to address this.

% Furthermore, as detailed in \cite{lopez2020sliding}, the design of $s_q$ ensures that $q_e$ evolves smoothly on $\mathbb{S}^3$, converging exponentially to the identity quaternion. Although $\sgn_+(q_{w_e})$ introduces a discontinuity, it is treated as constant along smooth trajectory segments, allowing $\dot{s}_q$ to be computed from the differentiable terms. This piecewise-smooth interpretation avoids unwinding and preserves the validity of the Lyapunov analysis.

% }

\vspace{-1\baselineskip} 
\section{6-DOF Flight Control}
In this section, we extend the results of Theorem 2 by introducing a position control law and integrating it with the attitude controller defined in \eqref{eq:s_att} and \eqref{eq:att_ctrl_law}.

First, we aim to determine $f$ in \eqref{eq:error_definition}.
To this end, we define an auxiliary variable
\begin{equation}\label{eq:kappa}
   \boldsymbol{\kappa} = f L_\mathbf{q}(\mathbf{e}_3),
\end{equation}
where we have applied the homogeneous property of $L_\mathbf{q}(\cdot)$. Applying the dot product of $L_\mathbf{q}(\mathbf{e}_3)$ to both sides of \eqref{eq:kappa}, followed by using the $L_\mathbf{q}(\cdot)$ isometric property and carrying out the formula \eqref{eq:L} yields
\begin{equation}\label{eq:f}
        f= \boldsymbol{\kappa} \cdot L_\mathbf{q}(\mathbf{e}_3)= \boldsymbol{\kappa} \cdot \begin{bmatrix}
            2(q_x q_z + q_w q_y)\\
            2(q_y q_z - q_w q_x)\\
q_w^2 - q_x^2 - q_y^2 + q_z^2\\
        \end{bmatrix}.
\end{equation}
Therefore, we can treat $\boldsymbol{\kappa}$ as our position control input and compute $f$ from \eqref{eq:f}.

Subsequently, we define the following sliding surface
\begin{equation}\label{eq:s_pos}
    \mathbf{s}_{\xi}  = \boldsymbol{\nu}_{e} + \boldsymbol{\Lambda}_{\xi} \boldsymbol{\xi}_{e},
\end{equation}
where $\boldsymbol{\Lambda}_\xi = \diag({\Lambda_{\xi_{ii}}})$ is a design parameter.
Based on this surface, we propose the following position control law\begin{equation}\label{eq:kappa2}{\boldsymbol{\kappa}} =  \hat{m} \left( {g\mathbf{e}_{3}} + \ddot{\boldsymbol{{\xi}}}_d -\boldsymbol{\Lambda}_{\xi} \boldsymbol{\nu}_{e} -\mathbf{K}_{\xi} \tanh\left(\mathbf{s}_{\xi}\oslash\boldsymbol{\phi}_{\xi}\right)\right),
\end{equation}
where $\mathbf{K}_{\xi}=\diag(K_{{\xi}_{ii}})$ and $\boldsymbol{\phi}_\xi$ are design parameters.

To ensure the thrust direction is always well defined, we impose the positive‑margin condition
\begin{equation}\label{eq:kappa_nonvanishing}
    \bigl\| \boldsymbol{\kappa}(t) \bigr\|
\;\ge\;
\underline{\kappa}_0 \;>\; 0,
\qquad\forall t\ge0.
\end{equation}
which is easily satisfied in practice by maintaining a small, non-zero throttle floor so the total thrust never reaches zero.

To tightly integrate the position control law with the attitude controller defined in \eqref{eq:s_att}–\eqref{eq:att_ctrl_law}, we propose the following attitude trajectory generation method. While our approach is inspired by \cite{lee2010geometric}, it introduces a key modification specifically tailored to our control framework, as will be detailed shortly.

Let us consider the desired rotation matrix $
    \mathbf{R}_{d} = \left[\mathbf{b}_{1_d}, \mathbf{b}_{2_d}, \mathbf{b}_{3_d}\right]$,
where $\mathbf{b}_{3_d} = {\boldsymbol{\kappa}}/{\lVert \boldsymbol{\kappa} \rVert}$.
The goal is to find $\mathbf{b}_{1_d}$ and $\mathbf{b}_{2_d}$ to construct ${\mathbf{R}}_{d}$, which will lead to finding ${\mathbf{q}}_d$, $\boldsymbol{\omega}_d$, and $\boldsymbol{\alpha}_d = \dot{\boldsymbol{\omega}}_d$.
Let $\mathbf{h}_d = [\cos \psi_d, \sin \psi_d, 0]^\top $ represent the desired heading direction of the vehicle in the horizontal plane.
We define $\mathbf{b}_{2_d}$ as $\mathbf{b}_{2_d} = {\boldsymbol{\chi}}/{\lVert \boldsymbol{\chi} \rVert}$ where $\boldsymbol{\chi} = \mathbf{b}_{3_d} \times \mathbf{h}_{d}$.
Next, we define $\mathbf{b}_{1_d} = \mathbf{b}_{2_d} \times \mathbf{b}_{3_d}$ to build ${\mathbf{R}}_{d}$, and subsequently ${\mathbf{q}_d}$.

To generate $\boldsymbol{\omega}_d$ and $\boldsymbol{\alpha}_d$, we use $\dot{\mathbf{R}}_{d} =[\dot{\mathbf{b}}_{1_d},\dot{\mathbf{b}}_{2_d},\dot{\mathbf{b}}_{3_d}]$ and $\ddot{\mathbf{R}}_{d} = [\ddot{\mathbf{b}}_{1_d}, \ddot{\mathbf{b}}_{2_d}, \ddot{\mathbf{b}}_{3_d}]$ and derive the following expressions
\begin{equation}\label{eq:ref_gen1}
    \dot{\mathbf{b}}_{3_d} = \frac{\dot{\boldsymbol{\kappa}}}{\lVert \boldsymbol{\kappa} \rVert} - \frac{\boldsymbol{\kappa} \cdot \dot{\boldsymbol{\kappa}}}{\lVert \boldsymbol{\kappa} \rVert ^3} \boldsymbol{\kappa},
\end{equation}
\begin{equation}
    \dot{\boldsymbol{\chi}} = \dot{\mathbf{b}}_{3_d} \times \mathbf{h}_{d} + \mathbf{b}_{3_d} \times \dot{\mathbf{h}}_d,
\end{equation}
\begin{equation}
    \dot{\mathbf{b}}_{2_d} = \frac{\dot{\boldsymbol{\chi}}}{\lVert \boldsymbol{\chi} \rVert} - \frac{\boldsymbol{\chi} \cdot \dot{\boldsymbol{\chi}}}{\lVert \boldsymbol{\chi} \rVert ^3} \boldsymbol{\chi},
\end{equation}
\begin{equation}
    \dot{\mathbf{b}}_{1_d} = \dot{\mathbf{b}}_{2_d} \times \mathbf{b}_{3_d} + \mathbf{b}_{2_d} \times \dot{\mathbf{b}}_{3_d},
\end{equation}
\begin{equation}
\begin{split}
    \ddot{\mathbf{b}}_{3_d} = \frac{\ddot{\boldsymbol{\kappa}}}{\lVert \boldsymbol{\kappa} \rVert} - 2\frac{ \boldsymbol{\kappa} \cdot \dot{\boldsymbol{\kappa}}}{\lVert \boldsymbol{\kappa} \rVert ^3} \dot{\boldsymbol{\kappa}} - \frac{ \lVert \dot{\boldsymbol{\kappa}} \rVert^2 + \boldsymbol{\kappa} \cdot \ddot{\boldsymbol{\kappa}}}{\lVert \boldsymbol{\kappa} \rVert ^3} \boldsymbol{\kappa} \\
     + 3\frac{ (\boldsymbol{\kappa} \cdot \dot{\boldsymbol{\kappa}})^2}{\lVert \boldsymbol{\kappa} \rVert ^5} \boldsymbol{\kappa},
\end{split}
\end{equation}
\begin{equation}
    \ddot{\boldsymbol{\chi}} = {\ddot{\mathbf{b}}_{3_d} \times \mathbf{h}_{d} + \mathbf{b}_{3_d} \times \ddot{\mathbf{h}}_{d} + 2 \dot{\mathbf{b}}_{3_d} \times \dot{\mathbf{h}}_{d}},
\end{equation}
\begin{equation}
\begin{split}
    \ddot{\mathbf{b}}_{2_d} = \frac{\ddot{\boldsymbol{\chi}}}{\lVert \boldsymbol{\chi} \rVert} - 2\frac{ \boldsymbol{\chi} \cdot \dot{\boldsymbol{\chi}}}{\lVert \boldsymbol{\chi} \rVert ^3} \dot{\boldsymbol{\chi}} - \frac{ \lVert \dot{\boldsymbol{\chi}} \rVert^2 + \boldsymbol{\chi} \cdot \ddot{\boldsymbol{\chi}}}{\lVert \boldsymbol{\chi} \rVert ^3} \boldsymbol{\chi} \\
     + 3\frac{ (\boldsymbol{\chi} \cdot \dot{\boldsymbol{\chi}})^2}{\lVert \boldsymbol{\chi} \rVert ^5} \boldsymbol{\chi},
\end{split}
\end{equation}
\begin{equation}
    \ddot{\mathbf{b}}_{1_d} = {\ddot{\mathbf{b}}_{2_d} \times \mathbf{b}_{3_d} + \mathbf{b}_{2_d} \times \ddot{\mathbf{b}}_{3_d} + 2 \dot{\mathbf{b}}_{2_d} \times \dot{\mathbf{b}}_{3_d}}.
\end{equation}

Using \eqref{eq:kappa2}, we compute $\dot{\boldsymbol{\kappa}}$ and $\ddot{\boldsymbol{\kappa}}$ as follows
\begin{equation}\label{eq:kappa_dot}
\begin{split}
    \dot{\boldsymbol{\kappa}} =  m \Bigl(\dddot{\boldsymbol{\xi}}_d -\boldsymbol{\Lambda}_{\xi}\mathbf{a}_{e}  -\dot{\mathbf{K}}_{\xi} \tanh\left(\mathbf{s}_{\xi}\oslash\boldsymbol{\phi}_{\xi}\right)\\
    -\mathbf{K}_{\xi}\sech^{2}(\mathbf{s_{\xi}}\oslash\boldsymbol{\phi}_{\xi}) \circ \dot{\mathbf{s}}_{\xi}\oslash\boldsymbol{\phi}_{\xi}\Bigr),
\end{split}
\end{equation}
\begin{equation}\label{eq:kappa_2dot}
\begin{split}
    \ddot{\boldsymbol{\kappa}} =  m \Biggl( \boldsymbol{\xi}^{(4)}_d -\boldsymbol{\Lambda}_{\xi} \mathbf{j}_{e}  -\ddot{\mathbf{K}}_{\xi} \tanh\left(\mathbf{s}_{\xi}\oslash\boldsymbol{\phi}_{\xi}\right)
    \\ 
    -\sech^{2}(\mathbf{s_{\xi}}\oslash\boldsymbol{\phi}_{\xi}) \circ \Bigl(2\dot{\mathbf{K}}_{\xi}\dot{\mathbf{s}}_{\xi}\oslash\boldsymbol{\phi}_{\xi}
    +\mathbf{K}_{\xi} \ddot{\mathbf{s}}_{\xi}\oslash\boldsymbol{\phi}_{\xi} \\ - 2\mathbf{K}_{\xi} \tanh(\mathbf{s_{\xi}}\oslash\boldsymbol{\phi}_{\xi}) \circ (\dot{\mathbf{s}}_{\xi}\oslash\boldsymbol{\phi}_{\xi})^{\circ 2}\Bigr)\Biggr),
    \end{split}
\end{equation}
where $\mathbf{a}_{e}$ and $\mathbf{j}_{e}$ are the acceleration and jerk errors.
In QSMC, $\mathbf{K}_\xi$ is constant, so $\dot{\mathbf{K}}_\xi = \ddot{\mathbf{K}}_\xi = \mathbf{0}$. However, we present the general forms above as they will be directly applicable when developing adaptive switching gains in the next section.

Having established $\mathbf{R}_d$, $\dot{\mathbf{R}}_d$, and $\ddot{\mathbf{R}}_d$, we utilize Poisson’s kinematic equation $\dot{\mathbf{R}}_d =\mathbf{R}_d\omega^{\times}_{R_d}$~\cite{stevens2015aircraft}, to compute the desired angular velocity and acceleration as follows
\begin{equation}
    \boldsymbol{\omega}_{R_d} = (\mathbf{R}_d^\top  \dot{\mathbf{R}}_d)^{\vee},
\end{equation}
\begin{equation}
    {\boldsymbol{\alpha}}_{R_d} =\dot{{\boldsymbol{\omega}}}_{R_d}= \biggr(\mathbf{R}_d^\top  \ddot{\mathbf{R}}_d-{\boldsymbol{\omega}^{\times}_{{R_d}}}^2 \biggr)^{\vee}.
\end{equation}
However, both $\boldsymbol{\omega}_{R_d}$ and $\boldsymbol{\alpha}_{R_d}$ reside in the tangent space $T_{R_d} \mathrm{SO}(3)$; while our error definitions \eqref{eq:error_definition} and subsequent control design requires these quantities to lie in $T_{R} \mathrm{SO}(3)$.
To address this, we map $\boldsymbol{\omega}_{R_d}$ and ${\boldsymbol{\alpha}}_{R_d}$ to $T_{R} \mathrm{SO}(3)$ as follows
\begin{equation}\label{eq:trajectory_gen1}
    \boldsymbol{\omega}_d = \mathbf{R}^\top \mathbf{R}_d\omega_{R_d},
\end{equation}
\begin{equation}\label{eq:trajectory_gen2}
\boldsymbol{\alpha}_d =\dot{\boldsymbol{\omega}}_d   =-\omega^{\times}\mathbf{R}^\top \mathbf{R}_d\omega_{R_d} +\mathbf{R}^\top \mathbf{R}_d\dot{\omega}_{R_d}.
\end{equation}
It is noteworthy that in \cite{lee2010geometric}, the quantities in $T_{R_d} \mathrm{SO}(3)$ are used for control developments. By mapping them to $T_{R} \mathrm{SO}(3)$ and using \eqref{eq:trajectory_gen1} in \eqref{eq:err_dyn}, we significantly simplify both the control design and the subsequent stability analysis, which would otherwise be considerably more tedious.

% \begin{equation}
% \begin{gathered}
% {\boldsymbol{\alpha}}_d =\dot{{\boldsymbol{\omega}}}_d= \dot{\mathbf{R}}^\top \mathbf{R}_d\omega_{R_d} +\mathbf{R}^\top \dot{\mathbf{R}}_d\omega_{R_d} + \mathbf{R}^\top \mathbf{R}_d\dot{\omega}_{R_d} \\
%     =-\omega^{\times}\mathbf{R}^\top \mathbf{R}_d\omega_d^{R_d} +\mathbf{R}^\top \dot{\mathbf{R}}_d\omega{R_d} + \mathbf{R}^\top \mathbf{R}_d\dot{\omega}_{R_d}\\
%     =-\omega^{\times}\mathbf{R}^\top \mathbf{R}_d\omega_d^{R_d} +\mathbf{R}^\top \mathbf{R}_d\omega_d^{{R_d}^{\times}}\omega_d^{R_d} + \mathbf{R}^\top \mathbf{R}_d\dot{\omega}^{R_d}_d\\
%     =-\omega^{\times}\mathbf{R}^\top \mathbf{R}_d\omega_d^{R_d} +\mathbf{R}^\top \mathbf{R}_d\dot{\omega}^{R_d}_d
%     \end{gathered}
% \end{equation}

\begin{theorem}[QSMC 6-DOF flight control]
Consider the error terms defined in \eqref{eq:error_definition}, with dynamics governed by \eqref{eq:err_dyn}, under Assumptions \ref{as:trajectory}--\ref{as:known_bounds}. 
For attitude control, adopt the sliding surface \eqref{eq:s_att}, control law \eqref{eq:att_ctrl_law}, lumped uncertainty bound \eqref{eq:delta_q_bar}, and conditions for design parameters \eqref{eq:gain_condition}.
For position control, set the sliding surface \eqref{eq:s_pos}, and apply control law \eqref{eq:f}, \eqref{eq:kappa}, and \eqref{eq:kappa_nonvanishing}.
Define the lumped uncertainty bound $\bar{\boldsymbol{\delta}}_\xi$ as
\begin{equation}\label{eq:delta_pos_bar}
    \bar{\boldsymbol{\delta}}_{\xi} = \lvert \frac{\hat{m}}{m}-1 \rvert(\lvert \ddot{\boldsymbol{{\xi}}}_d \rvert + \lvert \boldsymbol{\Lambda}_{\xi} \boldsymbol{\nu}_{e}\rvert + \lvert{g\mathbf{e}_{3}}\rvert).
\end{equation}
Suppose the the following conditions are satisfied for all $i = 1,2,3$, 
\begin{equation}\label{eq:gain_condition_pos}
\begin{aligned}
&\Lambda_{\xi_{ii}} > 0,\quad \phi_{\xi_i} > 0,\quad \pi_{\xi_i} > 0, \\
&K_{\xi_{ii}} = \frac{1}{\rho_{\xi}} \left( \bar{\delta}_{\xi_i} + \bar{d}_{a_i} + \pi_{\xi_i} \right),
\end{aligned}
\end{equation}
where $\pi_{\xi_i}$ is a design parameter, and $\rho_\xi = \frac{\hat{m}}{m}$.
Obtain $\mathbf{q}_d$, $\boldsymbol{\omega}_d$, and $\boldsymbol{\alpha}_d$ from $\boldsymbol{\kappa}$ and $\psi_d$ via \eqref{eq:ref_gen1}--\eqref{eq:trajectory_gen2}.
% , assuming $\mathbf{h}_d \nparallel \mathbf{b}_{3d}$.
Then, $\boldsymbol{\xi}_e$, $\boldsymbol{\nu}_e$, $\mathbf{q}_e$ and $\boldsymbol{\omega}_e$ are all globally uniformly ultimately bounded.
\end{theorem}

\begin{proof}
The proof proceeds in three steps. First, we analyze the stability of the position control loop. Second, we examine the well-posedness and boundedness of the attitude reference signals generated from the position control output. Finally, we show that the combined system achieves uniform ultimate boundedness by analyzing the composite Lyapunov function for the full 6-DOF closed-loop dynamics.

Let us begin with
\begin{equation}\label{eq:V_pos}
    V_\xi = \frac{1}{2}\mathbf{s}^\top _{\xi}\mathbf{s}_{\xi}.
\end{equation}
Given \eqref{eq:s_pos}, $V_\xi$ is smooth, and therefore, we proceed with the classical derivative of $V_\xi$ followed by substituting \eqref{eq:err_dyn}
\begin{equation}\label{eq:vdot_pos_1}
    \dot{V}_\xi = \mathbf{s}^\top \left( -{g\mathbf{e}_{3}} + \frac{1}{{m}} L_\mathbf{q}(f\mathbf{e}_{3}) + \mathbf{d}_a - \ddot{\boldsymbol{{\xi}}}_d + \boldsymbol{\Lambda}_{\xi} \boldsymbol{\nu}_{e} \right)
\end{equation}
Substituting \eqref{eq:f}, \eqref{eq:kappa}, and \eqref{eq:gain_condition_pos} into \eqref{eq:vdot_pos_1}, applying \eqref{eq:delta_pos_bar} and Assumption \ref{as:disturbance}, and following derivations similar to \eqref{eq:v_dot_3}--\eqref{eq:v_dot_4} yields
\begin{equation}\label{eq:vdot_pos_2}
\dot{V}_\xi \leq -\sum_{i=1}^{3} \biggl( \left( \Delta_{\xi_i} + \pi_{\xi_i} \right)\lvert s_{\xi_i} \rvert \left( \tanh{(\frac{\lvert s_{\xi_i} \rvert}{{\phi}_{\xi_i}})} - r_{\xi_i} \right) \biggr),
\end{equation}
where $\Delta_{\xi_i}=\bar{\delta}_{\xi_i} + \bar{d}_{a_i}$ and $r_{\xi_i}=\frac{\Delta_{\xi_i}}{\Delta_{\xi_i}+\pi_{\xi_i}} \in (0,1)$.
As seen in \eqref{eq:v_dot_5}--\eqref{eq:v_dot_6}, the expression \eqref{eq:vdot_pos_2} translates into
\begin{equation}\label{eq:vdot_pos_3}
\dot{V}_\xi = -c_3 \sqrt{V_\xi} + c_4,
\end{equation}
where $c_3 = \min_{i=1,2,3}\{(\Delta_{\xi_i}+\pi_{\xi_i})(1-r_{\xi_i})\}\sqrt{2}$, and $c_4=\sum_{i=1}^3s_{\xi_i}^\star(\Delta_{\xi_i}+\pi_{\xi_i})$, and $s_\xi^\star= \frac{\phi_{\xi_i}}{2}\ln(\frac{1+r_{\xi_i}}{1-r_{\xi_i}})$.

Similar to \eqref{eq:final_vdot_att_th2_part2}--\eqref{eq:v_q_bound_times}, \eqref{eq:vdot_pos_3} concludes that $V_\xi(t)$, and subsequently
$\mathbf{s}_\xi(t)$, $\boldsymbol{\xi}_e(t)$, and $\boldsymbol{\nu}_e(t)$ are bounded, i.e., $\|\mathbf{s}_\xi(t)\|\le\ \bar{s}_\xi$, $\|\boldsymbol{\xi_e}(t)\| \le \bar{\xi}_e$, and $\|\boldsymbol{\nu}_e(t)\| \le \bar\nu_e$.
Moreover, the error dynamics \eqref{eq:err_dyn} and the boundedness of $\boldsymbol{\xi}_d^{(3)}$, $\boldsymbol{\xi}_d^{(4)}$, and $\mathbf{d}_a$ from Assumptions \ref{as:trajectory} and \ref{as:model} imply that $\dot{\boldsymbol{\nu}}_e$ and $\ddot{\boldsymbol{\nu}}_e$ are bounded. Therefore, $\dot{\boldsymbol{s}}_\xi = \dot{\boldsymbol{\nu}}_e + \boldsymbol{\Lambda}_\xi \boldsymbol{\nu}_e$ and $\ddot{\mathbf{s}}_\xi = \ddot{\boldsymbol{\nu}}_e + \boldsymbol{\Lambda}_\xi \dot{\boldsymbol{\nu}}_e$ are bounded.

Using \eqref{eq:kappa}, \eqref{eq:kappa_nonvanishing}, and $0<\|\tanh(\mathbf{s_{\xi}} \oslash \boldsymbol{\phi}_{\xi})\|<\sqrt{3}$, we conclude $ 0<\underline{\kappa}_0 \le \|\boldsymbol{\kappa}(t)\| \le \bar{\kappa}_0 $, where
\begin{equation}
\bar{\kappa}_0 = \hat m\bigl(g+B_{\xi_2}+\|\Lambda_\xi\|\bar\nu+\sqrt{3}\|K_\xi\|\bigr).
\end{equation}
% \begin{equation}
% \underline{\kappa} = \hat m\bigl(g-B_{\xi_2}-\|\Lambda_\xi\|\underline\nu\bigr).
% \end{equation}
Differentiating \eqref{eq:kappa} twice and invoking the $C^{4}$ bounds
$\|\boldsymbol{\xi}_d^{(i)}\|\le B_{\xi_i}\,(i=2,3,4)$ from Assumption \ref{as:trajectory},
plus the boundedness of $\mathbf{s}_\xi$ and $\dot{\mathbf{s}}_\xi$, yields
$\|\boldsymbol{\kappa}^{(i)}(t)\|\le\bar\kappa_i$ for $i=0,1,2$.

Because $\|\boldsymbol{\kappa}(t)\|\ge\underline{\kappa}_0 >0$, all normalizations in \eqref{eq:ref_gen1}--\eqref{eq:trajectory_gen2} are well defined.
Also, $\psi_d \in C^{2}$ from Assumption \ref{as:trajectory}. Thus, the mapping \eqref{eq:ref_gen1}--\eqref{eq:trajectory_gen2} is $C^{2}$ and produces bounded signals
\begin{equation}
\|\mathbf{q}_d(t)\|=1,\quad
\|\boldsymbol{\omega}_d(t)\|\le\bar\omega,\quad
\|\boldsymbol{\alpha}_d(t)\|\le\bar\alpha.
\end{equation}
This implies that the desired trajectory generated via \eqref{eq:kappa}--\eqref{eq:kappa_nonvanishing} and the mapping \eqref{eq:ref_gen1}--\eqref{eq:trajectory_gen2} is smooth, bounded, and free of singularities, satisfying \eqref{eq:delta_q_bar} in Theorem 2.

Let us now consider the stability of translational and rotational dynamics together. Define
\begin{equation}
    V_1 = V_q + V_\xi
\end{equation}
It follows from \eqref{eq:v_dot_7}, \eqref{eq:final_vdot_att_th2_part2} and \eqref{eq:vdot_pos_3} that 
\begin{equation}
    \dot{V}_1^\circ = -c_1 \sqrt{V_q} + c_2 -c_3 \sqrt{V_\xi} + c_4.
\end{equation}
Note that $-c_1\sqrt{V_q}-c_3\sqrt{V_\xi}\le -\min\{c_1,c_3\}\bigl(\sqrt{V_q}+\sqrt{V_\xi}\bigr)\le
-c_5\,\sqrt{V_1}$, where $c_5 = \frac{1}{2}\min\{c_3,c_4\}$. Also, set $c_6=c_2+c_4$.
Then,
\begin{equation}\label{eq:70}
    \dot{V}_1^\circ = -c_5 \sqrt{V_1} + c_6.
\end{equation}
% Using the nonsmooth comparison lemma one obtains
% \begin{equation}
%     V_1(t)\le
% \begin{cases}
% \bigl(\sqrt{V(0)}-\tfrac{c_5}{2}\,t\bigr)^2 & \text{for }t\le t^\star,\\[4pt]
% \displaystyle\Bigl(\frac{c_6}{c_5}\Bigr)^{\!2} & \text{for }t\ge t^\star,
% \end{cases}
% \end{equation}
% where $t^\star= \max\{0,\dfrac{2}{c_5}\bigl(\sqrt{V(0)}-\tfrac{c_6}{c_5}\bigr)\}$ is the entry time to the ultimate bound.
Similar to \eqref{eq:v_q_bound_times}, we conclude that $V_1(t)$ enters a ball of radius $\left( \frac{c_6}{c_5} \right)^2$ in finite time $t^\star= \max\{0,\dfrac{2}{c_5}\bigl(\sqrt{V_1(0)}-\tfrac{c_6}{c_5}\bigr)\}$, and remains within that ball thereafter. 
This implies global uniform ultimate boundedness of $\boldsymbol{\xi}_e$, $\boldsymbol{\nu}_e$, $\mathbf{q}_e$, and $\boldsymbol{\omega}_e$.
\end{proof}

Note that, in contrast to several existing quaternion-based approaches \cite{sanchez2013time, serrano2023terminal, arellano2015quaternion, abaunza2019quadrotor}, the proposed QSMC offers a full 6-DOF flight control law entirely based on SMC, thereby leveraging the inherent robustness of SMC for both position and attitude control simultaneously. Furthermore, unlike many other 6-DOF Euler-based SMC designs, the proposed QSMC avoids any simplification of the attitude dynamics. This results in significantly enhanced control performance, particularly during aggressive maneuvers, as will be shown in our experiments.

% (iii) The proposed QSMC adopts hyperbolic formulation, offering smoother control over saturation-based implementations and simplicity over fuzzy or HOSMC variants. Also, the hyperbolic function incorporates a design parameter, $\boldsymbol{\phi}_\xi$ or $\boldsymbol{\phi}_q$, providing additional flexibility in calibrating the controllers, which proved effective in our hardware experiments, especially to minimize steady-state error.

% \vspace{-0.3\baselineskip} 

\section{Adaptive 6-DOF Flight Control}
As detailed in our hardware tests, the 6-DOF QSMC in Theorem 3 already outperforms benchmark controllers. Here, we take it further by introducing adaptation laws that unlock dramatically higher robustness and agility.

This adaptive framework removes the need for Assumption \ref{as:known_bounds}, making the controller far more practical. The adaptation law automatically adjusts the switching gains in response to changing flight conditions, improving tracking accuracy, while reducing control effort, compared to QSMC.

Crucially, our results resolve a longstanding issue in prior works such as \cite{nadda2018adaptive,lian2021adaptive,thanh2018quadcopter}: unchecked gain overgrowth as explained in \eqref{eq:gain_saturation}. This is, to our knowledge, the first sound fix to this problem that is validated by extensive experiments.

To develop AQSMC, we allow the switching gains $\mathbf{K}_\xi$ and $\mathbf{K}_q$ in \eqref{eq:att_ctrl_law} and \eqref{eq:kappa} to evolve over time as follows
\begin{equation}\label{eq:adaptation_law_2}
\begin{aligned}&
\dot{{K}}_{q_{ii}}=
\begin{cases}
 \Gamma_{q_{ii}} \hat{J}_{ii} \lvert s_{q_i}\rvert \tanh{(\frac{\lvert s_{q_i}\rvert }{\phi_{q_i}} - \epsilon_{q_i})} & \text{ if }\; {K}_{q_{ii}}>{K}_{q_{ii}}^\text{th}\\
 \mu_{q_i} & \text{ if }\; {K}_{q_{ii}}\leq {K}_{q_{ii}}^\text{th}
\end{cases}\\&
\dot{{K}}_{\xi_{ii}}=
\begin{cases}
 \Gamma_{\xi_{ii}}\rho_\xi \lvert s_{\xi_i}\rvert \tanh{(\frac{\lvert s_{\xi_i}\rvert }{\phi_{\xi_i}} - \epsilon_{\xi_i})} & \text{ if }\; {K}_{\xi_{ii}}>{K}_{\xi_{ii}}^\text{th}\\
 \mu_{\xi_i} & \text{ if }\; {K}_{\xi_{ii}}\leq {K}_{\xi_{ii}}^\text{th} 
\end{cases}
\end{aligned}
\end{equation}
Each gain is governed by four tunable parameters: $\Gamma = \diag(\Gamma_{ii}) > 0$, $\epsilon > 0$, $\mu > 0$, and $K^{\text{th}} > 0$. We now explain how these parameters govern the adaptation behavior. Later in Theorem 4, we analyze the stability, and in Section \ref{se:adaptive_experiments}, we provide practical insights into their impact via hardware tests.
Note that, for notational brevity, we omit subscripts in this section when the component-wise meaning is clear and applies to both attitude and position control.

The key term $\tanh(\frac{s}{\phi}-\varepsilon)$ changes sign at $|s|=\varepsilon \phi$, enabling ${K}$ to both increase and decrease during flight.
Therefore, ${K}$ increases when the error is large ($|s|>\varepsilon \phi$) and decreases when the error is small, allowing the controller to adapt aggressively only when needed.

Although this behavior is theoretically sound, our hardware experiments revealed that allowing ${K}$ to decay too much jeopardizes responsiveness to sudden disturbances. To address this, we introduce a threshold $K^\text{th}$, which sets the minimum allowable value for $K$. If ${K}$ drops below this threshold, it is increased at a constant rate $\mu$, ensuring the controller remains responsive and resilient to abrupt changes, without requiring high gains at all times.

Introducing \eqref{eq:adaptation_law_2} slightly modifies the attitude reference generation in 
\eqref{eq:ref_gen1}--\eqref{eq:trajectory_gen2}. In particular, $\dot{\mathbf{K}}_\xi$ and $\ddot{\mathbf{K}}_\xi$ in \eqref{eq:kappa_dot} and \eqref{eq:kappa_2dot} are no longer zero. We substitute \eqref{eq:adaptation_law_2} for $\dot{\mathbf{K}}_\xi$, and for $\ddot{\mathbf{K}}_{\xi}$, we have
\begin{equation}
\begin{aligned}
    \ddot{K}_{\xi_{ii}} = &\gamma_{\xi_{ii}}\rho_\xi \sgn{(s_{\xi_i})}\dot{s}_{\xi_i} \biggl(\tanh{(\frac{\lvert s_{\xi_i}\rvert }{\phi_{\xi_i}} - \epsilon_{\xi_i})}\\&
 +\frac{\lvert s_{\xi_i}\rvert}{\phi_{\xi_i}} \sech^{2}{(\frac{\lvert s_{\xi_i}\rvert }{\phi_{\xi_i}} - \epsilon_{\xi_i})} \biggr)
\end{aligned}
\end{equation}
for $K_{\xi_{ii}} > K_{\xi_{ii}}^{th}$, and  $\ddot{K}_{\xi_{ii}} = 0$ for $K_{\xi_{ii}} \leq K_{\xi_{ii}}^{th}$.
\begin{theorem}[AQSMC 6-DOF flight control]
Consider the error terms defined in \eqref{eq:error_definition}, with dynamics governed by \eqref{eq:err_dyn}, under Assumptions \ref{as:trajectory}--\ref{as:disturbance}. 
Adopt the sliding surfaces \eqref{eq:s_att} and \eqref{eq:s_pos}, the control laws \eqref{eq:att_ctrl_law}, \eqref{eq:f}, \eqref{eq:kappa}, and \eqref{eq:kappa_nonvanishing}. Let $\mathbf{K}_q$ and $\mathbf{K}_\xi$ evolve over time using \eqref{eq:adaptation_law_2}.
Suppose the the following conditions are satisfied for all $i = 1,2,3$, 
\begin{equation}\label{eq:gain_condition_adaptive}
\begin{aligned}
&\Lambda_{q_{ii}} > 0,\quad \phi_{q_i} > 0, \quad \Lambda_{\xi_{ii}} > 0,\quad \phi_{\xi_i} > 0,\\
&\Gamma_{q_{ii}} > 0,\quad \varepsilon_{q_i} > 0,\quad K_{q_{ii}}^{\text{th}} > 0, \quad \mu_{\xi_i} > 0, \\
&\Gamma_{\xi_{ii}} > 0,\quad \varepsilon_{\xi_i} > 0,\quad K_{\xi_{ii}}^{\text{th}} > 0, \quad \mu_{q_i} > 0, \\
\end{aligned}
\end{equation}
Obtain $\mathbf{q}_d$, $\boldsymbol{\omega}_d$, and $\boldsymbol{\alpha}_d$ from $\boldsymbol{\kappa}$ and $\psi_d$ via \eqref{eq:ref_gen1}--\eqref{eq:trajectory_gen2}.
% , assuming $\mathbf{h}_d \nparallel \mathbf{b}_{3d}$.
Then, $\boldsymbol{\xi}_e$, $\boldsymbol{\nu}_e$, $\mathbf{q}_e$ and $\boldsymbol{\omega}_e$ are all globally uniformly ultimately bounded.
\end{theorem}

\begin{proof}
Define $\Tilde{\mathbf{K}}_q = {\mathbf{K}}_q - \mathbf{K}_q^\star$ and $\Tilde{\mathbf{K}}_\xi = {\mathbf{K}}_\xi - \mathbf{K}_\xi^\star$, where $\mathbf{K}_q^\star$ and $\mathbf{K}_\xi^\star$ are two diagonal matrices satisfying the conditions for $\mathbf{K}_\xi$ and $\mathbf{K}_q$ in \eqref{eq:gain_condition} and \eqref{eq:gain_condition_pos}.
$\mathbf{K}_q^\star$ and $\mathbf{K}_\xi^\star$ are unknown, but their existence is guaranteed under Assumption \ref{as:disturbance}. 

Next, consider
\begin{equation}\label{eq:V_adapt}
    V_2 = \frac{1}{2}\left(\mathbf{s}^\top _{q} \mathbf{J}\mathbf{s}_{q} + \mathbf{s}^\top _{\xi}\mathbf{s}_{\xi} + {\tr} \left(\boldsymbol{\Gamma}^{-1}_{\xi}\Tilde{\mathbf{K}}^2_{\xi}+     \boldsymbol{\Gamma}^{-1}_{q}\Tilde{\mathbf{K}}^2_{q}\right)\right).
\end{equation}
Similar to the proof of Theorems 2 and 3, we take the generalized derivative of $V_2$, and apply the control laws and upper bounds of uncertain terms. We also add and subtract
$ \mathbf{s}^\top _{\xi}\rho_{\xi}\mathbf{K}^\star_{\xi}\tanh{(\mathbf{s}_{\xi}\oslash\boldsymbol{\phi}_{\xi})}$ and $\mathbf{s}^\top _{q}\hat{\mathbf{J}}\mathbf{K}^\star_q\tanh{(\mathbf{s}_{q}\oslash\boldsymbol{\phi}_{q})}$ to obtain
%=====Add and subtract those big terms in a separate equation=====%
% \begin{equation}\label{eq:v_dot_2_adapt}
% \begin{aligned}
% \dot{V}_2^\circ \leq\;&
% \lvert \mathbf{s}^\top _{\xi} \rvert \left( \bar{\boldsymbol{\delta}}_{\xi} + \bar{\mathbf{d}}_a \right) 
% -\mathbf{s}^\top _{\xi}\rho_{\xi}\hat{\mathbf{K}}_{\xi} \tanh{(\mathbf{s}_{\xi}\oslash\boldsymbol{\phi}_{\xi})} \\&
% + \textcolor{softblue}{\tr} \left(\boldsymbol{\Gamma}^{-1}_{\xi}\Tilde{\mathbf{K}}_{\xi}\dot{\hat{\mathbf{K}}}_{\xi} \right)
% +\lvert \mathbf{s}^\top _{q} \rvert \left( \bar{\boldsymbol{\delta}}_{q} + \bar{\mathbf{d}}_{\alpha} \right)  \\&
% - \mathbf{s}^\top _{q} \hat{\mathbf{J}}\hat{\mathbf{K}}_q\tanh{(\mathbf{s}_{q}\oslash\boldsymbol{\phi}_{q})} + 
% \textcolor{softblue}{\tr} \left(\boldsymbol{\Gamma}^{-1}_{q}\Tilde{\mathbf{K}}_{q}\dot{\hat{\mathbf{K}}}_{q}\right).
% \end{aligned}
% \end{equation}
% Add and substract $ \mathbf{s}^\top _{\xi}\rho_{\xi}\mathbf{K}^\star_{\xi}\tanh{(\mathbf{s}_{\xi}\oslash\boldsymbol{\phi}_{\xi})}$ and $\mathbf{s}^\top _{q}\hat{\mathbf{J}}\mathbf{K}^\star_q\tanh{(\mathbf{s}_{q}\oslash\boldsymbol{\phi}_{q})}$ and re-writing in terms of elements lead to 
%=====Add and subtract those big terms in a separate equation=====%
\begin{equation}\label{eq:v_dot_adaptive_1}
\begin{aligned}
\dot{V}_2^\circ\leq\;&
\sum_{i=1}^{3} \biggl( \lvert s_{\xi_i}\rvert (\bar{\delta}_{\xi_i}+ \bar{{d}}_{a_i})-  \rho_{\xi}K^\star_{\xi_{ii}}s_{\xi_i}\tanh{(\frac{s_{\xi_i}}{\phi_{\xi_i}})} \\& 
+\lvert s_{q_i}\rvert (\bar{\delta}_{q_i} + \bar{{d}}_{{\alpha}_i})- \hat{J}_{ii}K^\star_{q_{ii}}s_{q_i}\tanh{(\frac{s_{q_i}}{\phi_{q_i}})} \\&
- \rho_{\xi}\Tilde{K}_{\xi_{ii}}s_{\xi_i}\tanh(\frac{s_{\xi_i}}{\phi_{\xi_i}}) + \frac{1}{{\Gamma_{\xi_{ii}}}}\Tilde{K}_{\xi_{ii}}\dot{K}_{\xi_{ii}}  \\
& - \hat{J}_{ii}\Tilde{K}_{q_{ii}}s_{q_{i}}\tanh{(\frac{s_{q_i}}{\phi_{q_i}})}+
\frac{1}{{\Gamma_{q_{ii}}}}\Tilde{K}_{q_{ii}}\dot{K}_{q_{ii}} \biggr).
\end{aligned}
\end{equation}
According to the definition of $\mathbf{K}_q^\star$ and $\mathbf{K}_\xi^\star$ stated above, and the proof of Theorem 3, the first four terms on the right hand side of \eqref{eq:v_dot_adaptive_1} are bounded by $-c_5 \sqrt{V_2}+c_6$, with $c_5$ and $c_6$ as in \eqref{eq:70}. 
Furthermore, since $\tanh(\cdot)$ is odd and strictly increasing, we have $
 s  \tanh( s ) \ge  \lvert s \rvert \tanh(\lvert s \rvert  - \epsilon)$ for all $\epsilon > 0$.
Therefore, 
\begin{equation}\label{eq:v_dot_4_adapt}
\begin{aligned}
    \dot{V}_2^\circ \leq\;& -c_5 \sqrt{V_2} + c_6 -\sum_{i=1}^{3}\biggl( 
\rho_{\xi} \Tilde{K}_{\xi_{ii}} \lvert s_{\xi_i} \rvert \tanh(\frac{\lvert s_{\xi_i} \rvert }{\phi_{\xi_i}} - \epsilon_{\xi_i}) \\&
-\frac{1}{{\Gamma_{\xi_{ii}}}}\Tilde{K}_{\xi_{ii}} \dot{K}_{\xi_{ii}}  
+\hat{J}_{ii}\Tilde{K}_{q_{ii}}\lvert s_{q_{ii}}\rvert\tanh{(\frac{\lvert s_{q_{ii}}\rvert}{\phi_{q_i}}- \epsilon_{q_i}}) \\& -
\frac{1}{{\Gamma_{q_{ii}}}}   \Tilde{K}_{q_{ii}}\dot{K}_{q_{ii}}\biggr).
\end{aligned}
\end{equation}

Substituting \eqref{eq:adaptation_law_2} into \eqref{eq:v_dot_4_adapt} leads to three cases.

\textit{Case (i)} $K_{q_{ii}}>K_{q_{ii}}^{\text{th}}$ \textit{and} $K_{\xi_{ii}}>K_{\xi_{ii}}^{\text{th}}$ \textit{for} $i=1,2,3$. The entire summation term in \eqref{eq:v_dot_4_adapt} vanishes, leaving
\begin{equation}\label{eq:v_dot_5_adapt}
    \dot V_2^{\!\circ}\;\le\; -\,c_5\sqrt{V_2}+c_6.
\end{equation}

\textit{Case (ii)} $K_{q_{ii}}\le K_{q_{ii}}^{\text{th}}$ and $K_{\xi_{ii}}\le K_{\xi_{ii}}^{\text{th}}$ \textit{for} $i=1,2,3$. An additional bounded term appears on the right hand of side \eqref{eq:v_dot_5_adapt}.
This regime activates when $\lvert s_{q_{i}} \rvert < \varepsilon_{q_{i}}\phi_{q_{i}}$ and $\lvert s_{\xi_{i}} \rvert < \varepsilon_{\xi_{i}}\phi_{\xi_{i}}$; therefore, $\tanh(\frac{\lvert s_{q_{i}}\rvert}{\phi_{q_i}}-\epsilon_{q_i}) \le 0$ and $\tanh(\frac{\lvert s_{\xi_{i}}\rvert}{\phi_{\xi_i}}-\epsilon_{\xi_i}) \le 0$.
Note that $\tilde{K}_{q_{ii}} \le \lvert K_{q_{ii}}^{\text{th}} - K_{q_{ii}}^\star \rvert$ and $\tilde{K}_{\xi_{ii}} \le \lvert K_{\xi_{ii}}^{\text{th}} - K_{\xi_{ii}}^\star \rvert$.
Therefore,
\begin{equation}\label{eq:v_dot_6_adapt}
    \dot V_2^{\!\circ}\;\le\; -\,c_5\sqrt{V_2}+c_6+c_7
\end{equation}
where $c_7 = \sum_{i=1}^3 \biggl(
    \frac{\mu_{q_{i}}}
    {\Gamma_{q_{ii}}}\lvert K_{q_{ii}}^{\text{th}} - K_{q_{ii}}^\star \rvert+
    \frac{\mu_{\xi_{i}}}{\Gamma_{\xi_{ii}}}\lvert K_{\xi_{ii}}^{\text{th}} - K_{\xi_{ii}}^\star \rvert
    \biggr)$.

\textit{Case (iii) Mixed mode.}
If some gains are above their thresholds and some below, \eqref{eq:v_dot_6_adapt} still holds.

Therefore $V_2(t)$ enters a ball of radius $\left( \frac{c_6+c_7}{c_5} \right)^2$ in finite time 
$t^*= \max\{0,\dfrac{2}{c_5}\bigl(\sqrt{V_2(0)}-\tfrac{c_6+c_7}{c_5}\bigr)\}$, and remains within that ball thereafter, implying boundedness.

Showing the well-posedness of the attitude reference signals follows the same steps as in Theorem 3.
The only difference is that the adaptation laws are piecewise smooth, that is, continuous and differentiable on each interval between a finite set of switching instants, so $\boldsymbol{\kappa}$, $\dot{\boldsymbol{\kappa}}$, and $\ddot{\boldsymbol{\kappa}}$ are piecewise smooth as well.
Because the reference generator \eqref{eq:ref_gen1}–\eqref{eq:trajectory_gen2} involves only algebraic operations on these signals, the resulting $\mathbf{q}_d$, $\boldsymbol{\omega}_d$, and $\boldsymbol{\alpha}_d$ remain piecewise smooth, bounded, and free of singularities.
Consequently, the lumped-uncertainty term in \eqref{eq:delta_q_bar} remains uniformly bounded, satisfying the requirement of Theorem 2.
\end{proof}

% \vspace{-0.5\baselineskip}
\section{Experiments}\label{se:experiments}
% This section presents hardware experiments evaluating the proposed QSMC and AQSMC methods. We first assess QSMC against benchmark approaches, then demonstrate the advantages of AQSMC over QSMC.
\vspace{-0.2\baselineskip} 
\subsection{Experiment scenarios}
We start with attitude tracking experiments on a custom-built 3-DOF gimbal (Fig. \ref{fig:experiment_scenarios_gimbal}).
With this setup, we isolate the attitude controller and directly compare the attitude control performance of various methods.
We further carry out a sensitivity analysis in this scenario to examine the controllers’ behavior under different tuning configurations.
% \begin{figure*}[t]
%     \centering
%     \includegraphics[width=0.2\linewidth,trim={0cm 0cm 0cm 0.15cm},
%     clip]{Figures/General/gimbal.JPEG}
%     \includegraphics[width=0.28\linewidth]{Figures/General/trajectory.PNG}
%     \includegraphics[width=0.2\linewidth,trim={0cm 0cm 0cm 0.15cm},
%     clip]{Figures/General/gimbal.JPEG}
%     \includegraphics[width=0.2\linewidth,trim={2cm 1cm 2cm 9.4cm},
%     clip]{Figures/General/drone_with_mocap.jpg}
%     \caption{\textcolor{softblue}{Experimental setups used for controller validation. (Left) Custom-built 3-DOF gimbal platform for isolated attitude control testing under significant disturbances.(Center) Top-view trajectory of the quadrotor following a lemniscate path, captured during a free-flight test. (Right) Motion-capture-equipped flight arena used for 6-DOF free-flight experiments, including trajectory tracking and dynamic throw launches.}}
%     \label{fig:experiment_scenarios}
% \end{figure*}
\begin{figure*}[t]
    \centering
    \begin{subfigure}[b]{0.2\linewidth}
        \includegraphics[width=\linewidth, trim={0cm 0cm 0cm 0.15cm}, clip]{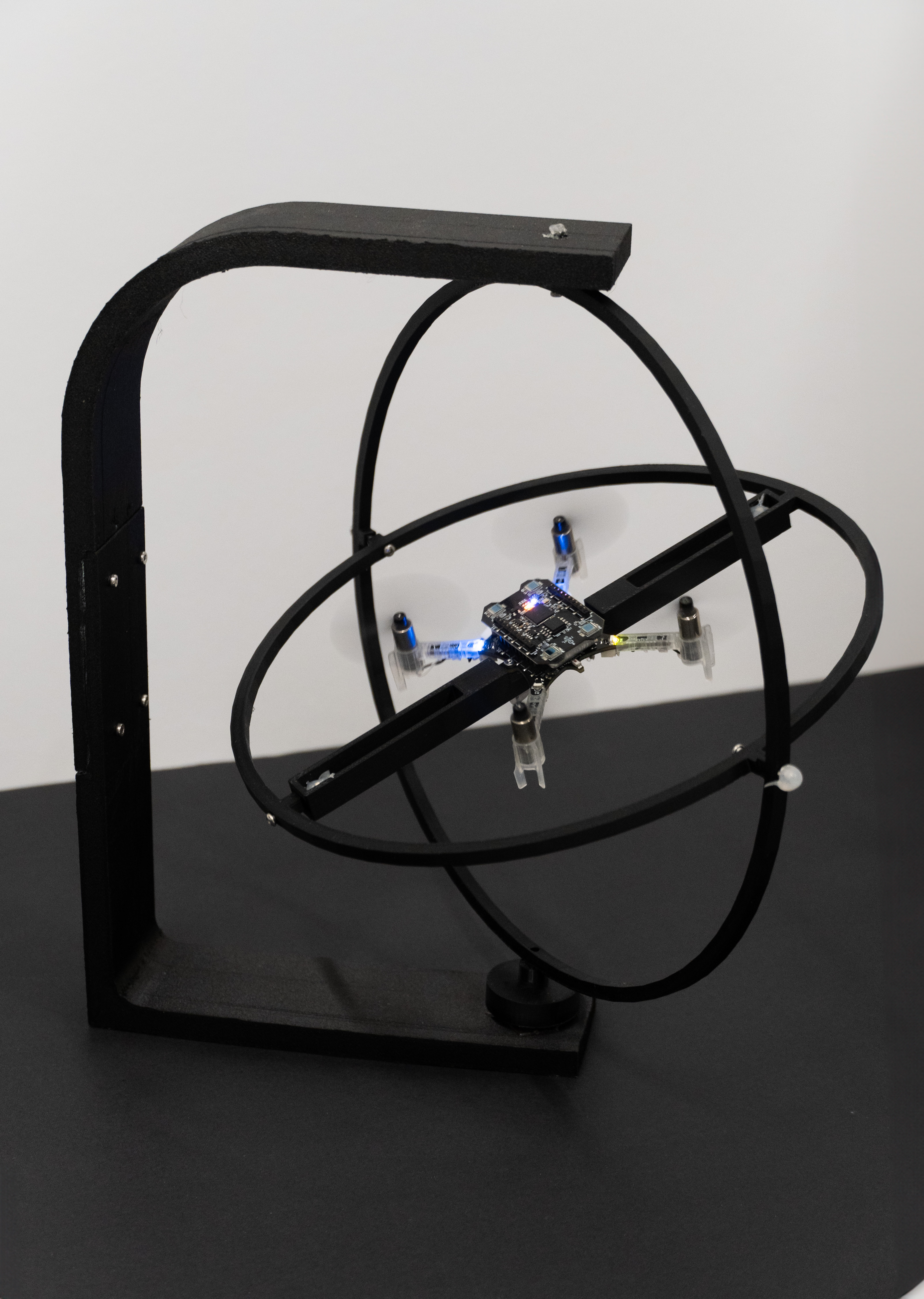}
        % \caption{\textcolor{softblue}{3-DOF gimbal platform}}
        \caption{}
        \label{fig:experiment_scenarios_gimbal}
    \end{subfigure}
    \hfill
    \begin{subfigure}[b]{0.2\linewidth}
        \includegraphics[width=\linewidth, trim={2cm 1cm 2cm 9.4cm}, clip]{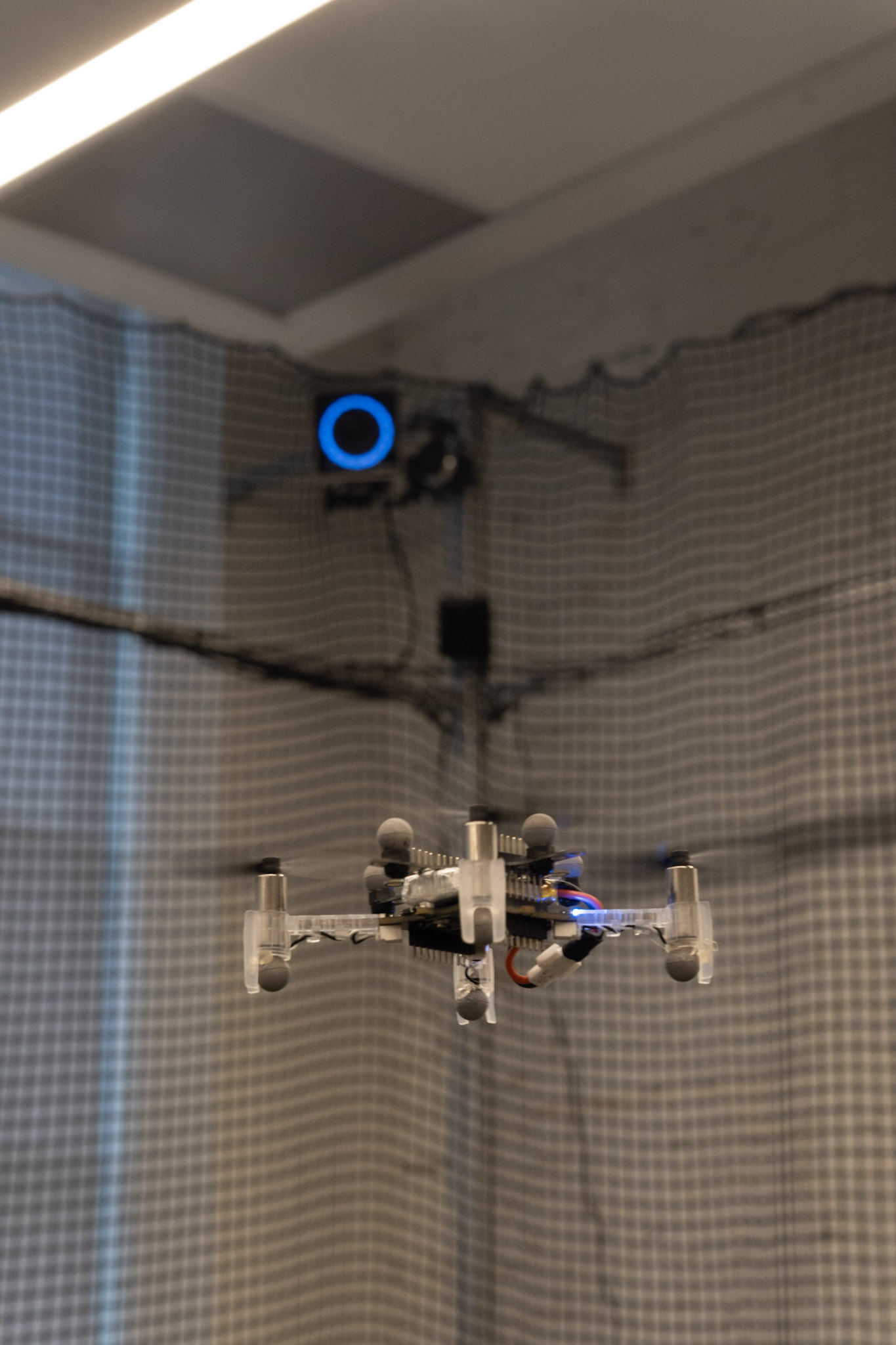}
        % \caption{\textcolor{softblue}{Motion-capture-equipped arena}}
        \caption{}
        \label{fig:experiment_scenarios_MoCap}
    \end{subfigure}
    \hfill
    \begin{subfigure}[b]{0.28\linewidth}
        \includegraphics[width=\linewidth]{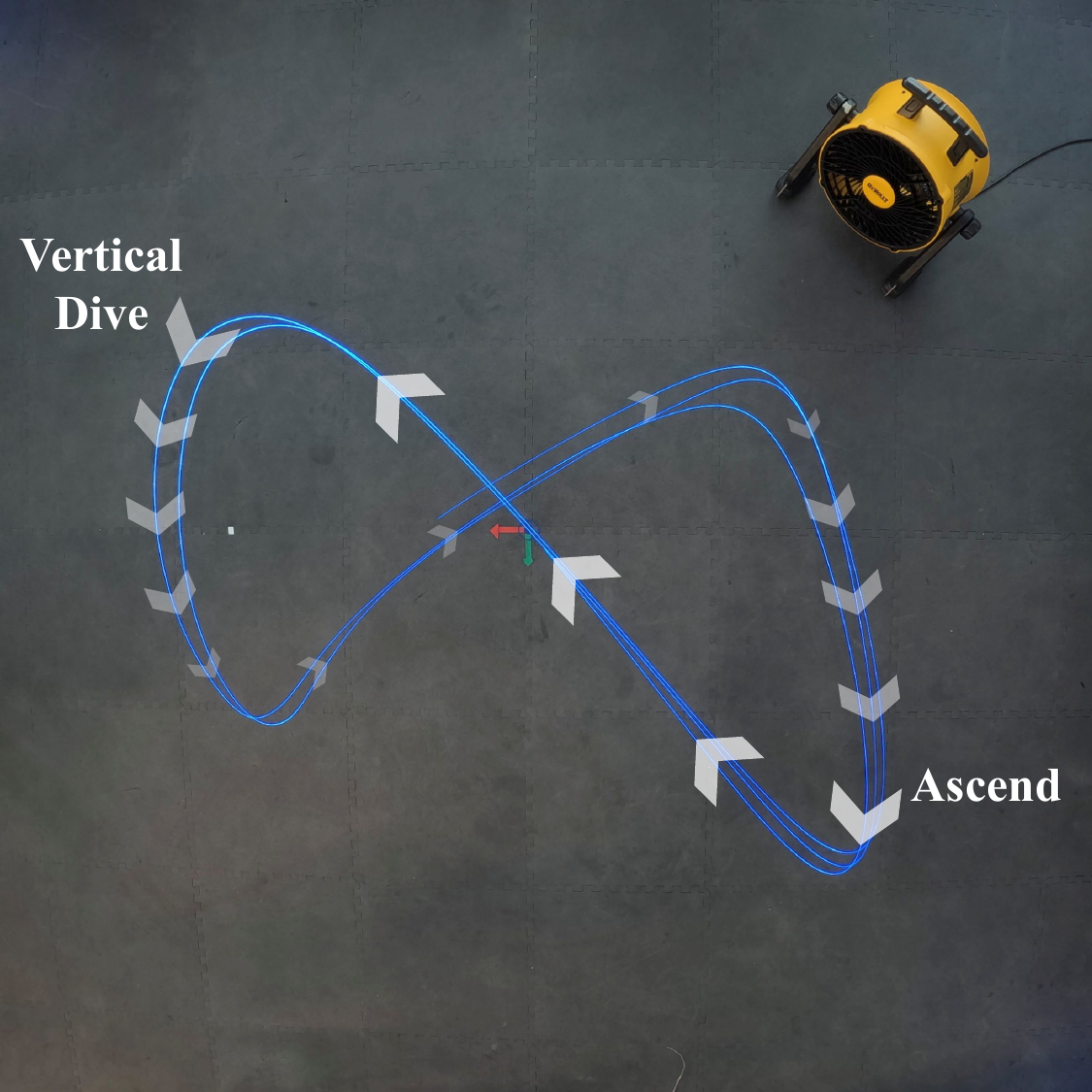}
        % \caption{\textcolor{softblue}{Top-view trajectory (lemniscate path)}}
        \caption{}
        \label{fig:experiment_scenarios_freeflight}
    \end{subfigure}
    \hfill
    \begin{subfigure}[b]{0.212\linewidth}
        \includegraphics[width=\linewidth]{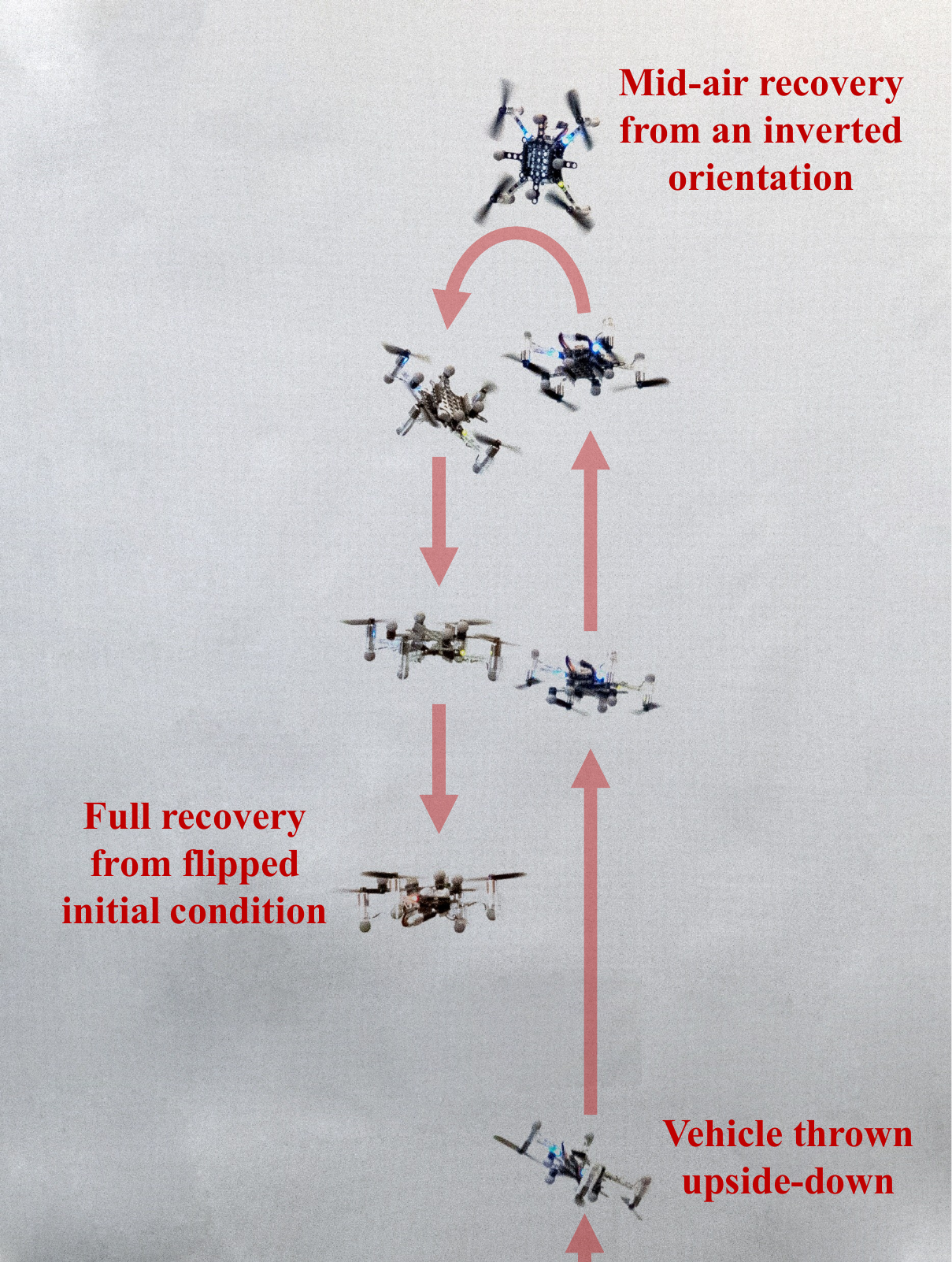}
        % \caption{\textcolor{softblue}{Throw launch experiment}}
        \caption{}
        \label{fig:experiment_scenarios_throw}
    \end{subfigure}

    \caption{
    Experimental setups used for controller validation:
    (\subref{fig:experiment_scenarios_gimbal}) custom-built 3-DOF gimbal platform for isolated attitude control testing;
    (\subref{fig:experiment_scenarios_MoCap}) 6-DOF free-flight using motion capture system;
    (\subref{fig:experiment_scenarios_freeflight}) lemniscate trajectory in the presence of wind disturbance;
    (\subref{fig:experiment_scenarios_throw}) dynamic throw-launch experiments.
    }
    \label{fig:experiment_scenarios}
\end{figure*}

Next, we conduct free flight tests, in an indoor environment equipped with motion capture system (Fig. \ref{fig:experiment_scenarios_MoCap}), tracking an aggressive lemniscate trajectory (Fig. \ref{fig:experiment_scenarios_freeflight}), and sensitivity analysis in this setup.
Finally, we carry out dynamic throw launch tests (Fig. \ref{fig:experiment_scenarios_throw}) to evaluate QSMC and AQSMC under abrupt initiation conditions.
% \textcolor{softblue}{Figure \ref{fig:experiment_scenarios} illustrates the above experiment scenarios.}
\subsection{Benchmark methods and controller parameters values}
Given our focus on classical nonlinear control methods, we benchmark our results against common controllers in this category, including Euler-based SMC (ESMC) \cite{xu2006sliding}, the geometric tracking controller (GTC) \cite{lee2010geometric}, and the proportional derivative- (PD-) quaternion attitude controller with an SMC position controller (QPD). 
{Appendix \ref{se:appendixA}} provides the details of each benchmark method. We fine-tuned each method for each experiment scenario to the best of our ability. Appendix \ref{se:appendixB} presents the parameter values used in our experiments.

% Other nonlinear approaches were also considered, such as robust GTC \cite{lee2013nonlinear},
% INDI-based controllers \cite{sun2022comparative}, and an NMPC implementation with ACADOS\cite{barroscarlos2020} . 
% However, the robust GTC failed in our tests due to actuator saturation and its reliance on negative thrusts; INDI required an optimal control allocation strategy, not
% feasible on the Crazyflie, and NMPC could only run off-board at around 100~Hz,
% which was too slow for attitude control. 
% These limitations prevented a reliable experimental evaluation of these
% additional controllers.

We also tested robust GTC (RGTC) \cite{lee2013nonlinear} and differential flatness with an INDI attitude controller (DF-INDI) \cite{sun2022comparative}, but both frequently failed in our test scenarios and were excluded from the reported results. RGTC required negative thrust values, consistent with \cite{lee2013nonlinear}, which our test vehicle could not produce. DF-INDI exhibited degraded performance during aggressive maneuvers and relied on optimization-based control allocation such as that used in \cite{sun2022comparative}, unlike the simpler inversion of \eqref{eq:G_1} used by other methods. This added complexity reduces DF-INDI’s practicality, and any performance gains from advanced allocation would likely extend to other controllers as well.

\subsection{Performance metrics}
To quantitatively compare the controllers, we compute the root mean square (RMS) of tracking errors and control efforts for each method.
As the objective of gimbal experiments is attitude tracking, we calculate the RMS of $\mathbf{q}_e$, denoted by $q_{e_\text{RMS}}$.
However, for free flight experiments, the objective is to track $\boldsymbol{\xi}_d$ and $\psi_d$; therefore, we calculate the RMS of $\boldsymbol{\xi}_e$, and $\psi_e = \psi - \psi_d$, denoted by $\xi_{e_\text{RMS}}$, and $\psi_{e_\text{RMS}}$.

To evaluate the control efforts of different controllers, we investigate the normalized pulse width modulation (NPWM) signal sent to the quadrotor motors.
An NPWM of zero represents no power, and an NPWM of 1 represents full throttle, providing a reliable indicator of the energy expended for vehicle control.
Therefore, we define the following metric
\begin{equation}
    \text{NPWM}_\text{RMS} = \sum_{i=1}^{4}\text{NPWM}_{i_\text{RMS}},
\end{equation}
where $\text{NPWM}_{i_\text{RMS}}$ indicates the RMS value of the NPWM of the $i$-th motor, for the duration of a flight test.

% \begin{table*}[t]
% \caption{Nominal values of design parameters for all control methods across all experimental scenarios}
% \label{tab:control_params}
% \centering
% \begin{tabular}{@{}lllll@{}}
% \toprule
% Name  & Control law & Gimballed experiments 0.2 & Gimballed experiments 0.5 & Free flight experiments\\ \midrule
% QSMC & $\boldsymbol{\tau} = \hat{\mathbf{J}}\dot{\boldsymbol{\omega}}_{d} + \boldsymbol{\omega} \times \hat{\mathbf{J}}\boldsymbol{\omega} -\hat{\mathbf{J}}  \boldsymbol{\Lambda}_{q} \operatorname{sgn_{+}}(q_{w_e})\dot{\Vec{\mathbf{q}}}_e
%     -\hat{\mathbf{J}} \mathbf{K}_q \tanh\left(\mathbf{s}_q\oslash\boldsymbol{\phi}_q \right)$ & \texttt{Placeholder} & \texttt{Placeholder} & \texttt{Placeholder}\\
% AQSMC &\texttt{Placeholder} & \texttt{Placeholder} & \texttt{Placeholder} & \texttt{Placeholder}\\
% QPD &\texttt{Placeholder} & \texttt{Placeholder} & \texttt{Placeholder} & \texttt{Placeholder}\\
% ESMC &\texttt{Placeholder} & \texttt{Placeholder} & \texttt{Placeholder} & \texttt{Placeholder}\\
% GTC &\texttt{Placeholder} & \texttt{Placeholder} & \texttt{Placeholder} & \texttt{Placeholder}\\
%  \bottomrule
% \end{tabular}
% \end{table*}
% parameters listed in Table \ref{tab:vehicle_params}.
\subsection{Quadrotor used for the experiments}
We used the Bitcraze Crazyflie 2.1 nano quadrotor, with the following parameters: 
$\beta = 45^\circ$, $c_t = 2.88 \times 10^{-8}\;[\mathrm{N/s^2}]$, $c_q = 7.24 \times 10^{-10}\;[\mathrm{Nm/s^2}]$, $l = 92\;[\mathrm{mm}]$, $\hat{m} = 32\;[\mathrm{g}]$, $\hat{\mathbf{J}} = \mathrm{diag}(1.66,1.66,2.93) \times 10^{-5}\;[\mathrm{kg\,m^2}]$. 
The vehicle’s low mass and inertia make it highly sensitive to disturbances, serving as an excellent testbed for evaluating controller robustness.

We implemented all controllers in C and integrated them into the vehicle firmware running on an onboard STM32F405 microcontroller. The position and attitude controllers ran at 250 Hz and 500 Hz, respectively.

For the gimbal experiments, we used onboard inertial measurements. For free flight, we used a motion capture system for measuring attitude and position.
For real-time data communication and logging, we used the Crazyradio 2.0 module.

% \begin{table}[t]
% \caption{\textcolor{red}{Vehicle parameters values}}
% \label{tab:vehicle_params}
% \centering
% \begin{tabular}{@{}lll@{}}
% \toprule
% & Parameter  & Value \\ \midrule
% & $\hat{m}$ & $32\;[g]$ \\
% & ${\hat{\mathbf{J}}} $ & ${\rm{diag}}\left( {1.66,1.66,2.93} \right) \times {10^{ - 5}}\;[kg{m^2}] $\\
% & $c_t$ & $ 2.88 \times {10^{ - 8}}\;[\frac{N}{s^2}] $ \\
% & $c_q$ & $ 7.24 \times {10^{ - 10}}\;[\frac{Nm}{s^2}] $ \\
% & $l$ & $92\;[mm]$ \\
% & $\beta$ & $45^{\circ}$\\
% \bottomrule
% \end{tabular}
% \end{table}
\subsection{Gimballed attitude control}
The gimbal experiments aimed to evaluate each method’s attitude controller independently of its position controller.

Despite our efforts to minimize the gimbal’s weight and friction, it still produced relatively large moments for our nano quadrotor, requiring high controller gains compared to free flight (Tab. \ref{tab:control_params}).
However, since the moments were unmodeled, they acted as large unknown disturbances, making the gimbal an effective platform for evaluating controller robustness..

To ensure consistent initial conditions, we set the desired attitude to the origin at the start of each trial and let the vehicle to stabilize for 10 $[\text{s}]$ before applying the desired trajectory.

\subsubsection{Scenario 1}
We began with a simple desired trajectory: $\boldsymbol{\eta}_d = 0.2 \left[\sin(0.2\pi t), \cos(0.2 \pi t), 0\right]^\top $.
Figures \ref{fig:gimbal_2_err} and \ref{fig:gimbal_2_pwm} show $\mathbf{q}_d$ and NPWM, with performance metrics tabulated in Tab. \ref{tab:performance_metrics_gimbal}.

\begin{figure}[t]
    \centering
    \includegraphics[width=1\linewidth]{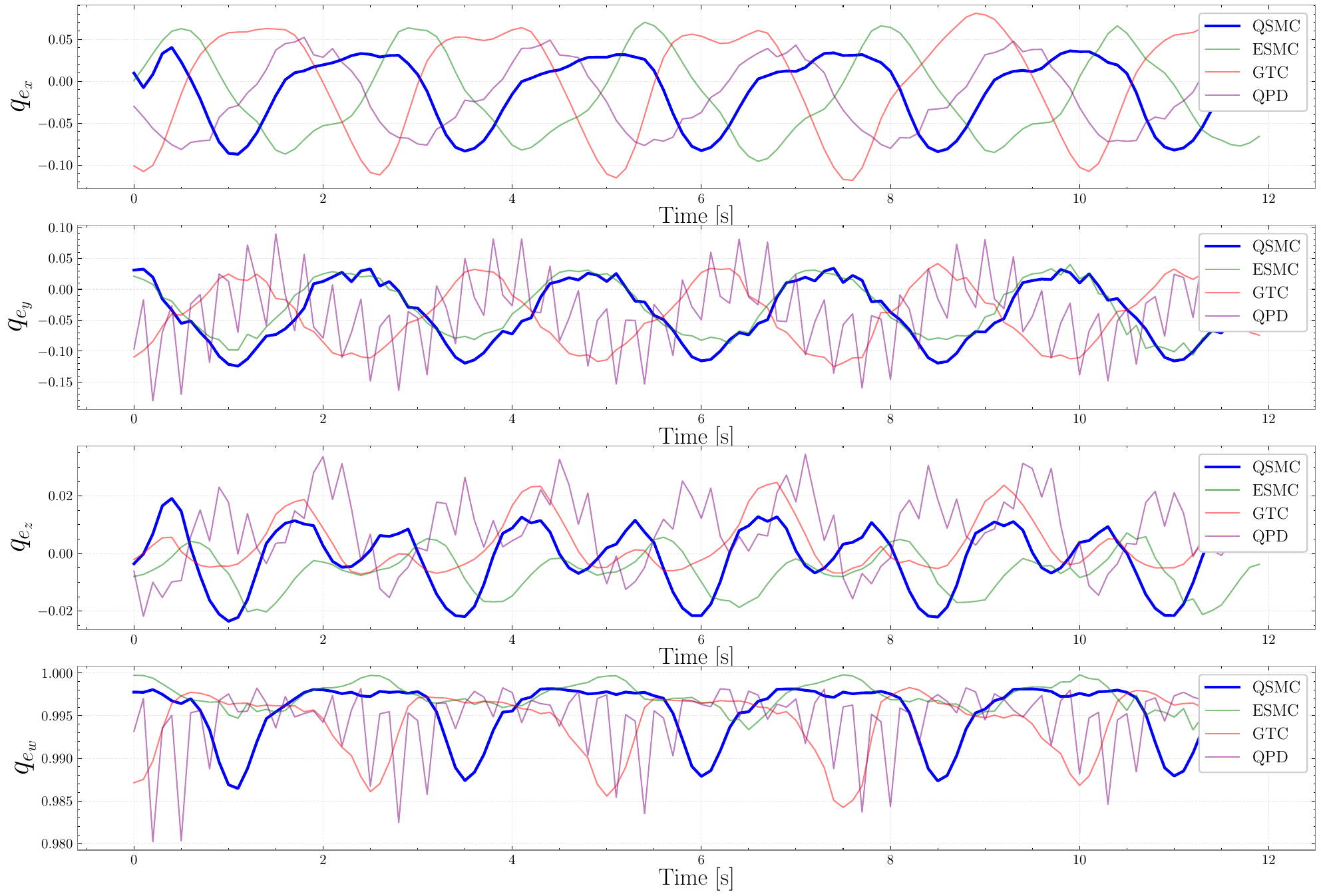}
    \caption{$\mathbf{q}_e$ in gimballed attitude control -- Scenario 1}
    \label{fig:gimbal_2_err}
\end{figure}
\begin{figure}[t]
    \centering
    \includegraphics[width=1\linewidth]{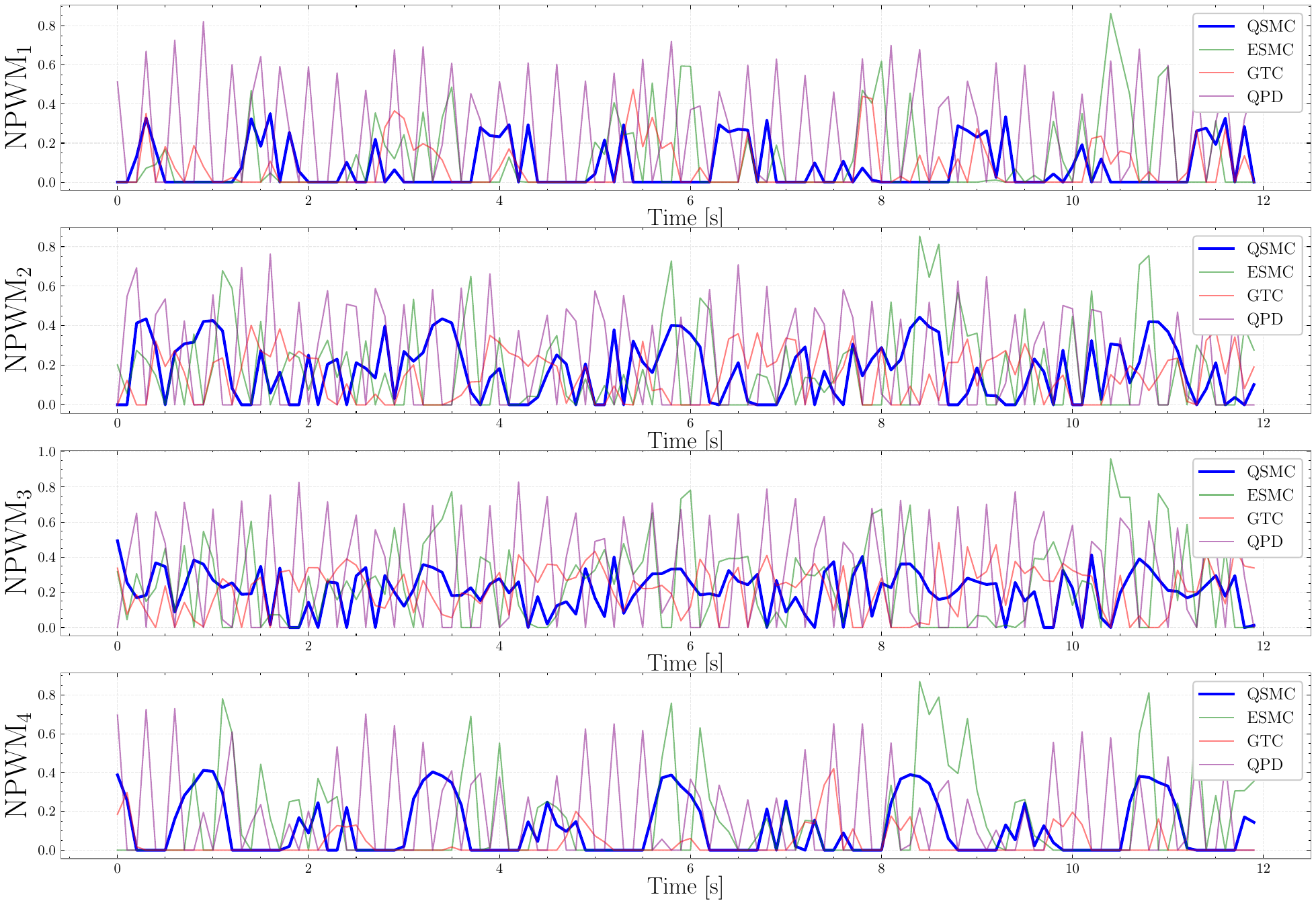}
    \caption{$\text{NPWM}_i$ in gimballed attitude control -- Scenario 1}
    \label{fig:gimbal_2_pwm}
\end{figure}

% gains that I found during fine tunning
% \begin{table*}[t]
%     \centering
%     \caption{Performance metrics in gimballed attitude control}
%     \label{tab:performance_metrics_gimbal}
%     \begin{tabular}{lcccc}
%         \toprule
%         & \multicolumn{2}{c}{Scenario 1} & \multicolumn{2}{c}{Scenario 2}\\
%         \cmidrule(lr){2-3} \cmidrule(lr){4-5}
%         Controller & $q_{e_\text{RMS}}$ & $\text{NPWM}_\text{RMS}$ & $q_{e_\text{RMS}}$ & $\text{NPWM}_\text{RMS}$ \\
%         \midrule
%         QSMC & 0.0846 $\pm$ 0.0454 & 0.4144 $\pm$ 0.1379 & 0.2058 $\pm$ 0.1165 & 0.8775 $\pm$ 0.3175 \\
%         ESMC & 0.0744 $\pm$ 0.0395 & 0.5821 $\pm$ 0.2243 & Unstable & Unstable \\
%         GTC  & 0.1080 $\pm$ 0.0596 & 0.2948 $\pm$ 0.0872 & 0.2834 $\pm$ 0.1575 & 0.6172 $\pm$ 0.2143 \\
%         QPD  & 0.0957 $\pm$ 0.0518 & 0.5019 $\pm$ 0.1722 & 0.2577 $\pm$ 0.1449 & 0.9720 $\pm$ 0.3607 \\
%         \bottomrule
%     \end{tabular}
% \end{table*}

% best gains
\begin{table*}[t]
    \centering
    \caption{Performance metrics in gimballed attitude control (mean $\pm$ standard deviation)}
    \label{tab:performance_metrics_gimbal}
    \begin{tabular}{lcccc}
        \toprule
        & \multicolumn{2}{c}{Scenario 1 -- $\boldsymbol{\eta}_d = 0.2 \left[\sin(0.2\pi t), \cos(0.2 \pi t), 0\right]^\top $} & \multicolumn{2}{c}{Scenario 2 -- $\boldsymbol{\eta}_d = 0.5 \left[\sin(0.2\pi t), \cos(0.2 \pi t), 0\right]^\top $}\\
        \cmidrule(lr){2-3} \cmidrule(lr){4-5}
        Controller & $q_{e_\text{RMS}}$ & $\text{NPWM}_\text{RMS}$ & $q_{e_\text{RMS}}$ & $\text{NPWM}_\text{RMS}$ \\
        \midrule
        QSMC & 0.0758 $\pm$ 0.0404 & 0.4013 $\pm$ 0.1316 &    0.1590 $\pm$ 0.0904 & 0.7231 $\pm$ 0.2648 \\
        ESMC & 0.0744 $\pm$ 0.0395 & 0.5821 $\pm$ 0.2243 &    Unstable & Unstable \\
        GTC  & 0.0903 $\pm$ 0.0507 & 0.3536 $\pm$ 0.1115 &    0.2322 $\pm$ 0.1331 & 0.6029 $\pm$ 0.2167 \\
        QPD  & 0.0855 $\pm$ 0.0471 & 0.7220 $\pm$ 0.2669 &    0.2260 $\pm$ 0.1296 & 0.8082 $\pm$ 0.3065 \\
        \bottomrule
    \end{tabular}
\end{table*}

As shown in Fig. \ref{fig:gimbal_2_err}, all controllers maintained small tracking errors throughout the mission. However, the sliding mode controllers, QSMC and ESMC, outperformed QPD and GTC, as confirmed by the $q_{e_\text{RMS}}$ values in Tab. \ref{tab:performance_metrics_gimbal}.
This can be attributed to the robustness of SMC that enables more effective compensation of the unknown gimbal moments.

While the $q_{e_\text{RMS}}$ values for QSMC and ESMC are comparable, the $\text{NPWM}_{\text{RMS}}$ values indicate that QSMC achieves similar tracking performance with 32\% less mean control effort compared to ESMC. This efficiency stems from QSMC’s ability to leverage the inherent characteristics of the $\mathbb{S}^3$ manifold, resulting in smoother and more precise control signals.

The GTC displays the largest tracking errors but requires the least control effort. Its attitude controller, as described in \eqref{eq:gtc_attitude_control_law}, is a nonlinear PD controller on $\mathrm{SO(3)}$ that depends on an accurate model of the vehicle’s dynamics.
This becomes a major limitation here, where unmodeled gimbal moments introduce substantial disturbances, causing the GTC to struggle.

As for the QPD, it demonstrates a slightly smaller tracking error compared to GTC, but with the largest control effort among all methods. 
We achieved QPD's lowest tracking errors by setting large values for $\mathbf{K}_P$, given in Tab. \ref{tab:control_params}.
While such large gains managed to stabilize the vehicle in the face of unmodeled gimbal moments, it created large oscillatory control signals, evident from Fig. \ref{fig:gimbal_2_pwm}.

% With respect to the control efforts shown in Fig. \ref{fig:gimbal_2_pwm}, QSMC and GTC demonstrate relatively smaller and smoother control signals since both controllers leverage the inherent characteristics of $\mathbb{S}^3$ and $\mathrm{SO(3)}$ for the vehicle control. On the other hand, ESMC suffers from large fluctuations, and overall larger control efforts compared to QSMC. Comparing the mean $\text{NPWM}_\text{RMS}$ of the two methods indicate more than 43\% larger value for ESMC.

% The GTC demonstrated moderate tracking accuracy but required higher control effort than QSMC, as evident from Figure . The QPD, being the simplest controller, achieved the largest tracking errors and required the highest control effort among all methods.

\subsubsection{Scenario 2}
Next, we set the desired attitude trajectory to $\boldsymbol{\eta}_d = 0.5 \left[\sin(0.2\pi t), \cos(0.2 \pi t), 0\right]^\top $.
Figures \ref{fig:gimbal_5_err} and \ref{fig:gimbal_5_pwm} present the $\mathbf{q}_e$ and NPWM signals, with metrics presented in Tab. \ref{tab:performance_metrics_gimbal}.

\begin{figure}
    \centering
    \includegraphics[width=1\linewidth]{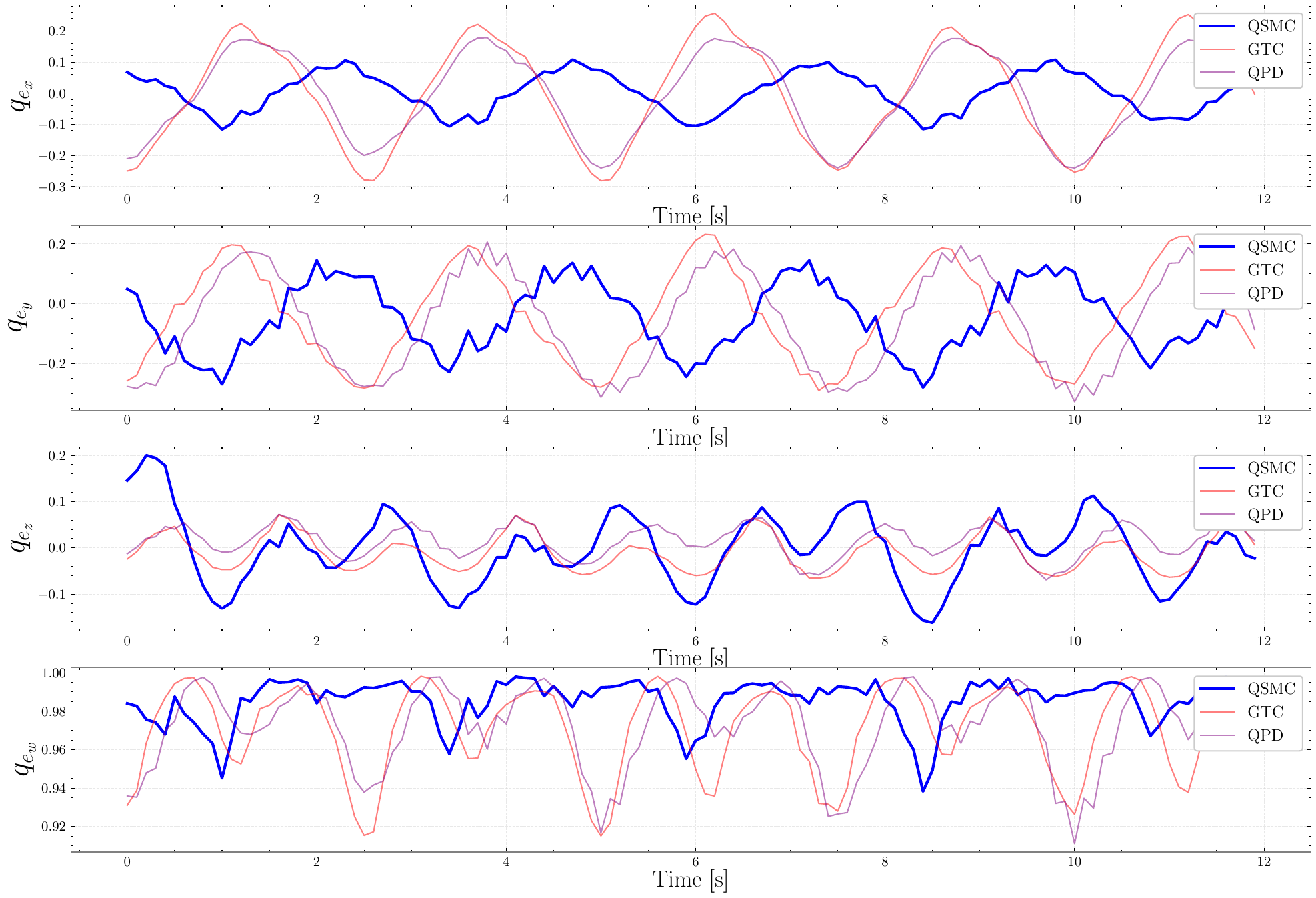}
    \caption{$\mathbf{q}_e$ in gimballed attitude control -- Scenario 2}
    \label{fig:gimbal_5_err}
% \end{figure}
% \begin{figure}
    \centering
    \includegraphics[width=1\linewidth]{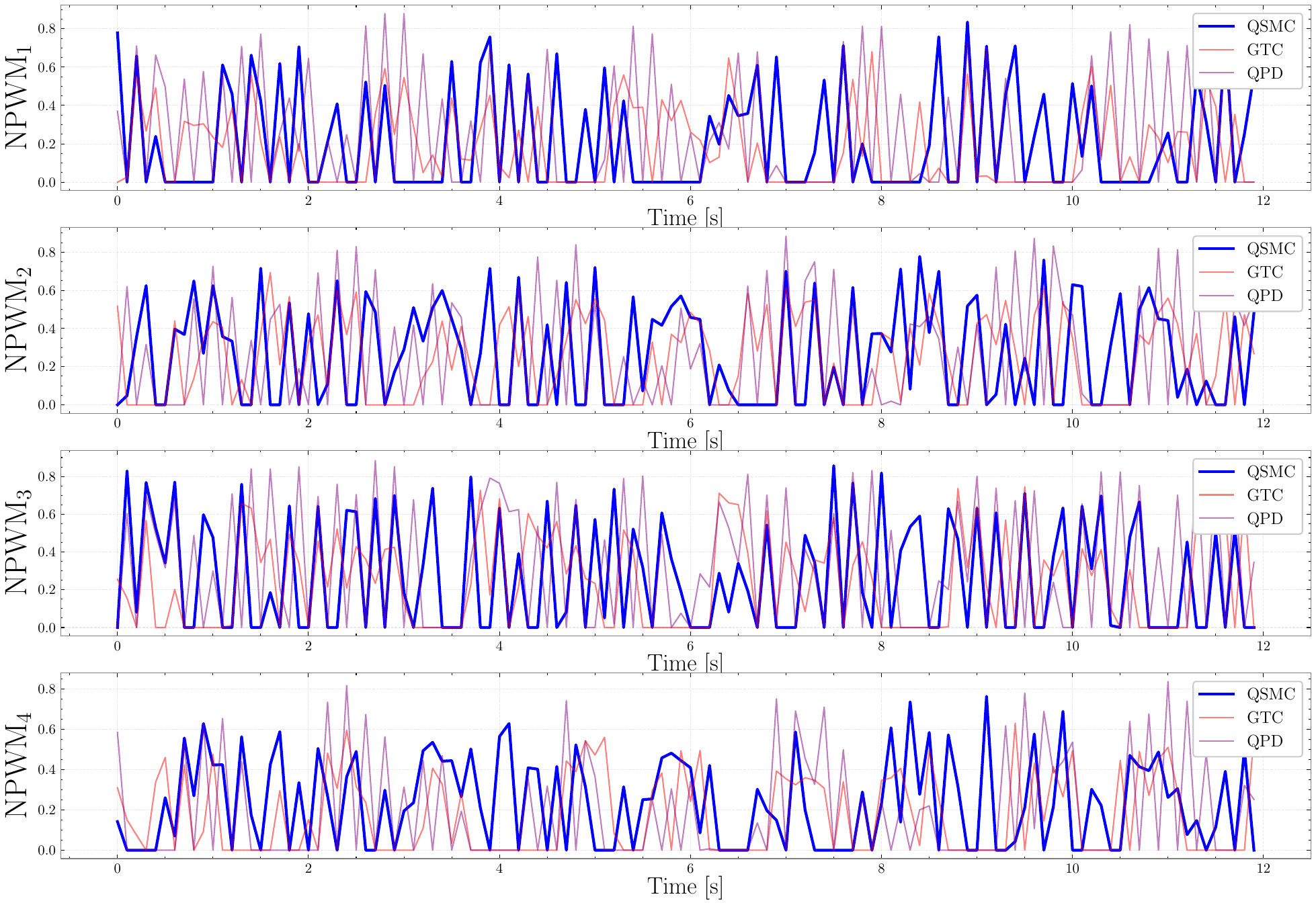}
    \caption{$\text{NPWM}_i$ in gimballed attitude control -- Scenario 2}
    \label{fig:gimbal_5_pwm}
\end{figure}

The larger roll and pitch angles in this scenario challenge the validity of simplified dynamics \eqref{eq:simplified_rot_dyn} used in the ESMC.
Despite several hours of experimentation and tuning followed by the sensitivity analysis to be detailed in the next section, the ESMC did not generate stable results.
This highlights the limitation of this method for quadrotors that has been overlooked in the literature.
Note that this limitation is different than the gimbal lock issue of Euler angles and directly stems from the simplification process explained in Section \ref{se:related_work}.

All other controllers demonstrate stable results in this scenario; however, the tracking errors have increased as the mission has become more challenging.
Examining the metrics in Tab. \ref{tab:performance_metrics_gimbal}, we observe similar trends to those seen in the first set of gimbal experiments. However, with the increase in the difficulty of the mission, the differences between the controllers are more pronounced.
The QSMC features 32\% lower mean tracking error, but 19\% larger mean control effort compared to the GTC, and 30\% lower mean tracking error, and 11\% lower mean control effort compared to the QPD.
This highlights the benefits of QSMC, achieving low tracking errors without a significant increase in the control effort.

% from the empirically optimized values

% \begin{table*}[htbp]
%     \centering
%     \footnotesize
%     \caption{Performance Metrics for Different Controllers}
%     \label{tab:performance_metrics_gimbal}
%     \begin{tabular}{l|cc|cc|ccc}
%         \toprule
%         & \multicolumn{2}{c}{Gimbal 0.2 (mean $\pm$ SD)} & \multicolumn{2}{c}{Gimbal 0.5 (mean $\pm$ SD)} & \multicolumn{3}{c}{Position (mean $\pm$ SD)} \\
%         \cmidrule(lr){2-3} \cmidrule(lr){4-5} \cmidrule(lr){6-8}
%         Controller & $\mathbf{q}_{e_{rms}}$ & $NPWM_{rms}$ & $\mathbf{q}_{e_{rms}}$ & $NPWM_{rms}$ & $\boldsymbol{\xi}_{e_{rms}}$ & $\psi_{e_{rms}}$ & $NPWM_{rms}$ \\
%         \midrule
%         QSMC & 0.0846 $\pm$ 0.0454 & 0.4144 $\pm$ 0.1379 & 0.2058 $\pm$ 0.1165 & 0.8775 $\pm$ 0.3175 & 0.1934 $\pm$ 0.1115 & 0.1573 $\pm$ 0.1569 & 1.5020 $\pm$ 0.1013 \\
%         ESMC & 0.0744 $\pm$ 0.0395 & 0.5821 $\pm$ 0.2243 & --- & --- & 0.2200 $\pm$ 0.1254 & 0.9022 $\pm$ 0.8925 & 1.5177 $\pm$ 0.1419 \\
%         GTC  & 0.1080 $\pm$ 0.0596 & 0.2948 $\pm$ 0.0872 & 0.2834 $\pm$ 0.1575 & 0.6172 $\pm$ 0.2143 & 0.3852 $\pm$ 0.2037 & 0.0526 $\pm$ 0.0523 & 1.4630 $\pm$ 0.1036 \\
%         QPD  & 0.0957 $\pm$ 0.0518 & 0.5019 $\pm$ 0.1722 & 0.2577 $\pm$ 0.1449 & 0.9720 $\pm$ 0.3607 & 0.2531 $\pm$ 0.1460 & 0.1122 $\pm$ 0.1121 & 1.5405 $\pm$ 0.1230 \\
%         \bottomrule
%     \end{tabular}
% \end{table*}

\begin{figure}[t]
    \centering
    \includegraphics[width=0.49\linewidth]{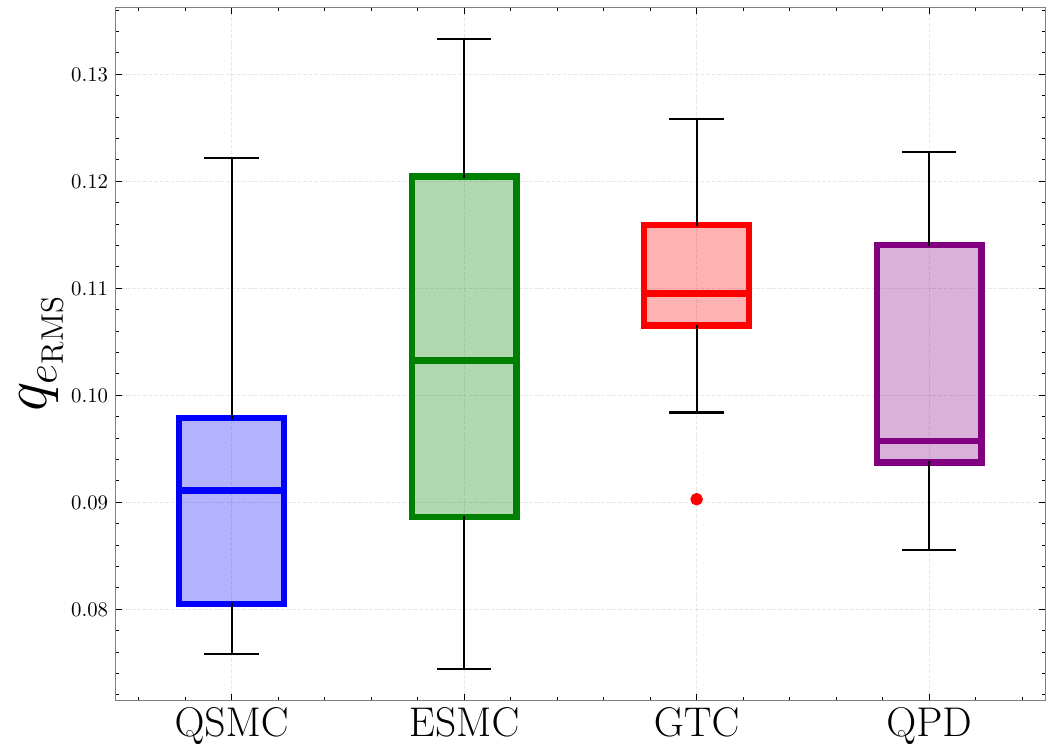}
    \includegraphics[width=0.48\linewidth]{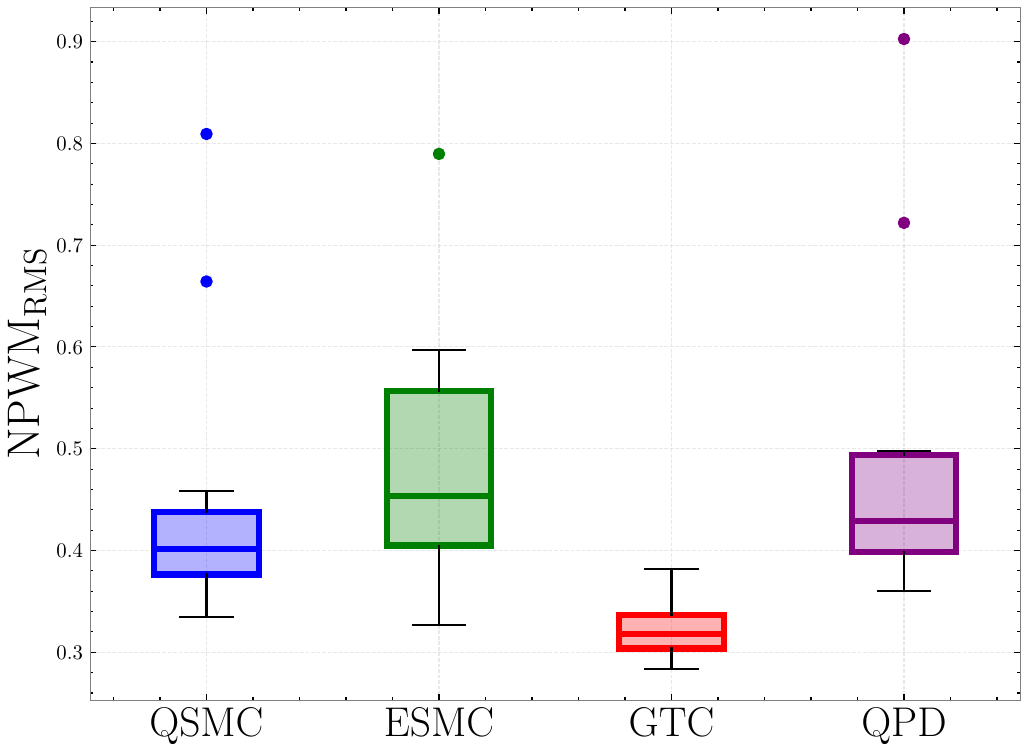}
    \caption{Box plots $q_{e_\text{RMS}}$ and $\text{NPWM}_{\text{RMS}}$ in sensitivity analysis of gimballed attitude control -- Scenario 1}
    \label{fig:s_box_2_q}
\end{figure}
\begin{figure}[t]
    \centering
    \includegraphics[width=0.49\linewidth]{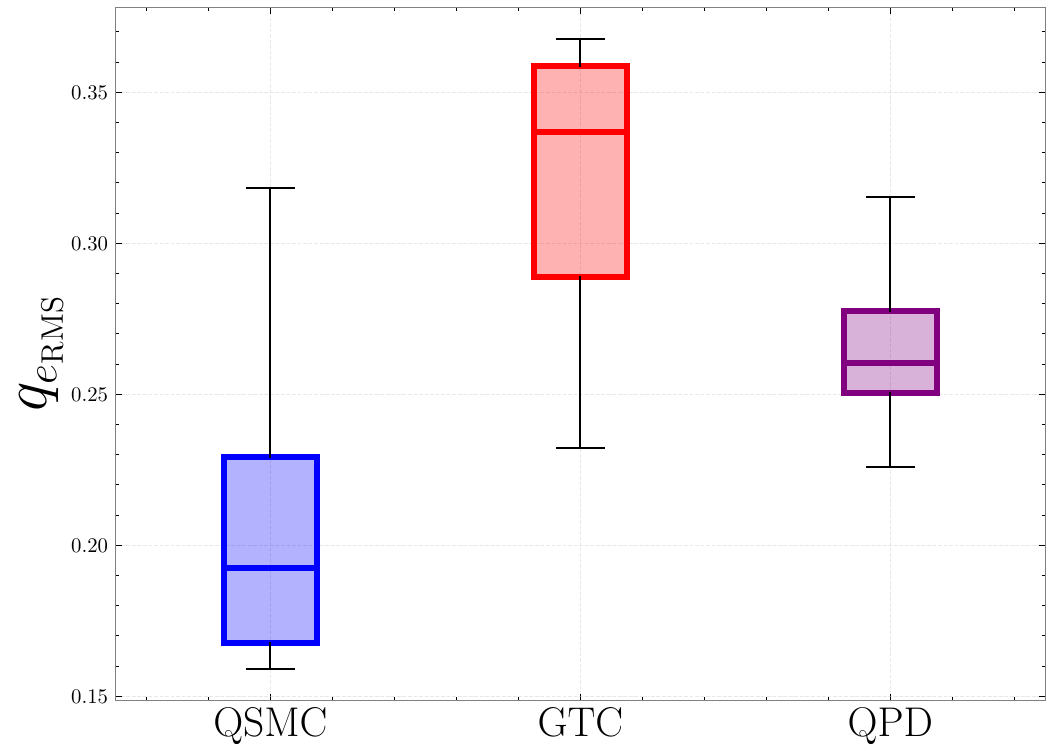}
    \includegraphics[width=0.48\linewidth]{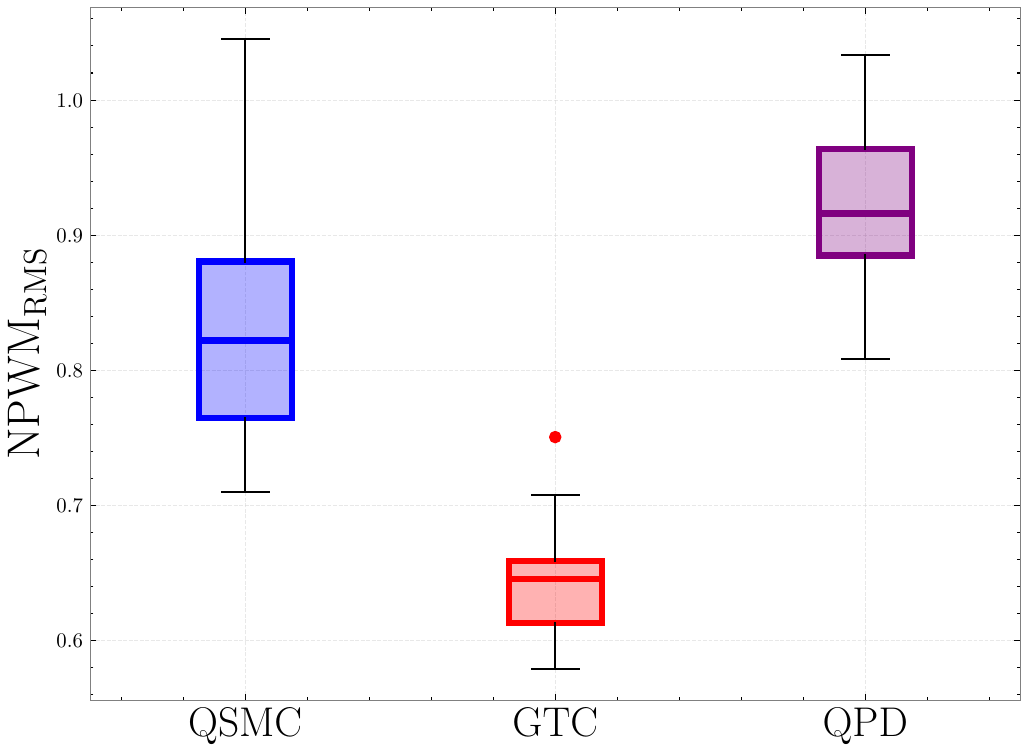}
    \caption{Box plots $q_{e_\text{RMS}}$ and $\text{NPWM}_{\text{RMS}}$ in sensitivity analysis of gimballed attitude control -- Scenario 2}
    \label{fig:s_box_5_q}
\end{figure}

% spaghetti plots
\begin{figure}[t]
    \centering
    \includegraphics[width=1\linewidth]{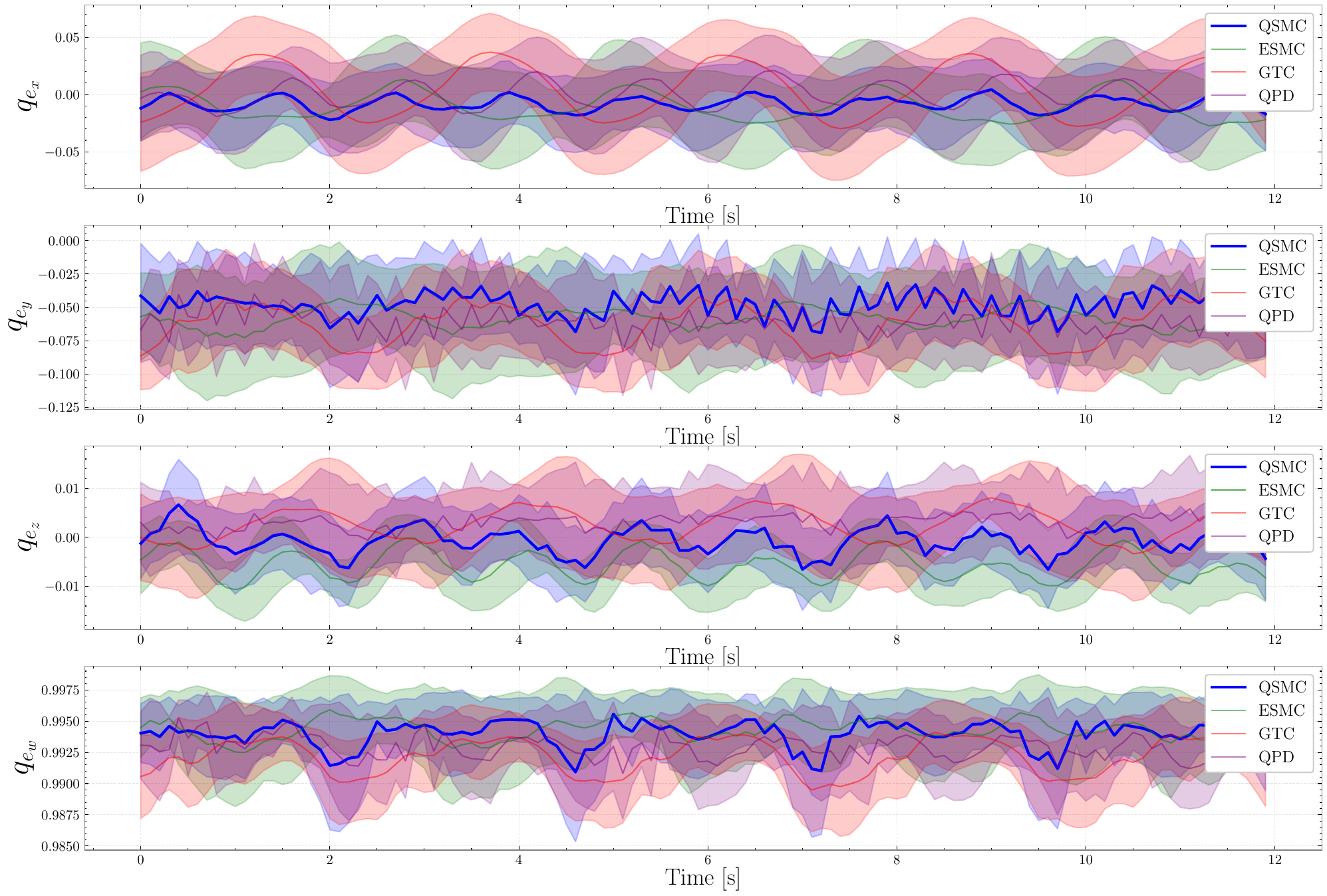}
    \caption{Variations of $\mathbf{q}_e$ in sensitivity analysis of gimballed attitude control -- Scenario 1}
    \label{fig:s_err_bound_2}
\end{figure}
\begin{figure}[t]
    \centering
    \includegraphics[width=1\linewidth]{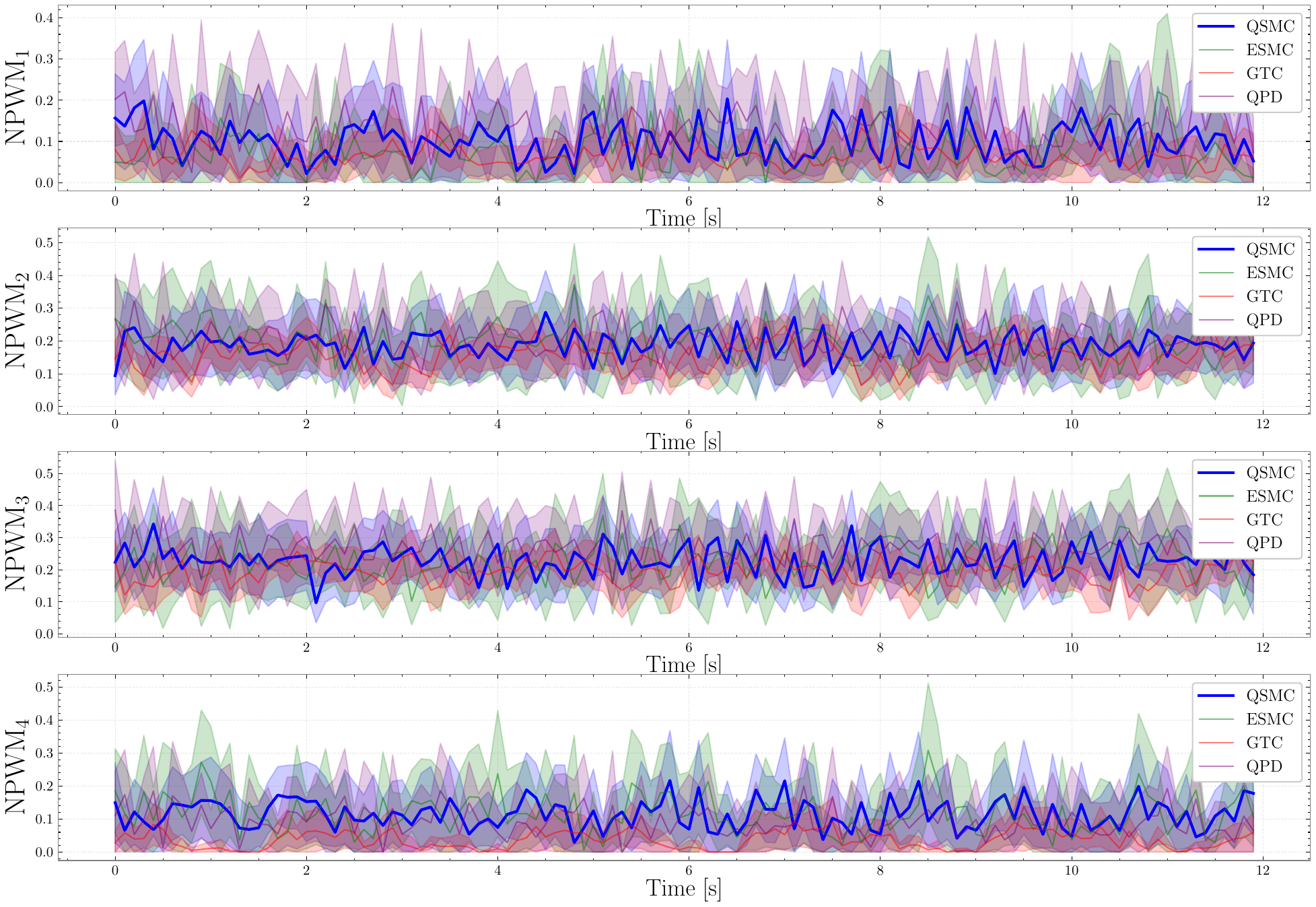}
    \caption{Variations of $\text{NPWM}_i$ in sensitivity analysis of gimballed attitude control -- Scenario 1}
    \label{fig:s_pwm_bound_2}
\end{figure}
\begin{figure}[t]
    \centering
    \includegraphics[width=1\linewidth]{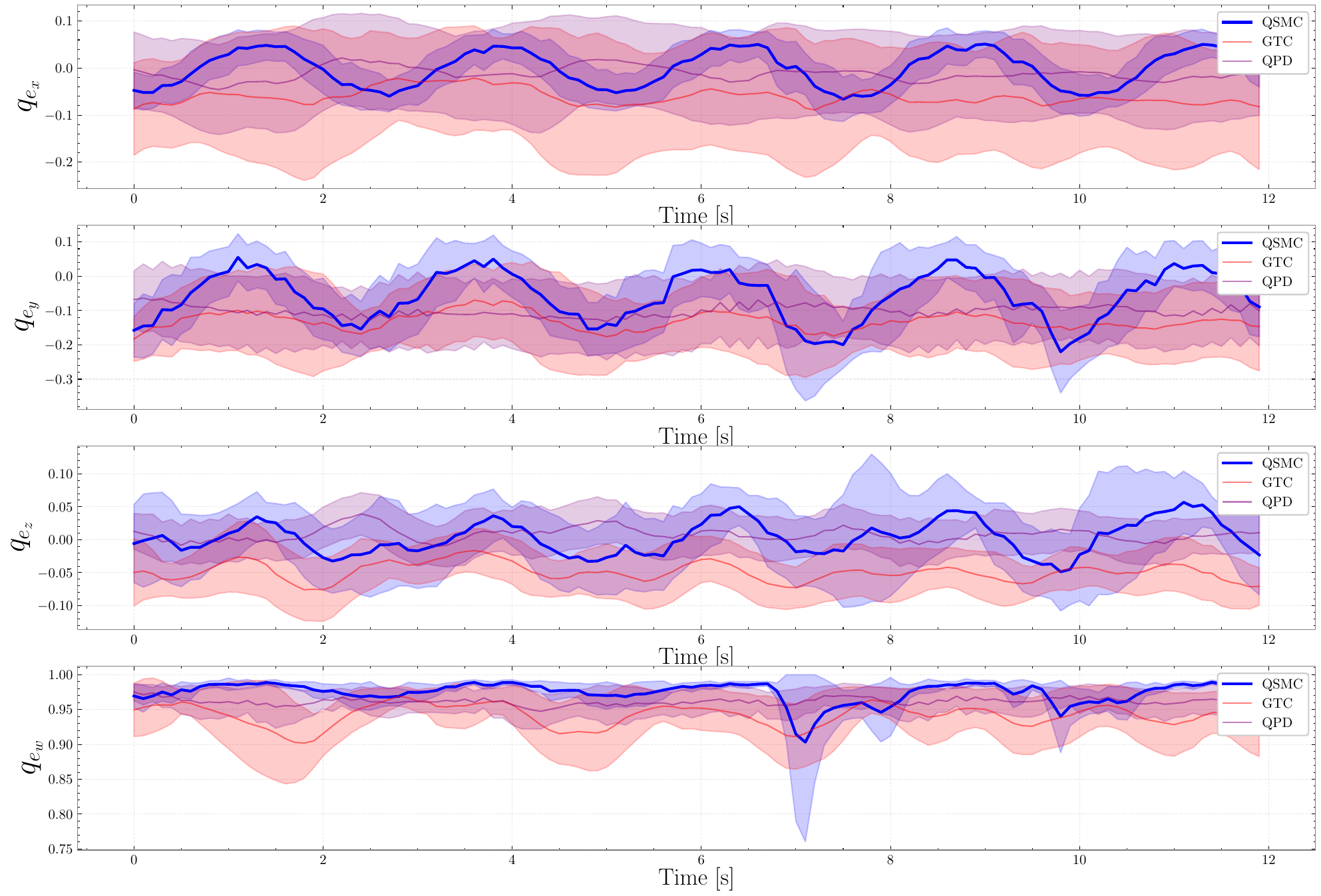}
    \caption{Variations of $\mathbf{q}_e$ in sensitivity analysis of gimballed attitude control -- Scenario 2}
    \label{fig:s_err_bound_5}
\end{figure}
\begin{figure}[t]
    \centering
    \includegraphics[width=1\linewidth]{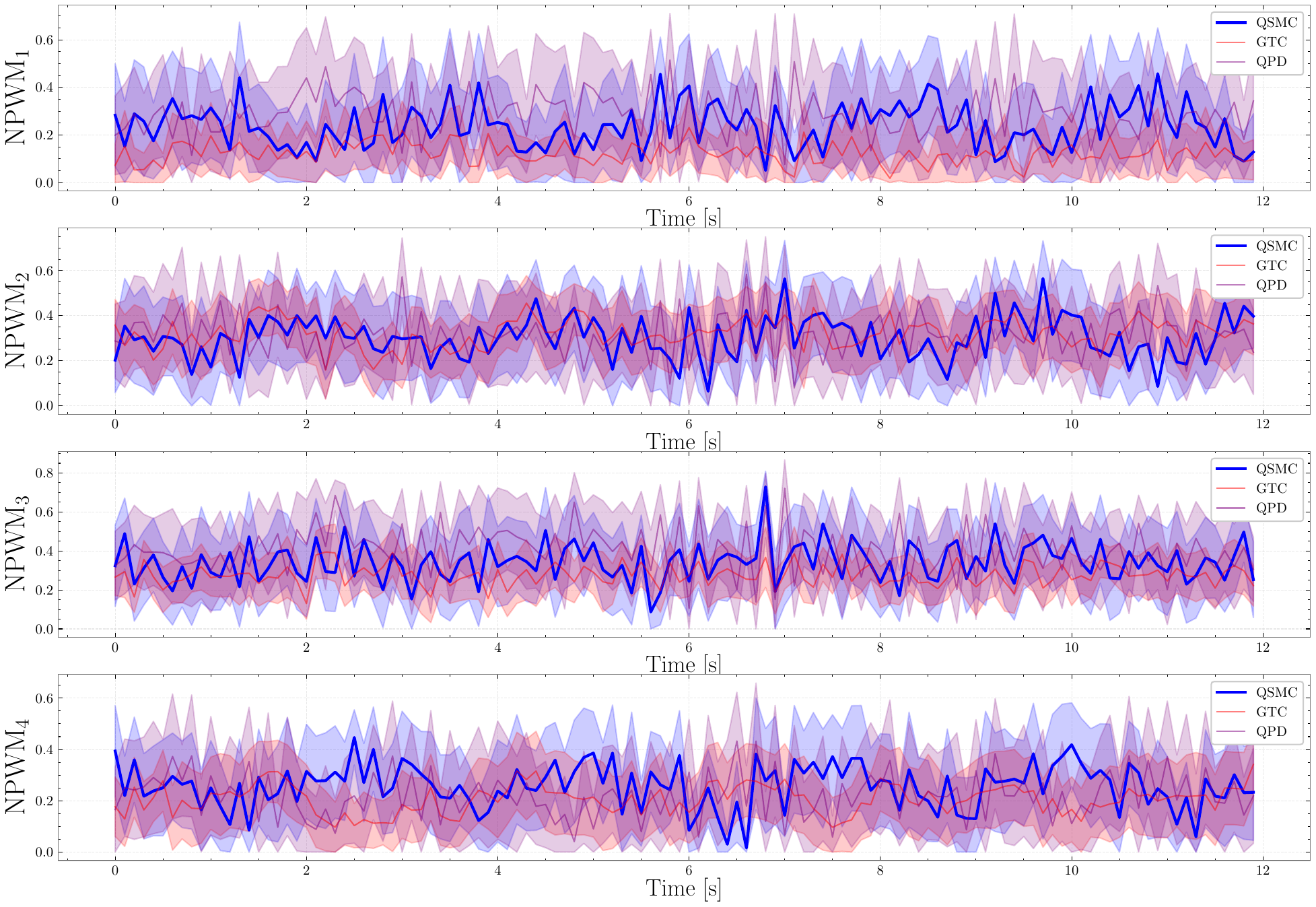}
    \caption{Variations of $\text{NPWM}_i$ in sensitivity analysis of gimballed attitude control -- Scenario 2}
    \label{fig:s_pwm_bound_5}
\end{figure}

\subsubsection{Sensitivity analysis}
As mentioned earlier, we carefully fine-tuned each controller for each experiment scenario to ensure their lowest tracking errors.
However, because this process is empirical, it introduces a potential confounding factor that a controller's apparent superiority might stem from more effective tuning.
While this confounding factor cannot be entirely eliminated, it can be mitigated via sensitivity analysis.

To this end, we conducted a Monte Carlo analysis, generating 10 new sets of values for the parameters of each controller by introducing random deviations of up to $\pm$20\% from the empirically optimized values.
We repeated gimballed attitude control tests using the new parameters values for both Scenario 1 and Scenario 2. 
This led to a total of 80 new trials.

Figures \ref{fig:s_box_2_q} and \ref{fig:s_box_5_q} present the median (the central line within each box), interquantile range (IQR) (the height of each box), and data beyond the IQR (whiskers) of $q_{e_\text{RMS}}$ and $\text{NPWM}_\text{RMS}$ for each controller across these trials.
To illustrate the general trend of $\mathbf{q}_e$ and NPWM over time, we present their error band plots in Figs. \ref{fig:s_err_bound_2}--\ref{fig:s_pwm_bound_5}, showing 95\% confidence intervals of data around the average values.
Note that in Scenario 2, the ESMC generated unstable results for all choices of control parameters; therefore, it is absent from Scenario 2 plots.
The instability of the ESMC even in this set of trials reinforces our earlier observation that the oversimplification of the rotational dynamics in the ESMC leads to performance degradation and instability for high-roll and -pitch maneuvers.

As for Scenario 1, Fig. \ref{fig:s_box_2_q} indicates that the minimum $q_{e_\text{RMS}}$ obtained in one experiment corresponds to the ESMC.
However, the largest IQR of $q_{e_\text{RMS}}$ and $\text{NPWM}_\text{RMS}$ also coincide with the ESMC, indicating the high sensitivity of the ESMC to parameters tuning.
Moreover, the error band plots Fig. \ref{fig:s_err_bound_2} and \ref{fig:s_pwm_bound_2} reveal that the ESMC exhibits the largest peaks in tracking error variations and actuator efforts.

Among the other methods, the GTC exhibits the largest median of $q_{e_\text{RMS}}$ and the smallest median of $\text{NPWM}_\text{RMS}$ in both scenarios.
The QPD has a smaller median of $q_{e_\text{RMS}}$ compared to GTC, consistent with observations in Figs. \ref{fig:gimbal_2_err} -- \ref{fig:gimbal_5_pwm}.
% which can be attributed to the inherent robustness of model-free methods.
% However, the median of $\text{NPWM}_\text{RMS}$ is larger than that of GTC and QSMC in both scenarios.
% The relatively large peaks of actuator efforts in Figs. \ref{fig:s_pwm_bound_2} and \ref{fig:s_pwm_bound_5} verify this observation.

Notably, QSMC presents the lowest median of $q_{e_\text{RMS}}$ and the second lowest median of $\text{NPWM}_\text{RMS}$.
The IQR of $q_{e_\text{RMS}}$ for QSMC remains relatively small, indicating relatively low sensitivity to parameter variations, as opposed to ESMC.
While there were occasions that certain tunings of QSMC led to inferior results compared to other controllers, QSMC has generally demonstrated the lowest tracking errors and the second lowest control efforts, as indicated by both box plots (Figs. \ref{fig:s_box_2_q} and \ref{fig:s_box_5_q}), and error band plots (Figs. \ref{fig:s_err_bound_2} -- \ref{fig:s_pwm_bound_5}).
This stems from the robustness of SMC, and the unwinding-free formulation that leverages the inherent characteristics of $\mathbb{S}^3$.

% \begin{figure}[t]
%     \centering
%     \includegraphics[trim=0.5cm 1.1cm 0cm 0cm,clip,width=0.49\linewidth]{Figures/Results/plot_trajectory_lemniscate_top_view.pdf}
%     \includegraphics[trim=1.5cm 1.8cm 0.5cm 3.5cm,clip,width=0.49\linewidth]{Figures/Results/plot_trajectory_lemniscate_side_view.pdf}
%     \caption{Desired lemniscate trajectory with a fan (yellow cylinder) generating wind disturbance (blue lines) in a section of the trajectory}
%     \label{fig:desired_lemniscate_trajectory}
% \end{figure}
\begin{figure}[t]
    \centering
    \includegraphics[width=1\linewidth]{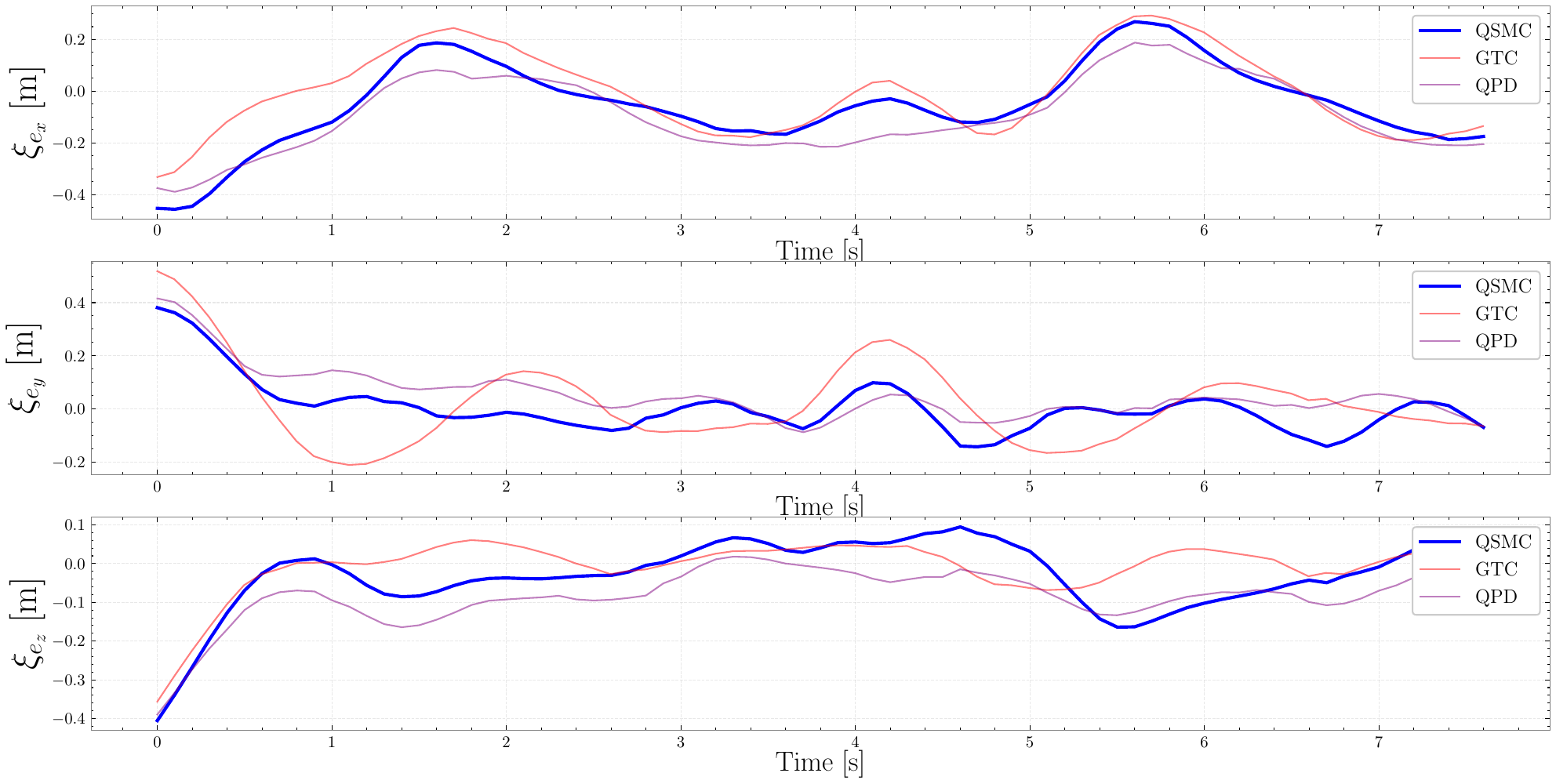}
    \caption{$\boldsymbol{\xi}_e$ in lemniscate trajectory tracking}
    \label{fig:lemniscate_err}
% \end{figure}
% \begin{figure}
    \centering
    \includegraphics[width=1\linewidth]{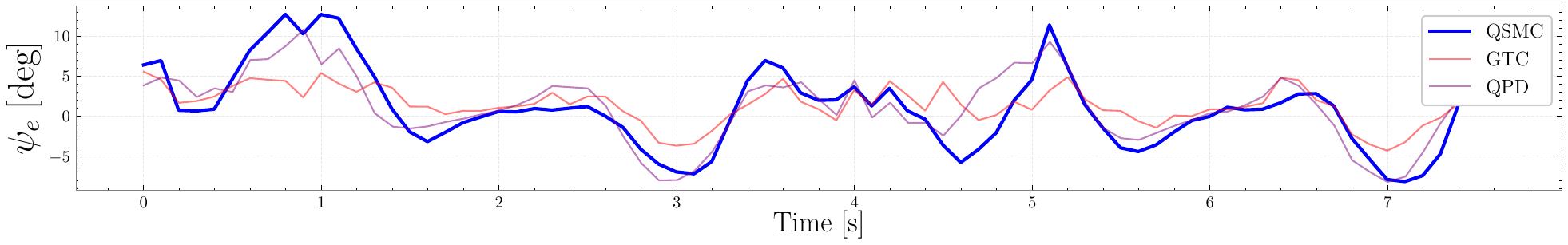}
    \caption{$\psi_e$ in lemniscate trajectory tracking}
    \label{fig:lemniscate_heading_err}
% \end{figure}
% \begin{figure}
    \centering
    \includegraphics[width=1\linewidth]{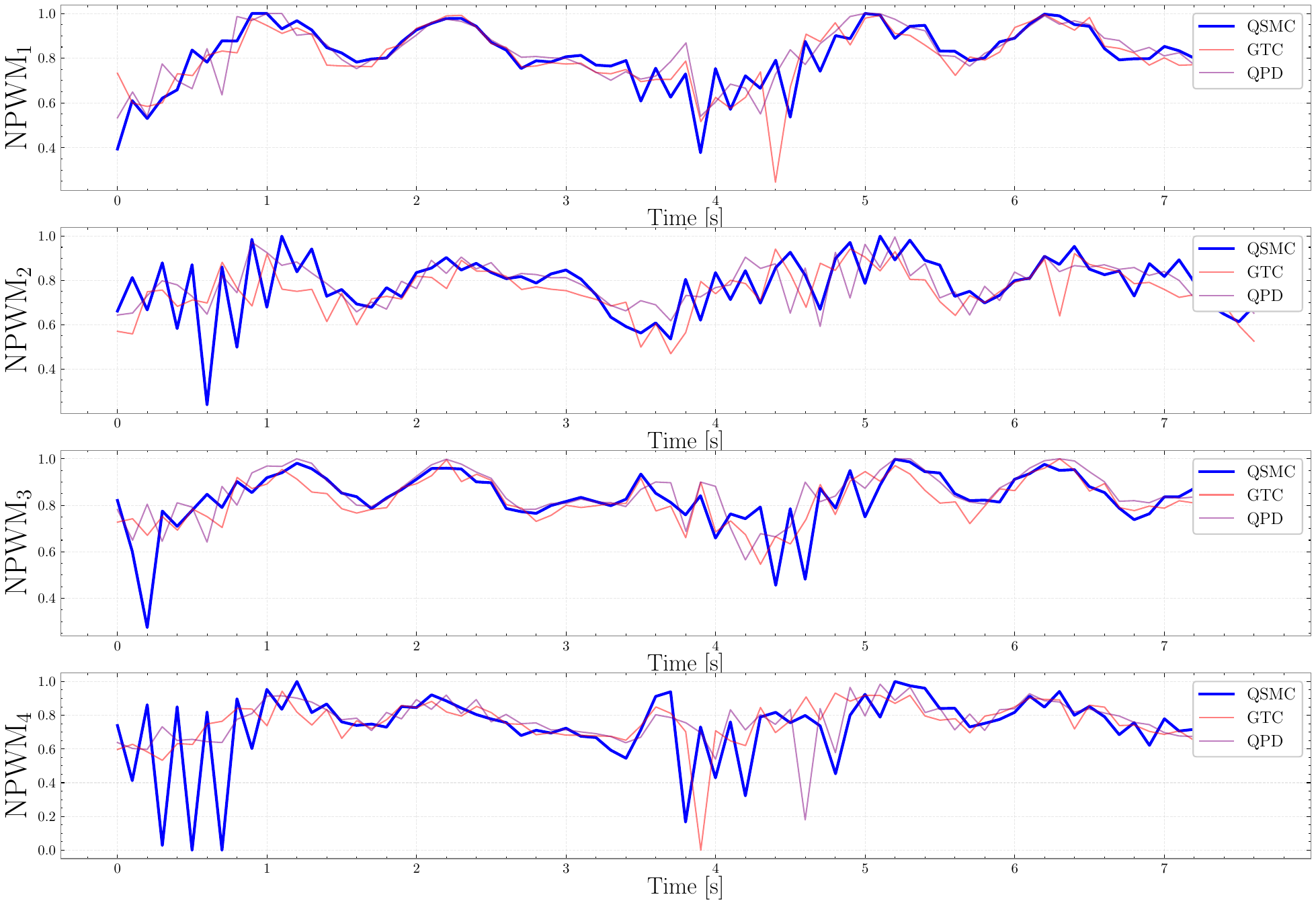}
    \caption{$\text{NPWM}_i$ in lemniscate trajectory tracking}
    \label{fig:lemniscate_pwm}
\end{figure}
% \begin{table}[t] % best gains
%     \centering
%     \caption{Performance metrics in lemniscate trajectory tracking (mean $\pm$ standard deviation}
%     \label{tab:performance_metrics_lemniscate}
%     \begin{tabular}{lcccc}
%         \toprule
%         Controller & $\xi_{e_\text{RMS}}$ & $\psi_{e_\text{RMS}}$ & $\text{NPWM}_\text{RMS}$ \\
%         \midrule
%         QSMC & 0.2305 $\pm$ 0.1301 & 5.7031 $\pm$ 5.3437 & 1.5490 $\pm$ 0.1099\\
%         GTC  & 0.2932 $\pm$  0.1601 & 3.3111 $\pm$ 2.8356 & 1.5396 $\pm$ 0.0911 \\
%         QPD  & 0.3047 $\pm$ 0.1633 & 4.4812 $\pm$ 4.2844 & 1.6174 $\pm$ 0.0989 \\
%         \bottomrule
%     \end{tabular}
% \end{table}

\begin{table}[t] % best gains
    \centering
    \caption{Performance metrics for lemniscate trajectory tracking}
    \label{tab:performance_metrics_lemniscate}
    \begin{tabular}{lcccc}
        \toprule
        Controller & $\xi_{e_\text{RMS}}$ & $\psi_{e_\text{RMS}}$ & $\text{NPWM}_\text{RMS}$ \\
        \midrule
        QSMC & 0.2170 $\pm$ 0.1225 & 5.0904 $\pm$ 4.9728 & 1.6078 $\pm$ 0.1553\\
        GTC  & 0.2299 $\pm$  0.1327 & 2.7241 $\pm$ 2.2914 & 1.5840 $\pm$ 0.1183 \\
        QPD  &  0.2284 $\pm$ 0.1254 & 4.3893 $\pm$ 4.2335 & 1.6282 $\pm$ 0.1124 \\
        \bottomrule
    \end{tabular}
\end{table}

\subsection{Lemniscate trajectory tracking}\label{se:limniscate_trajectory_tracking}
% After evaluating the attitude controllers, we extend our results to tracking a lemniscate trajectory shown in Fig. \ref{fig:desired_lemniscate_trajectory}.
After evaluating the attitude controllers, we extend our results to tracking a lemniscate trajectory shown in Fig.~\ref{fig:experiment_scenarios_freeflight}.
To construct this trajectory, we used the polynomial trajectory planning method described in \cite{richter2016polynomial}.
The vehicle completes two loops along this trajectory.

Note that the original lemniscate trajectory does not include variations along the vertical axis; however, to increase the mission’s complexity and induce high accelerations, we introduced a dive in the desired trajectory.  
The acceleration in this trajectory reaches 5.84 $[\text{m}/\text{s}^2]$, challenging for a nano quadrotor.
The low inertial of the vehicle makes it more susceptible to destabilizing effects from such acceleration; making this mission demanding for the controllers.
We also used a fan to generate wind disturbance for a section of the trajectory, adding another layer of complexity to the mission.

\subsubsection{Experiments with empirically tuned parameters}
Despite our efforts, we were not able to obtain stable results using the ESMC.
The amplitude of $\varphi_d$ and $\theta_d$ angles in this mission exceeds 30$^\circ$ and 50$^\circ$, respectively.
This is far beyond the Euler angles range where the approximation in \eqref{eq:simplified_rot_dyn} remains valid.
While not presented here, we obtained stable results with the ESMC during sluggish maneuvers involving small $\varphi$ and $\theta$; however, the ESMC did not demonstrate any clear advantage over the other methods.
% Overall, the instability of the ESMC in tracking the trajectory illustrated in Fig. \ref{fig:desired_lemniscate_trajectory} aligns with our observations in the Scenario 2 of the gimbal experiments, highlighting a clear limitation of this method. 
Overall, the instability of the ESMC in tracking the lemniscate trajectory aligns with our observations in Scenario 2 of the gimbal experiments, highlighting a clear limitation of this method.
This likely explains why, despite a substantial body of literature, the adoption of ESMC in practical implementations remains low.

The other three controllers successfully accomplished this mission, with their results summarized in Figs. \ref{fig:lemniscate_err} -- \ref{fig:lemniscate_pwm}, and Tab. \ref{tab:performance_metrics_lemniscate}.
We tuned the controllers with the primary goal of achieving their lowest position tracking errors.
Among them, the QSMC achieved the lowest position tracking performance, with the minimum mean and standard deviation of $\xi_{e_\text{RMS}}$, while requiring less control effort than the QPD. 
The QPD marginally outperformed the GTC in position tracking; potentially due to the robustness offered by its SMC position controller that compensated for the wind disturbance.
However, the GTC excelled in other metrics, demonstrating the lowest control effort and the smallest heading error.
Despite these advantages, its position tracking accuracy - our primary tuning goal - remained inferior to the QSMC and the QPD. 

Note that the QSMC and the QPD share the same position control law. However, achieving the minimal position tracking errors for each controller required slightly different tuning due to differences in their attitude controllers.
% Figure \ref{fig:attitude_in_lemniscate} illustrates the desired $\mathbf{q}_d$ generated by each controller, as well as the performance of the attitude controllers in achieving the desired attitude.

% % gains that I found during tunning 
% \begin{table}[t]
%     \centering
%     \caption{Performance metrics for lemniscate trajectory}
%     \label{tab:performance_metrics_gimbal}
%     \begin{tabular}{lcccc}
%         \toprule
%         Controller & $\xi_{e_\text{RMS}}$ & $\boldsymbol{\psi}_{e_\text{RMS}}$ & $\text{NPWM}_\text{RMS}$ \\
%         \midrule
%         QSMC & 0.2305 $\pm$ 0.1301 & 5.7031 $\pm$ 5.3437 & 1.5490 $\pm$ 0.1099\\
%         % ESMC & 0.1933 $\pm$ 0.1093 & 0.6416 $\pm$ 0.6280 & 1.4245 $\pm$ 0.1314\\
%         GTC  & 0.3143 $\pm$ 0.1740 & 3.3475 $\pm$ 2.7070 & 1.4717 $\pm$ 0.0911 \\
%         QPD  & 0.3047 $\pm$ 0.1633 & 4.4812 $\pm$ 4.2844 & 1.6174 $\pm$ 0.0993 \\
%         \bottomrule
%     \end{tabular}
% \end{table}

% \begin{figure}[t]
%     \centering
%     \includegraphics[width=\linewidth]{Figures/Results/qd_QSMCvsQPD_bestgain.pdf}
%     \caption{$\mathbf{q}_d$ and $\mathbf{q}$ for QSMC and QPD in lemniscate trajectory tracking}
%     \label{fig:attitude_in_lemniscate}
% \end{figure}

% \begin{figure}
%     \centering
%     \includegraphics[width=\linewidth]{Figures/Results/qd_QSMCvsQPD_similarPosgain.png}
%     \caption{Differences between $\mathbf{q}_d$ and $\mathbf{q}$ between QSMC and QPD in lemniscate trajectory tracking with identical position controllers}
%     \label{fig:attitude_in_lemniscate}
% \end{figure}
\begin{figure}[t]
    \centering
    \includegraphics[width=0.49\linewidth]{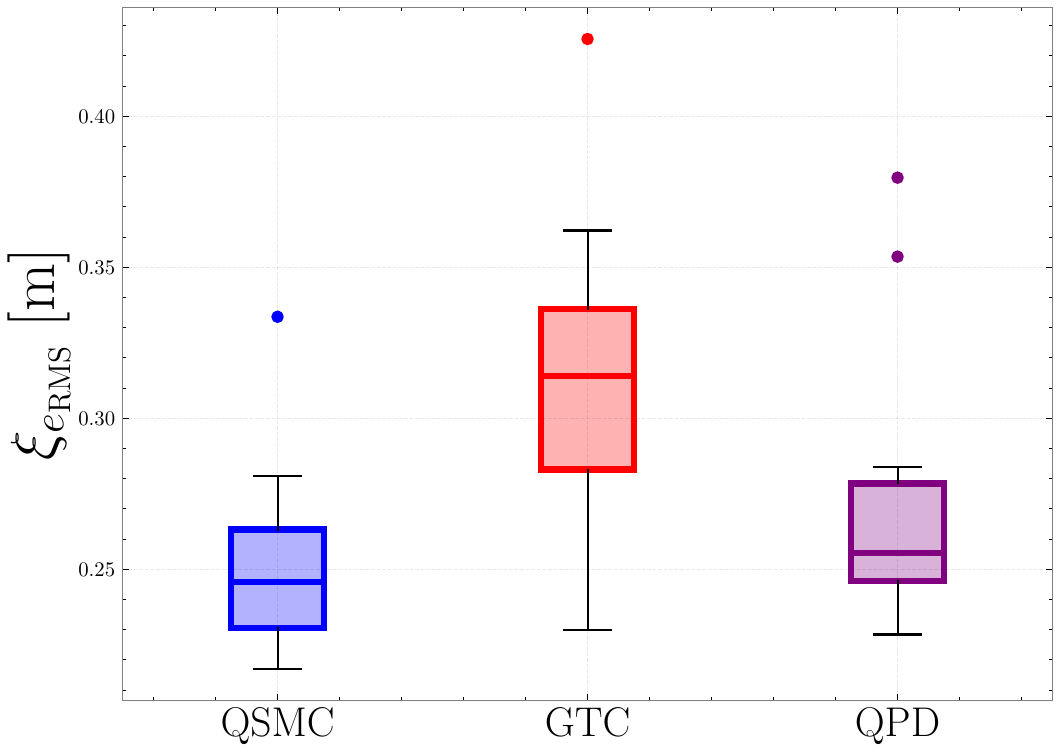}
    \includegraphics[width=0.48\linewidth]{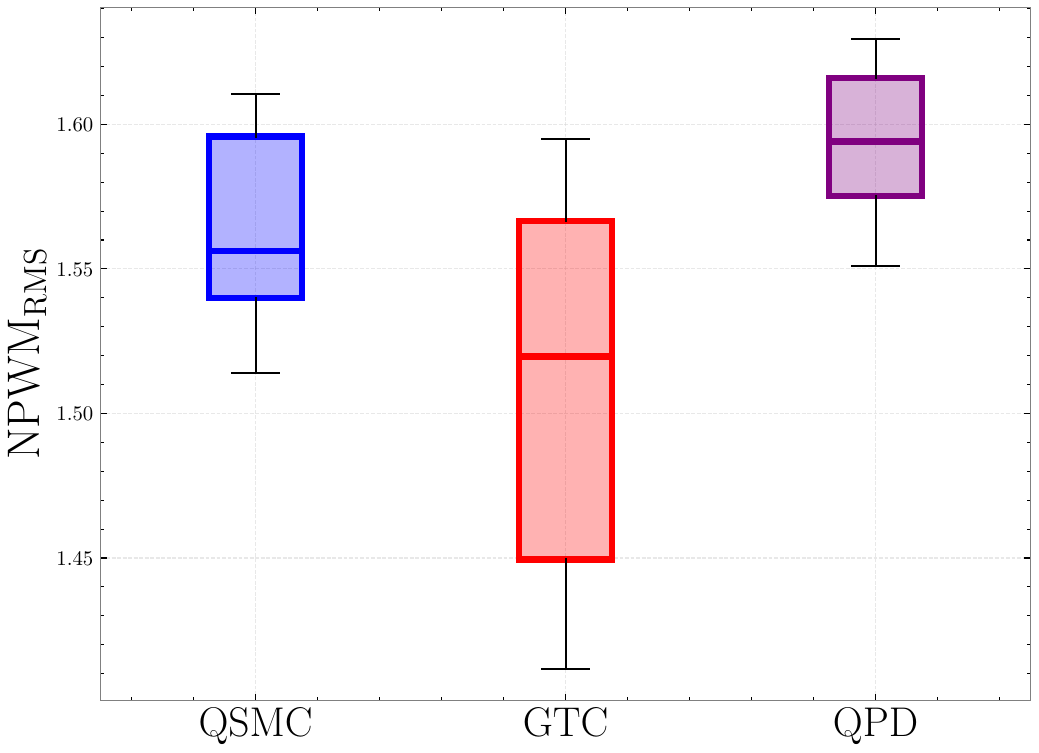}
    \caption{Box plot of $\xi_{e_\text{RMS}}$ and $\text{NPWM}_\text{RMS}$ in sensitivity analysis of lemniscate trajectory tracking. The GTC and QPD had one failed trial. The QSMC succeeded in all trials.}
    \label{fig:s_box_err_lemniscate}
% \end{figure}
% \begin{figure}
    \centering
    \includegraphics[width=1\linewidth]{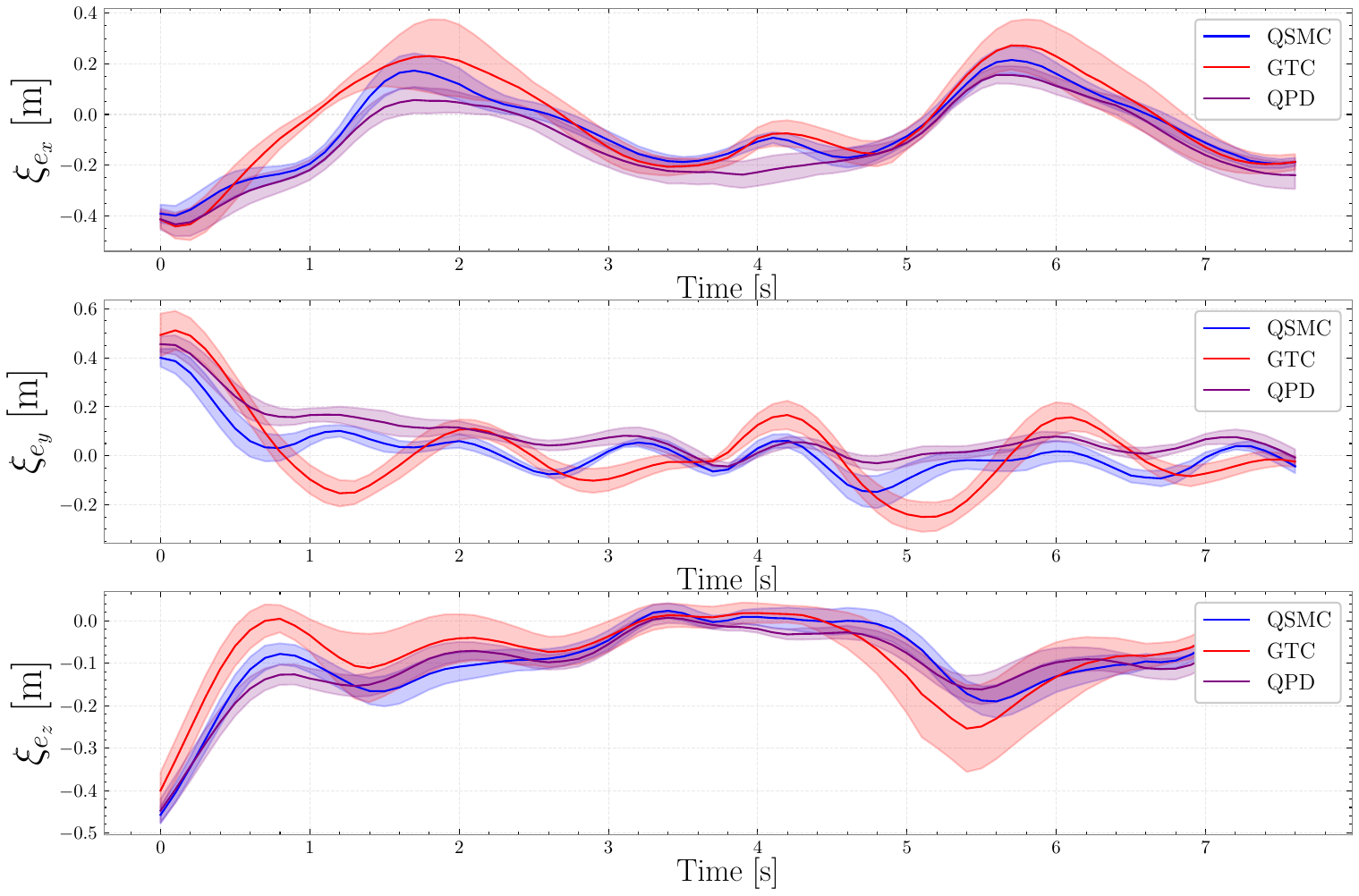}
    \caption{Variations of $\boldsymbol{\xi}_e$ in sensitivity analysis of lemniscate trajectory tracking}
    \label{fig:s_err_bound_err_lemniscate}
% \end{figure}
% \begin{figure}
    \centering
    \includegraphics[width=1\linewidth]{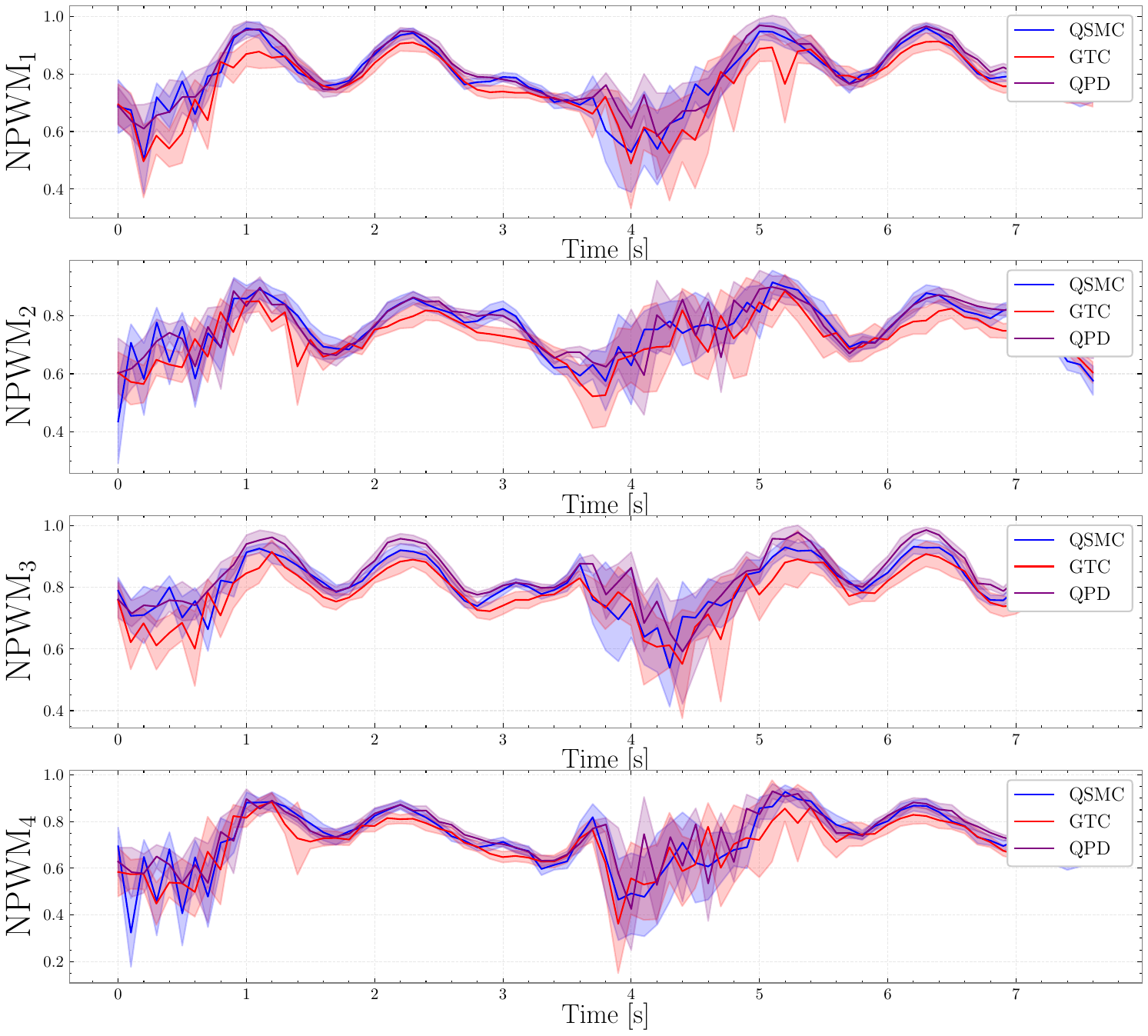}
    \caption{Variations of $\text{NPWM}_i$ in sensitivity analysis of lemniscate trajectory tracking}
    \label{fig:s_err_bound_pwm_lemniscate}
\end{figure}
\vspace{-0.2\baselineskip}
\subsubsection{Sensitivity analysis}
To alleviate the confounding factor of empirical tuning in the above conclusions, we present a sensitivity analysis of the four controllers. 
Similar to the sensitivity analysis in gimballed attitude control, we generated 10 new sets of parameter values for each controller by introducing random deviations of up to $\pm 20\%$ from the empirically optimized values, conducting 40 new trials.

Figures \ref{fig:s_box_err_lemniscate}--\ref{fig:s_err_bound_pwm_lemniscate} summarize the results of these trials.
The ESMC failed to generate stable results and is therefore absent from these plots.
Notably, the GTC and QPD had one failed trial where the quadrotor crashed.
In contrast, the QSMC was the only controller to succeed in all trials.

The QSMC achieves the lowest median value of $\xi_{e_\text{RMS}}$, and the second-lowest median of $\text{NPWM}_{e_\text{RMS}}$.
Consistent with previous observations, the GTC exhibits the lowest control efforts; however, it has the highest median and IQR of $\xi_{e_\text{RMS}}$, indicating that its position tracking accuracy and robustness are inferior to the other two methods.
The QPD displays a smaller median of $\xi_{e_\text{RMS}}$ compared to the GTC and an IQR comparable to the QSMC, likely due to the robustness of its SMC position controller.
However, it requires the highest control effort among the controllers.

Overall, these results indicate that even with variations in the tuning parameters, the QSMC consistently exhibits superior position tracking performance.
It was also the only controller that did not crash the vehicle in any of the trials, showcasing remarkable robustness compared to the other methods.

The advantages of SMC are also evident when it is applied to the position control loop in QPD, leading to a smaller median and IQR of position error compared to the GTC.
While the GTC remains very capable of generating low control efforts, it reveals limitations when robust position tracking under disturbances is prioritized.

\subsection{AQSMC}\label{se:adaptive_experiments}
Having established the advantages of QSMC, we now compare AQSMC and QSMC, exploring the additional benefits offered by the adaptive switching gains.
However, before this comparison, we first experimentally investigate the effects of the design parameters in the adaptation laws, providing practical recommendations for their tuning.

\begin{figure}[t]
    \centering
    \includegraphics[width=1\linewidth, trim={0cm 0.14cm 0cm 0cm}, clip]{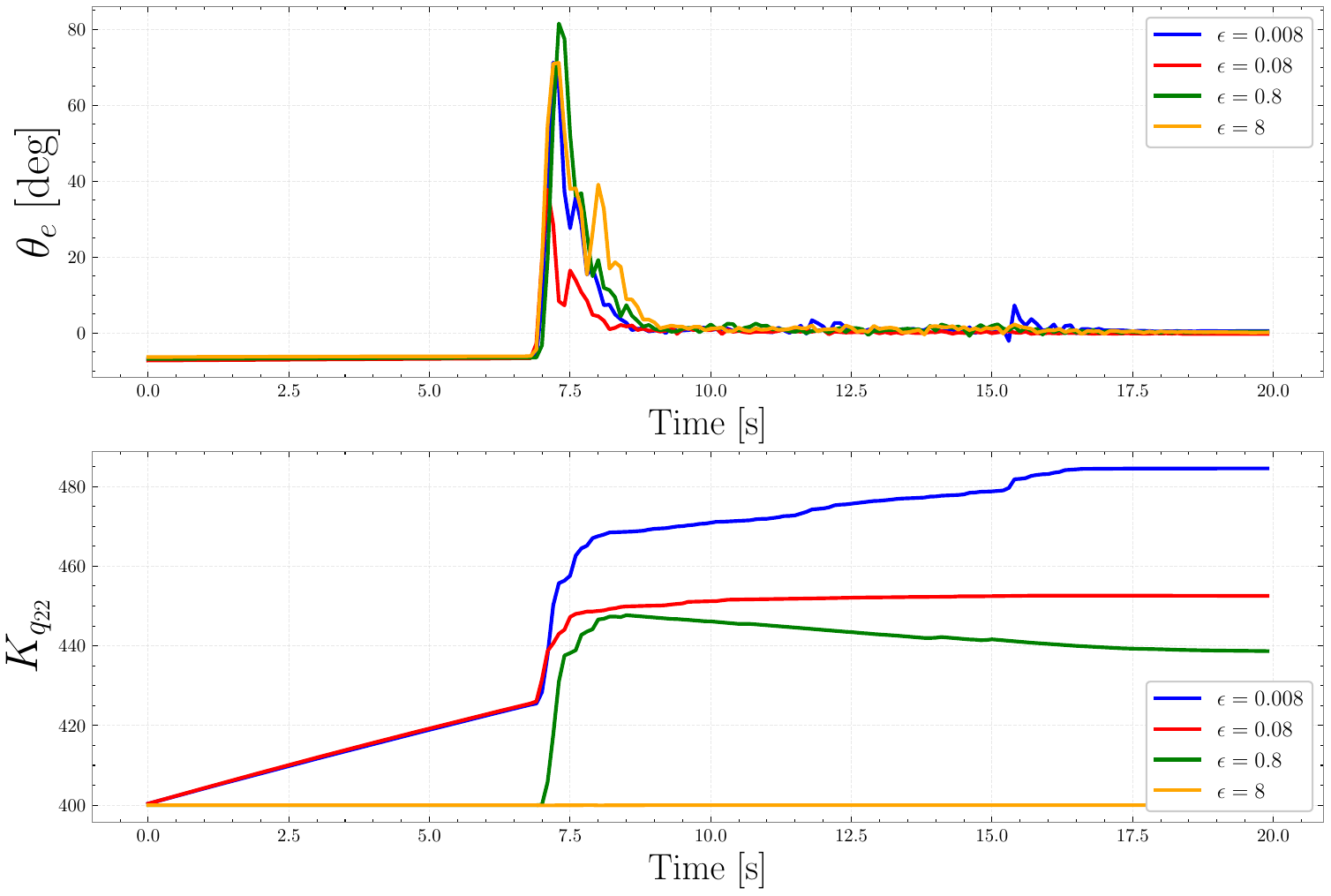}
    \caption{The effect of $\epsilon$ on tracking error and adaptive switching gain. A disturbance is applied at $t=7\;[\text{s}]$.}
    \label{fig:epsilon_err_k}
% \end{figure}
% \begin{figure}
%     \centering
%     \includegraphics[width=1\linewidth]{Figures//Results/motor_epsilon_stablize.png}
%     \caption{$\text{NPWM}_i$ for various $\epsilon$ values of AQSMC in single-axis gimballed attitude control - \textcolor{softblue}{remove epsilon 80 and 8 cases}}
%     \label{fig:epsilon_npwm}
% \end{figure}
% \begin{figure}
%     \centering
%     \includegraphics[width=1\linewidth]{Figures//Results/eqy_khat_mu_stablize.png}
%     \caption{$q_{e_y}$ and $K_{q_{22}}$ for various $\mu$ values of AQSMC in single-axis gimballed attitude control}
%     \label{fig:mu_err_k}
% \end{figure}
% \begin{figure}
%     \centering
%     \includegraphics[width=1\linewidth]{Figures//Results/motor_mu_stablize.png}
%     \caption{$\text{NPWM}_i$ for various $\mu$ values of AQSMC in single-axis gimballed attitude control}
%     \label{fig:mu_npwm}
% \end{figure}
% \begin{figure}
    \centering
    \includegraphics[width=1.0\linewidth, trim={0cm 0.14cm 0cm 0cm}, clip]{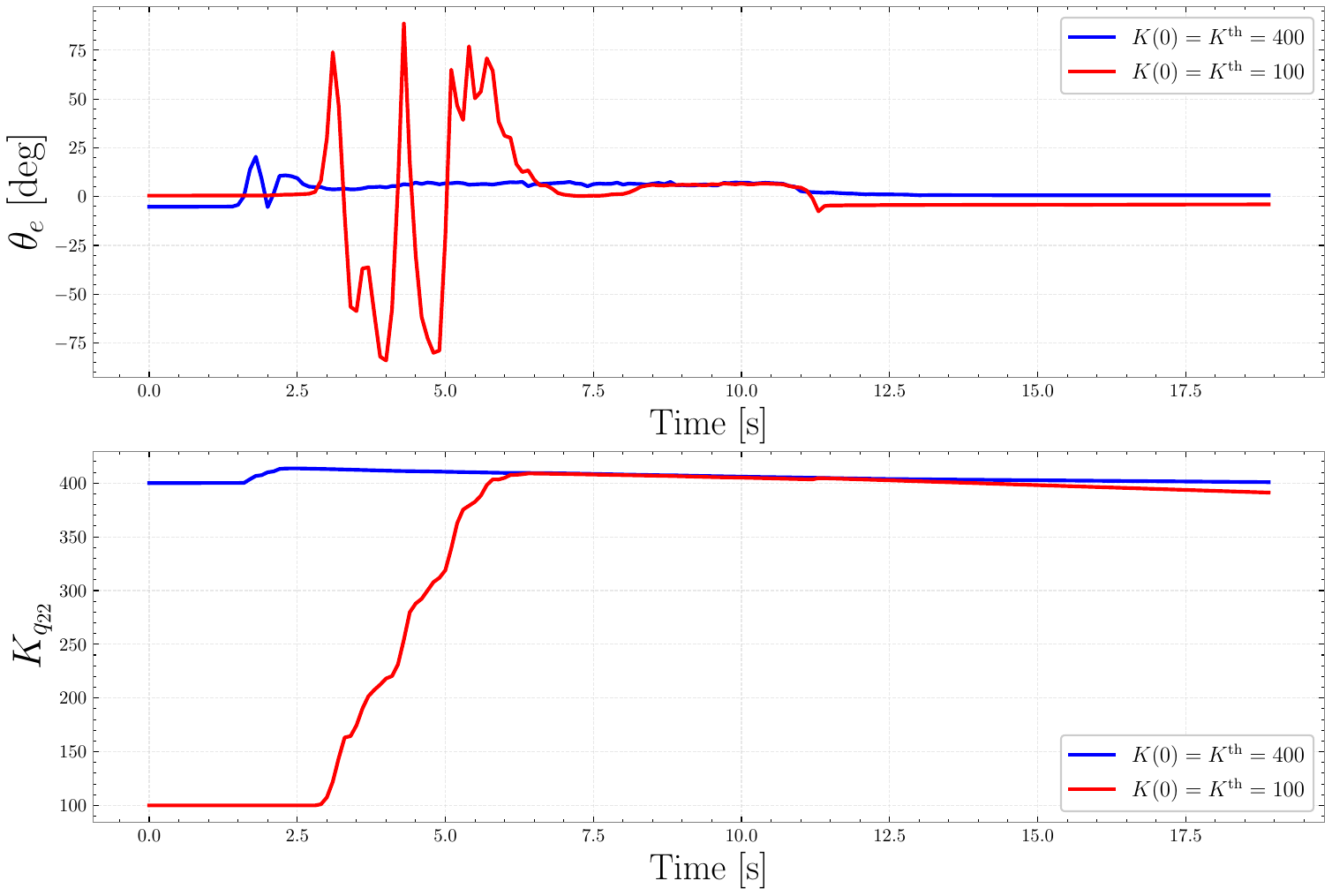}
    \caption{The effect of $K(0)$ switching gain value on disturbance rejection. A disturbance is applied at $t=2\;[\text{s}]$.}
    \label{fig:k_threshold}
\end{figure}

\subsubsection{Tuning AQSMC-specific parameters}
The adaptation laws \eqref{eq:adaptation_law_2} involve several design parameters.
Among them, $\Gamma$ is commonly recognized as the learning rate, governing the speed at which the adaptive gains are updated. 
The other parameters are however less prevalent, and we direct our attention to study their roles.
For brevity, we omit the indices of the parameters in our notations as our analysis apply to all axes of position and attitude control laws.

% it is evident from \eqref{eq:adaptation_law_2} that when $\lvert s \rvert > \epsilon \phi$, $\dot{K} > 0$; hence,

\textit{The effects of $\epsilon$:}
This parameter governs the condition under which ${K}$ is either increasing or decreasing. 
Smaller values of $\epsilon$, make $\dot{K} > 0$ more easily triggered.

We conducted single-axis gimballed attitude control experiments, locking two axes and focusing solely on the pitch axis, with a desired pitch angle of $\theta_d = 0$.
The vehicle operates in attitude hold for 20 seconds.
We apply a 5\;$[\text{m}/\text{s}]$ wind disturbance at $t = 7\;[\text{s}]$ for 10 seconds, observing the controller’s response to the disturbance and its subsequent recovery behavior for different $\epsilon$ values.

Figure \ref{fig:epsilon_err_k} illustrates the results.
At the onset of the experiment, $K$ begins to increase for the two smallest values of $\epsilon$. While this improves the controller’s ability to reject disturbance, particularly evident for $\epsilon = 0.08$, it also introduces sensitivity, causing oscillatory behavior near $\theta_e = 0$.
For these smaller values of $\epsilon$, $K$ shows a limited tendency to decrease as $s$ must converge to a small boundary defined by $\phi \epsilon$.

Conversely, larger values of $\epsilon$ make the controller less reactive. For instance, with $\epsilon = 8$, $K$ does not respond to disturbance, resulting in slower disturbance rejection. Therefore, $\epsilon$ should not be too large, to avoid sluggish responses, nor too small, to prevent excessive sensitivity. 

The case of $\epsilon = 0.8$ appears to strike a balance. It allows $K$ to respond effectively to the disturbance while ensuring that $K$ decreases after the error converges to zero. 
Note that, overall, AQSMC offers low sensitivity to changes in $\epsilon$. 
We were able to observe the above effects only by changing $\epsilon$ by orders of magnitude and applying a severe wind disturbance.

\textit{The effects of $K^\text{th}$ and $K(0)$}:
We note that $\dot{K} < 0$ only when $\lvert s \rvert < \epsilon \phi$, implying small tracking errors.
Under these conditions, minimizing $K$ is desirable to reduce control efforts. However, this can pose a risk: a sudden disturbance increase can lead to poor disturbance rejection or even instability.

Figure \ref{fig:k_threshold} illustrates this through a gimballed pitch stabilization experiment with a disturbance applied at $t=2\;[\text{s}]$.
While both controllers reject the disturbance, the lower gain causes erratic oscillations and a long settling time that could destabilize the vehicle in a free flight. 
For this reason, $K^\text{th}$ should not be set to overly small values. 
To maintain stable and responsive behavior under sudden disturbances, we recommend selecting values for $\hat{K(0)}$ and $K^\text{th}$ similar to the QSMC switching gains. 
Accordingly, for all subsequent experiments, we set $K^{\text{th}}$ to match the switching gains used in the QSMC.

\textit{The effects of $\mu$:}
This parameter defines the rate at which $K$ increases when $K < K^{\text{th}}$.
In our hardware experiments, $\mu$ did not appear to be as influential as other parameters in shaping the behavior of $K$.
However, large values of $\mu$ can cause oscillations in $K$, which may degrade system performance.
To avoid this, we recommend using smaller values of $\mu$ to ensure $K$ remains close to $K^{\text{th}}$ when $\lvert s \rvert < \epsilon \phi$.
We selected $\mu = 0.02$ for the rest of the experiments.

\begin{figure}[t]
    \centering
    \includegraphics[width=1\linewidth, trim={0cm 0.14cm 0cm 0cm}, clip]{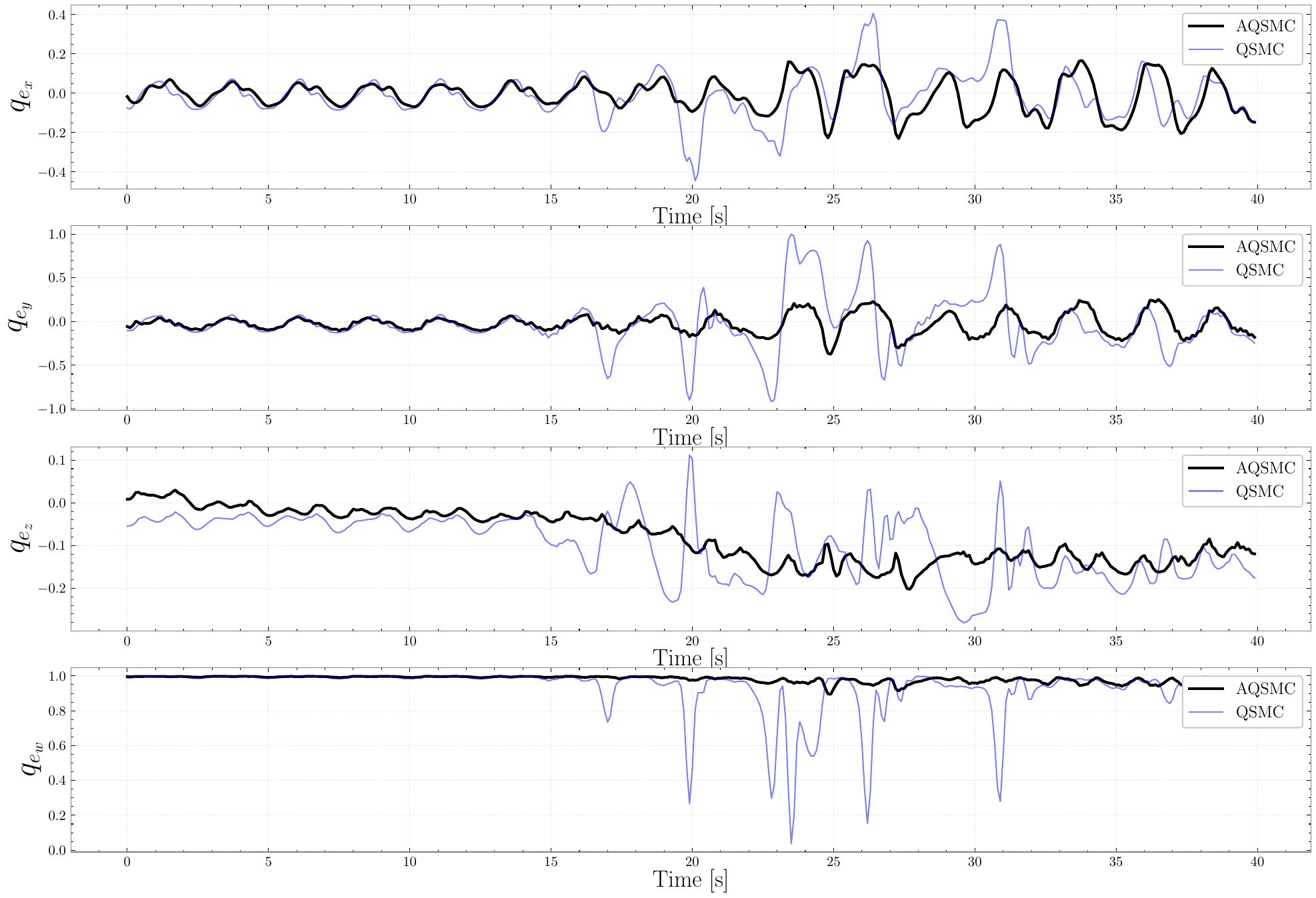}
    \caption{$\mathbf{q}_e$ in gimballed attitude control comparing AQSMC and QSMC}
    \label{fig:gimbal_aqsmc_err}
% \end{figure}
% \begin{figure}
    \centering
    \includegraphics[width=1\linewidth, trim={0cm 0.14cm 0cm 0cm}, clip]{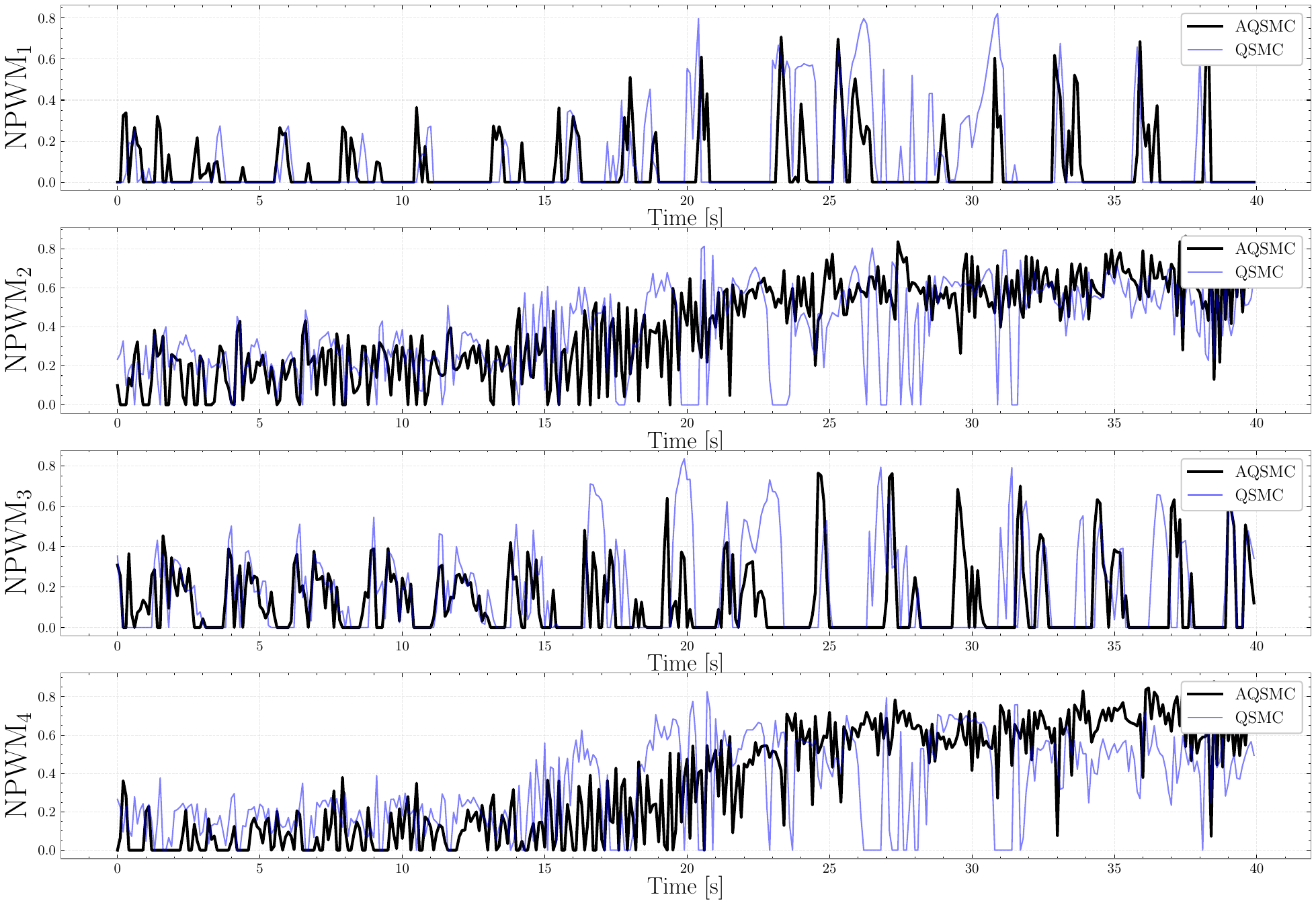}
    \caption{$\text{NPWM}_i$ in gimballed attitude control comparing AQSMC and QSMC}
    \label{fig:gimbal_aqsmc_pwm}
% \end{figure}
% \begin{figure}
    \centering
    \includegraphics[width=1\linewidth, trim={0cm 0.14cm 0cm 0cm}, clip]{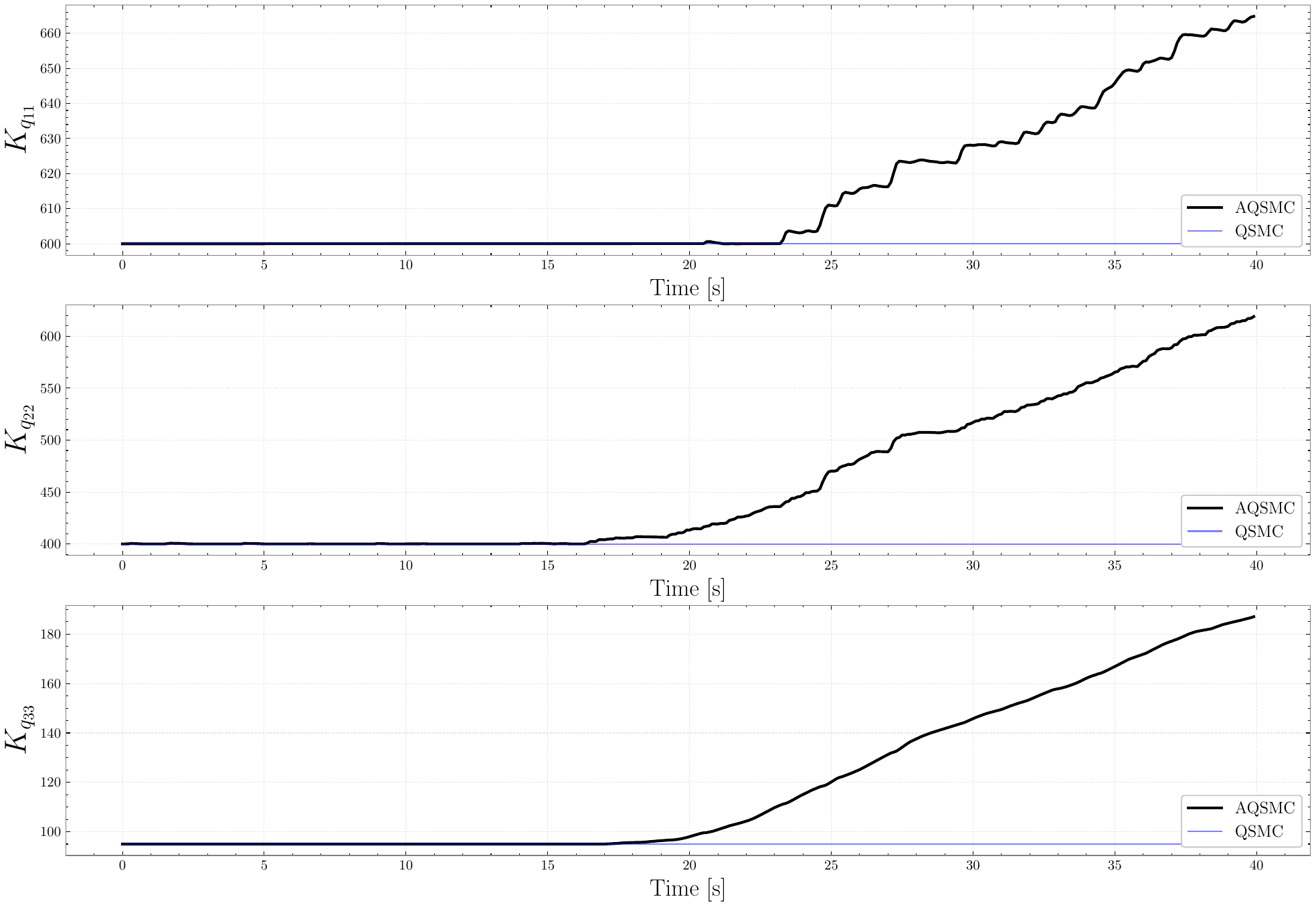}
    \caption{Switching gains in gimballed attitude control comparing AQSMC and QSMC}
    \label{fig:gimbal_aqsmc_k}
\end{figure}
\begin{table}[t]
    \centering
    \caption{Performance metrics for gimballed attitude control of QSMC and AQSMC }
    \label{tab:performance_metrics_gimbal_aqsmc_vs_qsmc}
    \begin{tabular}{lccc}
        \toprule
        Controller & $q_{e_\text{RMS}}$ & $\text{NPWM}_\text{RMS}$ \\
        \midrule
        QSMC & 0.3350 $\pm$ 0.1880  & 0.6943 $\pm$ 0.2165\\
        AQSMC  & 0.1744 $\pm$  0.0928 & 0.6907 $\pm$ 0.2076 \\
        \bottomrule
    \end{tabular}
\end{table}
\subsubsection{Gimballed attitude control}
Using the above recommendations for tuning adaptive gains, we conducted a new set of gimbal attitude control experiments with a focus on comparing AQSMC and QSMC in response to a wind disturbance.
The vehicle tracks $\boldsymbol{\eta}_d = 0.2 [\sin(0.2\pi t), \cos(0.2 \pi t), 0]^\top $, with a strong wind disturbance applied at $t = 15\;[\text{s}]$.
Figures \ref{fig:gimbal_aqsmc_err}--\ref{fig:gimbal_aqsmc_k} and Tab. \ref{tab:performance_metrics_gimbal_aqsmc_vs_qsmc} present the results.

For $t < 15\;[s]$, both the AQSMC and QSMC display similar tracking performance with minimal attitude errors.
With the application of the wind disturbance, QSMC struggles to maintain stability. Not only did the wind directly impact the vehicle, but it also affected the gimbal, which subsequently exacerbated the unmodeled gimbal moments exerted on the vehicle.
As such, the disturbance likely exceeds the upper bounds considered in \eqref{eq:gain_condition_pos}, leading to significant attitude deviations and eventual instability.
In contrast, as depicted in Fig. \ref{fig:gimbal_aqsmc_k}, the AQSMC adapts dynamically by increasing the gains in response to the disturbance. This enables AQSMC to stabilize the system, albeit with a noticeable increase in steady-state error.
Interestingly, AQSMC achieves the above results with lower control efforts highlighted by $\text{NPWM}_{\text{RMS}}$ in Tab. \ref{tab:performance_metrics_gimbal_aqsmc_vs_qsmc}
% It is evident from Fig. \ref{fig:gimbal_aqsmc_pwm} and Tab. \ref{tab:performance_metrics_gimbal_aqsmc_vs_qsmc} that AQSMC requires increased control effort during the disturbance phase. 
% However, AQSMC's ability to stabilize the vehicle outweighs this marginally higher control effort.

\begin{figure}
    \centering
    \includegraphics[width=1\linewidth]{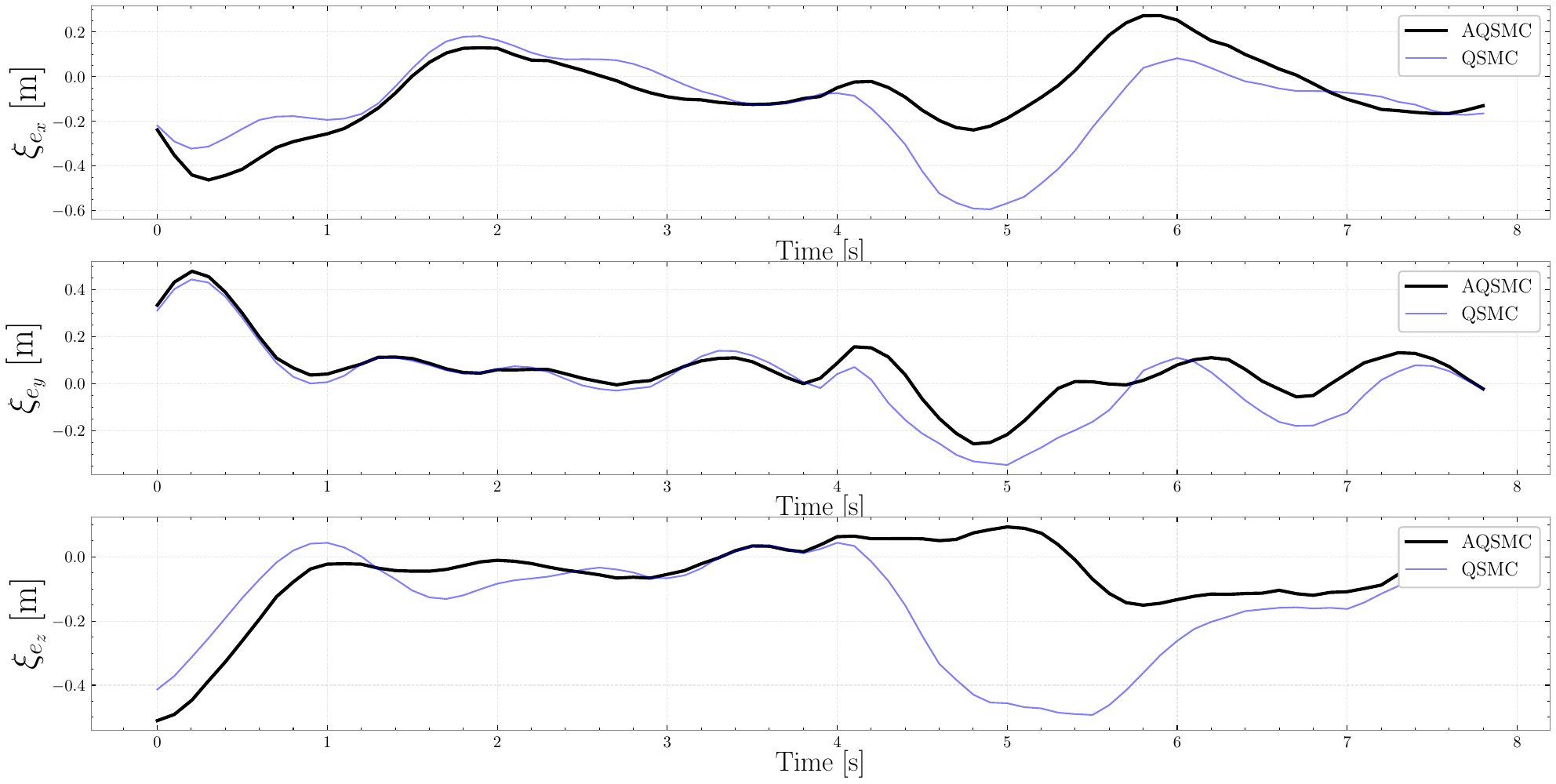}
    \caption{$\boldsymbol{\xi}_e$ in lemniscate trajectory tracking comparing AQSMC and QSMC with the highest level of disturbance}
    \label{fig:pos_error_lemniscate_aqsmc_vs_qsmc}
% \end{figure}
% \begin{figure}
    \centering
    \includegraphics[width=1\linewidth]{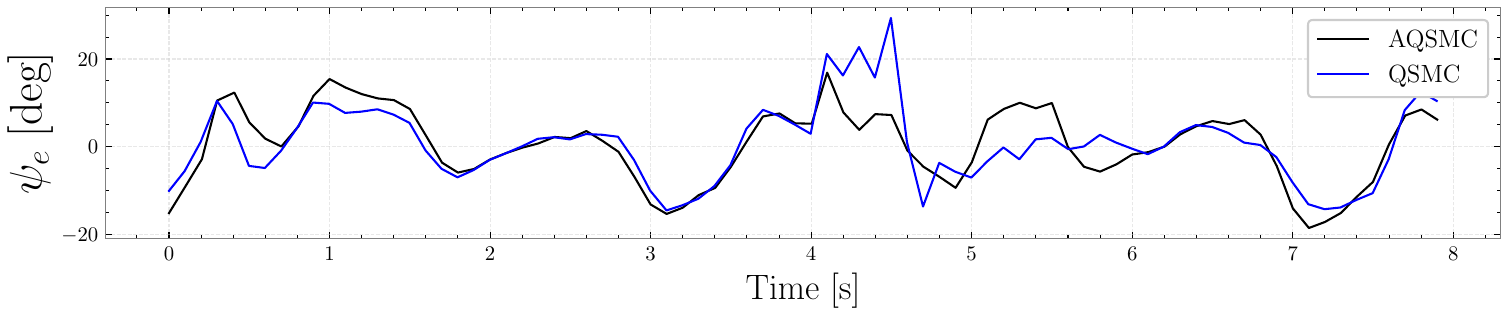}
    \caption{${\psi}_e$ in lemniscate trajectory tracking comparing AQSMC and QSMC with the highest level of disturbance}
    \label{fig:psi_error_lemniscate_aqsmc_vs_qsmc}
% \end{figure}
% \begin{figure}
    \centering
    \includegraphics[width=1\linewidth]{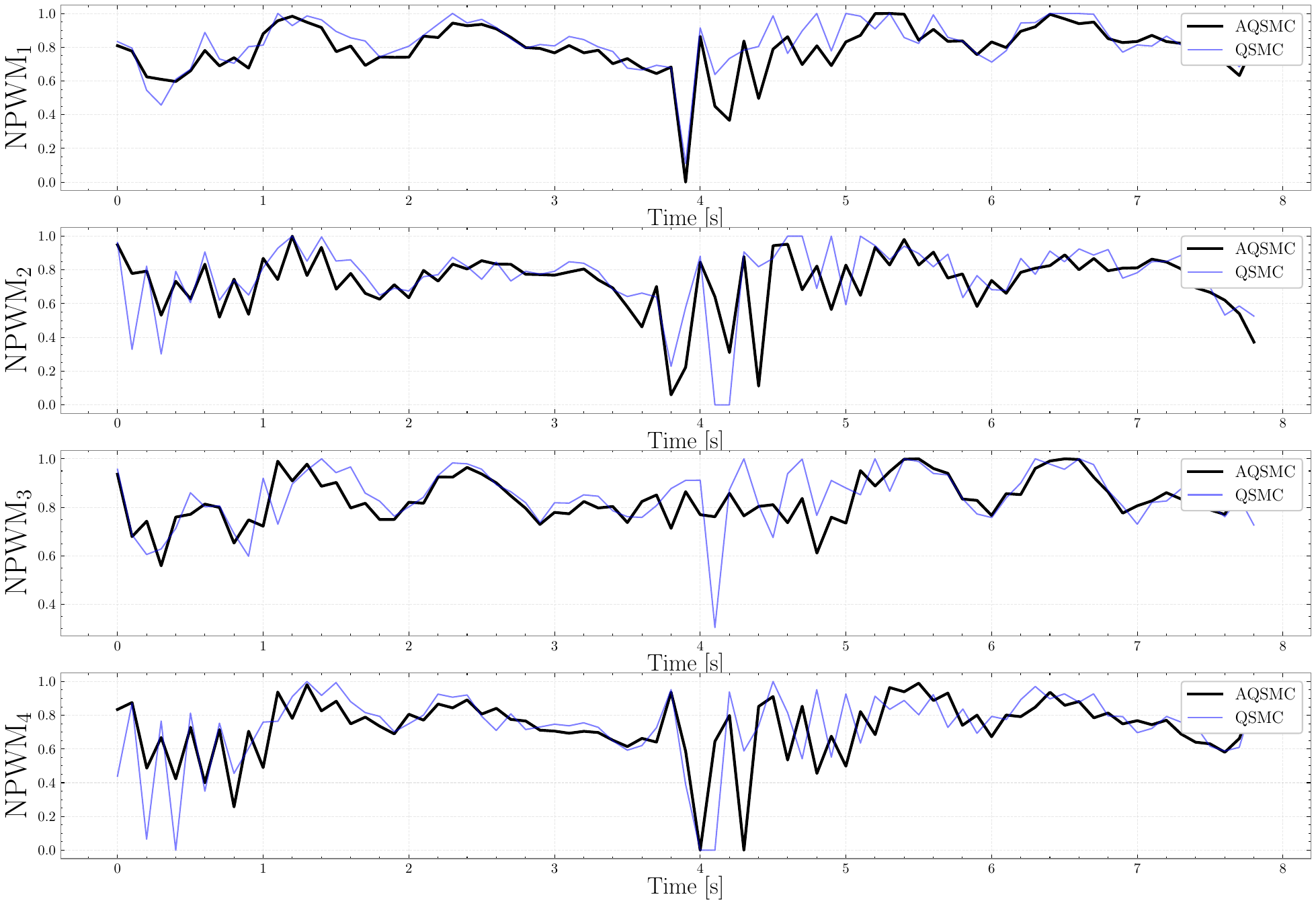}
    \caption{$\text{NPWM}_i$ in lemniscate trajectory tracking comparing AQSMC and QSMC with the highest level of disturbance}
    \label{fig:motor_lemniscate_aqsmc_vs_qsmc}
\end{figure}
\begin{table*}[t]
    \centering
    \caption{Performance metrics in lemniscate trajectory tracking comparing AQSMC and QSMC}
    \label{tab:performance_metrics_lemniscate_aqsmc_vs_qsmc}
    \begin{tabular}{lcccccc}
        \toprule
        & \multicolumn{3}{c}{QSMC} & \multicolumn{3}{c}{AQSMC}\\
        \cmidrule(lr){2-4} \cmidrule(lr){5-7}
        Wind Speed& $\xi_{e_\text{RMS}}$ & $\psi_e$ & $\text{NPWM}_\text{RMS}$ & $\xi_{e_\text{RMS}}$ & $\psi_e$ & $\text{NPWM}_\text{RMS}$ \\
        \midrule
        % None      &0.2456 $\pm$ 0.1411& 8.4546 $\pm$ 8.2722& 1.6708 $\pm$ 0.1106  &0.2294 $\pm$ 0.1286 & 10.0745$\pm$ 9.9522& 1.5545 $\pm$ 0.1131\\
        Up to 3.8 $\text[m/s]$ &0.2536 $\pm$ 0.1461& 7.4813 $\pm$ 7.2829& 1.6437 $\pm$ 0.1404  &0.2065 $\pm$ 0.1189 & 8.0102 $\pm$ 7.9883& 1.5533 $\pm$ 0.1311\\
        Up to 4.6 $\text[m/s]$ &0.6068 $\pm$ 0.3490& 10.0517$\pm$ 9.9257& 1.6229 $\pm$ 0.2092  &0.2284 $\pm$ 0.1302 & 7.6474 $\pm$ 7.5598& 1.5503 $\pm$ 0.1345\\
        Up to 5.6 $\text[m/s]$ &0.3469 $\pm$ 0.1795& 8.6149 $\pm$ 8.5771& 1.6223 $\pm$ 0.1765  &0.2484 $\pm$ 0.1417 & 8.1526 $\pm$ 8.1152& 1.5695 $\pm$ 0.1630\\
        \bottomrule
    \end{tabular}
\end{table*}

\subsubsection{Lemniscate trajectory tracking}
We revisited the lemniscate trajectory tracking experiment from Section \ref{se:limniscate_trajectory_tracking}, this time focusing on comparing AQSMC and QSMC under varying levels of wind disturbance.

Table \ref{tab:performance_metrics_lemniscate_aqsmc_vs_qsmc} summarizes the performance metrics under varying wind disturbance levels.
As wind speed increases, $\xi_{e_\text{RMS}}$ and $\psi_{e_\text{RMS}}$ rise noticeably for QSMC. 
In contrast, AQSMC maintains a consistent performance, with only a slight increase in $\boldsymbol{\xi_{e_\text{RMS}}}$ and $\psi_{e_\text{RMS}}$.
Notably, AQSMC demonstrates a remarkably consistent $\text{NPWM}_\text{RMS}$, indicating efficient control effort even under challenging conditions, whereas QSMC shows greater $\text{NPWM}_\text{RMS}$ fluctuations, especially at higher wind speeds.
% On the other hand, QSMC exhibits fluctuations in $\text{NPWM}_\text{RMS}$, particularly at higher wind speeds.

We present the time evolution of $\boldsymbol{\xi}_e$, $\psi_e$, and $\text{NPWM}_i$ for the lemniscate trajectory tracking with the highest disturbance level in Figs. \ref{fig:pos_error_lemniscate_aqsmc_vs_qsmc}–-\ref{fig:motor_lemniscate_aqsmc_vs_qsmc}.
The results clearly demonstrate the superior performance of AQSMC in handling large wind disturbances. As disturbance levels increase, the performance gap between AQSMC and QSMC widens, showcasing the added benefits of the proposed adaptive gains.

\subsubsection{Throw launches}
We extended our flight tests to include throw launches, which are more aggressive than the lemniscate trajectory tracking experiments. 
In these trials, we manually launched the aircraft in various orientations and at different speeds. Once the vehicle velocity reached 2.5 $[\text{m/s}]$, the controllers were activated, powering the motors and attempting to stabilize the vehicle in position $[0,0,1]^\top \;[\text{m}]$.

For the controllers' tuning, we used the same parameters used for the lemniscate trajectory tracking tests.
Due to the manual nature of the throw launches, maintaining consistent initial conditions was challenging, making it difficult to provide quantitative metrics for this scenario.
Nevertheless, both controllers demonstrated remarkable recovery capabilities, even stabilizing the vehicle after upside-down throws by executing flip maneuvers.
Across 10 successful upside-down throw launches trials for each controller, AQSMC achieved an average recovery time of $1.04\;[\text{s}]$, compared to $1.67\;[\text{s}]$ for QSMC. 
Additionally, the maximum acceleration observed during these trails reached $30.17\;  [\text{m}/\text{s}^2]$ for AQSMC, and 
 $27.06\;[\text{m}/\text{s}^2]$ for QSMC.
Figure \ref{fig:throw_launch_aqsmc} illustrates one such trial.
These results highlight the remarkable agility of the controllers, especially AQSMC, in handling aggressive initial conditions and large accelerations, which are challenging to manage on a vehicle with low mass and inertia.

Of note, for comparison, we attempted similar throw launches with the other controllers. 
However, their performance proved inconsistent and resulted in several crashes, indicating that these controllers would require additional tuning to match the reliability and robustness demonstrated by QSMC and AQSMC under such demanding conditions.

% We successfully conducted throw launch experiments for the quadrotor using AQSMC.
% As our throw launches were manual, it was difficult to replicate identical trials to compare AQSMC with QSMC.
% However, in our experience, AQSMC demonstrated comparable performance in recovering from highly dynamic initial conditions and stabilizing the vehicle.
% Figure \ref{fig:throw_launch_aqsmc} illustrates one of our trials where the quadrotor executes a flip maneuver with accelerations reaching  30.17  $[\text{m}/\text{s}^2]$, 11\% higher than the maximum acceleration recorded in our trials for QSMC, highlighting the increased capability of the controller in executing agile maneuvers.

% \begin{figure}[t]
%     \centering
%     \includegraphics[width=1\linewidth]{Figures/Results/Throwlaunch_QSMC.png}
%     \caption{Quadrotor trajectory in an upside-down throw launch using QSMC}
%     \label{fig:throw_launch_qsmc}
% \end{figure}

\begin{figure}
    \centering
    \includegraphics[scale=0.25, trim={6cm 0cm 4cm 0cm}, clip ]{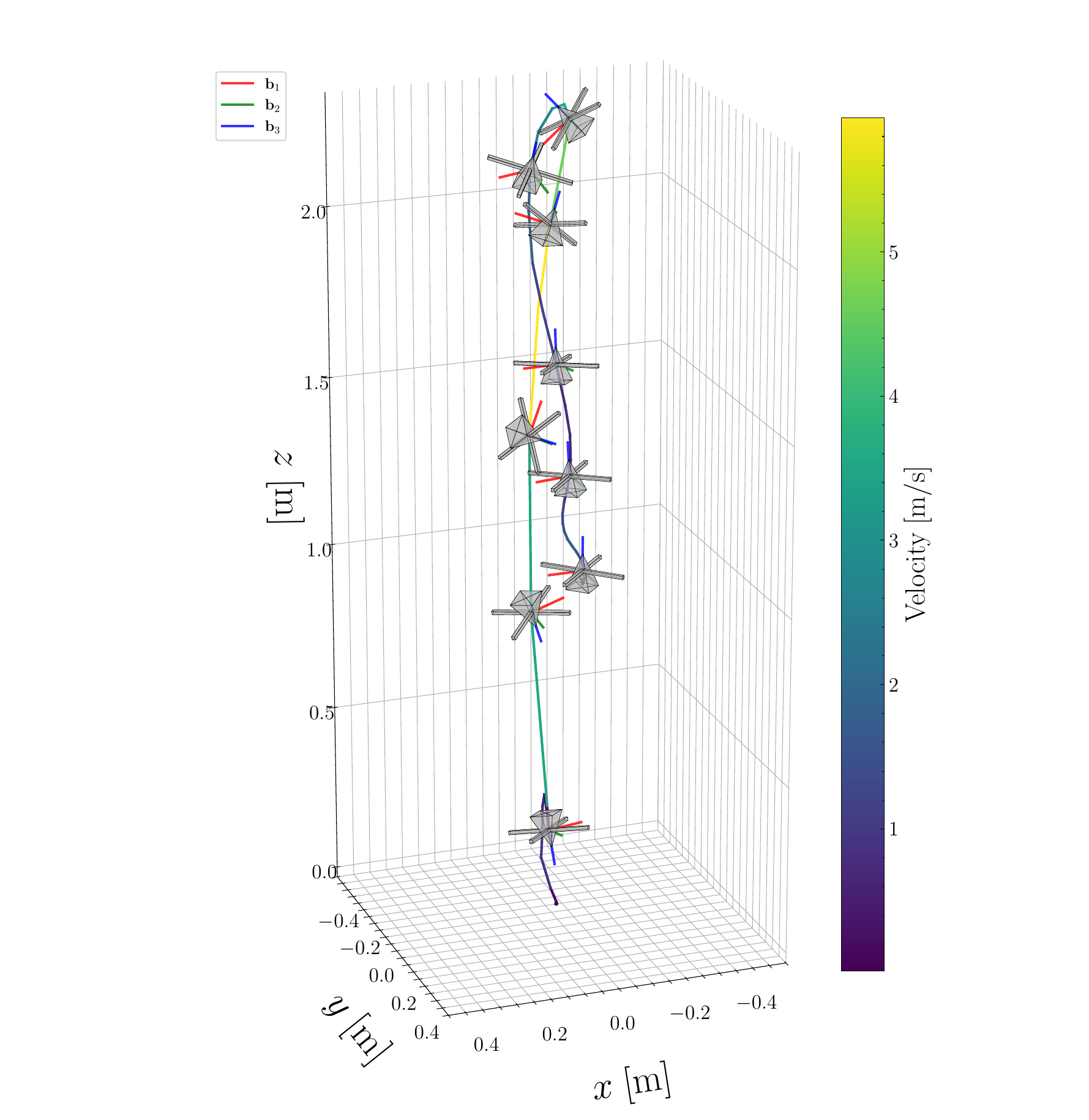}
    \caption{Quadrotor trajectory in a throw launch conducting a flip maneuver using AQSMC}
    \label{fig:throw_launch_aqsmc}
\end{figure}

\section{Conclusion}
In this paper, we presented an adaptive unwinding-free quaternion-based sliding mode controller for 6-DOF quadrotor control.
Our goal was to present a solution that can effectively offer the robustness advantages of SMC for both position and attitude control.
As noted in our literature review, prior sliding mode controllers have suffered from certain limitations, which the proposed QSMC and AQSMC address.
Our comparative hardware experiments verified their performance advantages, and our sensitivity analysis demonstrated that the performance gains offered by the proposed methods are not merely the result of better parameter tuning.

We also highlighted the added benefits of adaptive switching gains in AQSMC over QSMC, showcasing its ability to excel under severe disturbances where even QSMC fails.
The proposed adaptation laws themselves feature a new approach that prevents the excessive growth of switching gains seen in prior adaptive SMC methods, which had limited their practicality.

Beyond theoretical contributions, this paper presented two key strengths in its study design. 
To compare the controllers, we presented a sensitivity analysis on the effect of tuning the controller parameters to reduce the confounding error arising from empirical tuning. 
Furthermore, we designed gimbal experiments that allowed us to evaluate attitude controllers in isolation from the outer-loop position controllers.
The gimbal experiments also highlighted the limitations of Euler-based approaches that have been largely overlooked in the literature.

Given the computational efficiency of AQSMC, which enabled its implementation on the microcontroller of a nano quadrotor, along with its robust performance, we believe our results are promising and open several avenues for future work. Potential directions include integrating advanced SMC formulations such as HOSMC and TSMC, combining the method with optimization- or learning-based approaches, and exploring integrated planning and control frameworks. These advancements could further enhance the utility of AQSMC in complex real-world applications.

% \vspace{-0.5\baselineskip}
\appendices
\section{Control Law of Benchmark Methods}\label{se:appendixA}
% To provide context for the design parameters of the benchmark methods, we provide their control laws in the following. 
For ESMC, we rely on the seminal work of \cite{xu2006sliding}, which has been the basis for many ESMC works in the literature. 
However, the control law in \cite{xu2006sliding} utilizes $\sat(\cdot)$ to reduce chattering.
For a fair comparison with our hyperbolic QSMC law, we replace it with $\tanh(\cdot)$.

For ESMC, the attitude tracking error is $\boldsymbol{\eta}_e = \boldsymbol{\eta} - \boldsymbol{\eta}_d$, and $\boldsymbol{\xi}_e$, $\boldsymbol{\nu}_e$, and $\boldsymbol{\omega}_e$ follow the definitions in \eqref{eq:error_definition}.
The sliding surface for the position controller is the same as \eqref{eq:s_pos}, whereas the sliding surface for the attitude controller is defined as
\begin{equation}
\mathbf{s}_\eta = \boldsymbol{\omega}_e + \Lambda_{\eta} \boldsymbol{\eta}_e.
\end{equation}
The position controller generates thrust using
\begin{equation}\label{eq:ESMC}
% f^\text{ESMC}=\frac{\hat{m}}{\cos{\varphi}\cos{\theta}} \bigg( \ddot{z} + g -\Lambda_{\xi_3}^{\text{ESMC}} {\nu}_{e_3}
% - K_{\xi_{33}}^{\text{ESMC}} \tanh( \frac{s_{\xi_3}}{\phi_{\xi_3}^{\text{ESMC}}}) \bigg)
\begin{array}{ll}
f^\text{ESMC} &= \frac{\hat{m}}{\cos{\eta_1}\cos{\eta_2}} \bigg( \ddot{\xi}_{d_3} + g -\Lambda_{\xi_{33}}^{\text{ESMC}} {\nu}_{e_3} \\ &
- K_{\xi_{33}}^{\text{ESMC}} \tanh( \frac{s_{\xi_3}}{\phi_{\xi_3}^{\text{ESMC}}}) \bigg)
\end{array}
\end{equation}
The position controller also computes $\varphi_d$ and $\theta_d$ as follows 
\begin{equation}
\begin{array}{ll}
    \varphi_d = \arcsin\left(u_x \sin{\psi_d} - u_y \cos{\psi_d}\right),\\
    \theta_d =\arcsin\left(\frac{u_x \cos{\psi_d} +u_y \sin{\psi_d}}{\cos{\varphi_d}}\right),
\end{array}
\end{equation}
where $u_x$ and $u_y$ are two auxiliary variables defined as
% \begin{equation}
% \begin{array}{ll}
% u_x =\frac{\hat{m}}{f^{\text{ESMC}}}\left(\ddot{x}_{d}-\Lambda_{\xi_1}^{\text{ESMC}}{\nu}_{e_1} - K_{\xi_{11}}^{\text{ESMC}}\tanh(\frac{s_{\xi_1}}{\phi_{\xi_1}^{\text{ESMC}}})\right),\\
% u_y =\frac{\hat{m}}{f^{\text{ESMC}}}\left(\ddot{y}_{d}-\Lambda_{\xi_2}^{\text{ESMC}}{\nu}_{e_2} - K_{\xi_{22}}^{\text{ESMC}}\tanh(\frac{s_{\xi_2}}{\phi_{\xi_2}^{\text{ESMC}}})\right).
% \end{array}
% \end{equation}
\begin{equation}
u_i =\frac{\hat{m}}{f^{\text{ESMC}}}\left(\ddot{\xi}_{d_i}-\Lambda_{\xi_{ii}}^{\text{ESMC}}{\nu}_{e_i} - K_{\xi_{ii}}^{\text{ESMC}}\tanh(\frac{s_{\xi_i}}{\phi_{\xi_i}^{\text{ESMC}}})\right),
\end{equation}
where $i\in\{1,2\}$.
Finally, the attitude controller takes the following form
\begin{equation}
\tau_{i}^\text{ESMC} =\hat{J}_{ii}\left(\ddot{\eta}_{d_i} - \Lambda_{\eta_{ii}}\eta_{e_i}- K_{\eta_{ii}}\tanh(\frac{s_{\eta_i}}{\phi_{\eta_i}})-a_{i}\prod^{3}_{j \neq i}\dot{\eta}_j \right),
\end{equation}
where $i \in \{1,2,3\}$, $a_1 = \frac{\hat{J}_{22}-\hat{J}_{33}}{\hat{J}_{11}}$, $a_2 = \frac{\hat{J}_{33}-\hat{J}_{11}}{\hat{J}_{22}}$, and $a_3 = \frac{\hat{J}_{11}-\hat{J}_{22}}{\hat{J}_{33}}$.

For GTC, we adopt the approach presented in \cite{lee2010geometric}.
In this method, the attitude tracking error is $ \mathbf{R}_e = \mathbf{R}_d^T \mathbf{R}$. The rest of the tracking error terms retain their definitions from \eqref{eq:error_definition}.
The position control law comprises
\begin{equation}\label{eq:gtc}
f^{\text{GTC}} =\boldsymbol{\kappa}^\text{GTC} \cdot \mathbf{R}\mathbf{e}_3,
\end{equation}
\begin{equation}
    \boldsymbol{\kappa}^\text{GTC} =\hat{m}\left(\ddot{\boldsymbol{\xi}}_d +g\mathbf{e}_3 -\mathbf{K}_{\xi}^{\text{GTC}}\boldsymbol{\xi}_e- \mathbf{K}_{\nu}\boldsymbol{\nu}_e\right),
\end{equation}
and the attitude control has the following structure  
\begin{equation}\label{eq:gtc_attitude_control_law}
\begin{array}{ll}
\boldsymbol{\tau}^\text{GTC} & = \boldsymbol{\omega} \times \hat{\mathbf{J}}\boldsymbol{\omega} -\hat{\mathbf{J}} \left(\boldsymbol{\omega}^{\times} \mathbf{R}^T\mathbf{R}_d \boldsymbol{\omega}_d -\mathbf{R}^T\mathbf{R}_d\dot{\boldsymbol{\omega}}_d \right)\\ &
    -\mathbf{K}_{\omega}\boldsymbol{\omega}_e- \mathbf{K}_{R}\mathbf{R}_e.
\end{array}
\end{equation}

For QPD, the tracking error definitions are identical to \eqref{eq:error_definition}, and its position controller matches that of QSMC. However, for attitude control, we use the following formulation
\begin{equation}\label{eq:QPD}
\begin{aligned}
\boldsymbol{\tau}^{QPD} =-\mathbf{K}_D \boldsymbol{\omega}_e - \mathbf{K}_P \operatorname{sgn_{+}}(q_{w_e})\Vec{\mathbf{q}}_e,
\end{aligned}
\end{equation}
where we use $\operatorname{sgn_{+}}(\cdot)$ to ensure unwinding-free control law similar to QSMC.
\vspace{-1.0\baselineskip}
\section{Controller Parameter Values}\label{se:appendixB}
Table \ref{tab:control_params} summarizes the parameters used for each controller across the different experimental scenarios.

\begin{table*}[t]
\centering
\caption{Values of design parameters for each controller across experiments}
\label{tab:control_params}
\begin{tabular}{lccc}
\toprule
Parameter & Gimballed Attitude Control - Scenario 1 & Gimballed Attitude Control - Scenario 2 & Lemniscate Trajectory Tracking \\
\midrule
&&\textbf{QSMC}&\\
\midrule
\midrule
$\mathbf{K}_{q}$ & $\diag([679.9,501.6,99.9])$ & $\diag([4502.3,1083.5,121.7])$  &$\diag([400,400,400])$  \\
$\boldsymbol{\Lambda}_{q}$ & $\diag([11.3,9.8,13.3])$ & $\diag([11.62,9.80,8.48])$ & $\diag([8,8,8])$ \\
$\boldsymbol{\phi}_{q}$ & $[1.901,1.818,1.136]^T$ & $[4.878,4.424,4.484]^T$ & $[3.33,3.33,5]^T$ \\
$\mathbf{K}_{\xi}$ & --- &--- & $\diag([4,4,3.5])$  \\
$\boldsymbol{\Lambda}_{\xi}$ & --- &---  &$\diag([3,3,2])$  \\
$\boldsymbol{\phi}_{\xi}$ &  --- & --- & $[1.25,1.25,1.25]^T$ \\
\midrule
% &&\textbf{AQSMC}&\\
% \midrule
% \midrule
% $\mathbf{K}_{q_0}=\mathbf{K}_q^{\text{th}}$ & $\diag([600,400,95])$ & ---  &$\diag([400,400,400])$  \\
% $\boldsymbol{\Lambda}_{q}$ & $\diag([10.0,10.0,12.0])$ & --- & $\diag([8,8,8])$ \\
% $\boldsymbol{\phi}_{q}$ & $[2.0,2.0,5.0]^T$ & --- & $[3.33,3.33,5]^T$ \\
% $\boldsymbol{\Gamma}_{q}$ & $\diag([3.0,3.0,0.8])$ & --- & $\diag([5.0,5.0,2.5])$ \\
% $\boldsymbol{\mu}_{q}$ & $[0.02,0.02,0.01]^T$ & --- & $[0.001,0.001,0.001]^T$ \\
% $\boldsymbol{\epsilon}_{q}$ & $[0.8,0.8,0.5]^T$ & --- & $[0.05,0.05,0.05]^T$ \\
% $\mathbf{K}_{\xi_0}=\mathbf{K}_{\xi}^{\text{th}}$ & --- &--- & $\diag([4,4,3.5])$  \\
% $\boldsymbol{\Lambda}_{\xi}$ & --- &---  &$\diag([3,3,2])$  \\
% $\boldsymbol{\phi}_{\xi}$ &  --- & --- & $[1.25,1.25,1.25]^T$ \\
% $\boldsymbol{\Gamma}_{\xi}$ &  --- & --- & $\diag([0.05,0.05,0.05])$ \\
% $\boldsymbol{\mu}_{\xi}$ &  --- & --- & $[1.0e-6,1.0e-6,1.0e-6]^T$ \\
% $\boldsymbol{\epsilon}_{\xi}$ &  --- & --- & $[0.01,0.01,0.01]^T$ \\
% \midrule
&&\textbf{ESMC}&\\
\midrule
\midrule
$\mathbf{K}_{\eta}$ & $\diag([10,10,10])$ & ---   & $\diag([8,8,8])$   \\
$\boldsymbol{\Lambda}_{\eta}$ & $\diag([7,7,7])$ & ---    &$\diag([4,4,4])$   \\
$\boldsymbol{\phi}_{\eta}$ & $[4,4,2]^T$ &  ---  &$[2,2,2]^T$ \\
$\mathbf{K}_{\xi}^{\text{ESMC}}$ &  --- & ---  &  $\diag([5,5,8])$    \\
$\boldsymbol{\Lambda}_{\xi}^{\text{ESMC}}$ & --- & ---   & $\diag([2,2,1])$  \\
$\boldsymbol{\phi}_{\xi}^{\text{ESMC}}$ &  --- & --- &  $[1.11,1.11,1.11]^T$ \\
\midrule
&&\textbf{GTC}&\\
\midrule
\midrule
$\mathbf{K}_{R}$ & $\diag([752.1,834.7,156.0])\hat{\mathbf{J}}$ &$\diag([794.0,826.3,150.0])\hat{\mathbf{J}}$  &$\diag([414.43,345.56,246.42])\hat{\mathbf{J}}$ \\
$\mathbf{K}_{\omega}$ & $\diag([202.1,209.9,222.1])\hat{\mathbf{J}}$ &$\diag([215.8,232.8,237.3])\hat{\mathbf{J}}$  &$\diag([59.08,69.54,63.90])\hat{\mathbf{J}}$  \\
$\mathbf{K}_{\xi}^\text{GTC}$ &  --- & ---   &$\diag([7.70,6.91,7.37])$ \\
$\mathbf{K}_{\nu}$ &  --- & ---  &$\diag([2.34,1.67,4.18])$  \\
\midrule
&&\textbf{QPD}&\\
\midrule
\midrule
$\mathbf{K}_{P}$  & $\diag([2251.2,2166.9,729.9])\hat{\mathbf{J}}$ &$\diag([1926.1,1644.3,1003.9])\hat{\mathbf{J}}$  &$\diag([1045.54,881.97,678.06])\hat{\mathbf{J}}$  \\
$\mathbf{K}_{D}$ & $\diag([232.3,196.0,127.1])\hat{\mathbf{J}}$ &$\diag([366.0,392.1,138.2])\hat{\mathbf{J}}$  &$\diag([97.97,106.79,118.29])\hat{\mathbf{J}}$  \\
$\mathbf{K}_{\xi}^{\text{QPD}}$ &  --- & ---  &$\diag([4.54,3.52,3.48])$ \\
$\boldsymbol{\Lambda}_{\xi}^{\text{QPD}}$ &  --- & ---  &$\diag([2.27,2.01,1.32])$  \\
$\boldsymbol{\phi}_{\xi}^{\text{QPD}}$ &  --- & --- & $[1.102,1.277,1.131]^T$ \\
\bottomrule
\end{tabular}
\end{table*}

\bibliographystyle{IEEEtran}
\bibliography{References}

\begin{thebibliography}{10}
\providecommand{\url}[1]{#1}
\csname url@samestyle\endcsname
\providecommand{\newblock}{\relax}
\providecommand{\bibinfo}[2]{#2}
\providecommand{\BIBentrySTDinterwordspacing}{\spaceskip=0pt\relax}
\providecommand{\BIBentryALTinterwordstretchfactor}{4}
\providecommand{\BIBentryALTinterwordspacing}{\spaceskip=\fontdimen2\font plus
\BIBentryALTinterwordstretchfactor\fontdimen3\font minus \fontdimen4\font\relax}
\providecommand{\BIBforeignlanguage}[2]{{%
\expandafter\ifx\csname l@#1\endcsname\relax
\typeout{** WARNING: IEEEtran.bst: No hyphenation pattern has been}%
\typeout{** loaded for the language `#1'. Using the pattern for}%
\typeout{** the default language instead.}%
\else
\language=\csname l@#1\endcsname
\fi
#2}}
\providecommand{\BIBdecl}{\relax}
\BIBdecl

\bibitem{yazdanshenas2024quaternion}
A.~Yazdanshenas and R.~Faieghi, ``Quaternion-based sliding mode control for six degrees of freedom flight control of quadrotors,'' in \emph{Proceedings of the IEEE/RSJ International Conference on Intelligent Robots and Systems (IROS)}.\hskip 1em plus 0.5em minus 0.4em\relax IEEE, 2024.

\bibitem{nan2022nonlinear}
F.~Nan, S.~Sun, P.~Foehn, and D.~Scaramuzza, ``Nonlinear mpc for quadrotor fault-tolerant control,'' \emph{IEEE Robotics and Automation Letters}, vol.~7, no.~2, pp. 5047--5054, 2022.

\bibitem{romero2022model}
A.~Romero, S.~Sun, P.~Foehn, and D.~Scaramuzza, ``Model predictive contouring control for time-optimal quadrotor flight,'' \emph{IEEE Transactions on Robotics}, vol.~38, no.~6, pp. 3340--3356, 2022.

\bibitem{izadi2024multi}
M.~Izadi, Z.~Shayan, and R.~Faieghi, ``Multi-model predictive attitude control of quadrotors,'' \emph{arXiv preprint arXiv:2406.15610}, 2024.

\bibitem{sun2022comparative}
S.~Sun, A.~Romero, P.~Foehn, E.~Kaufmann, and D.~Scaramuzza, ``A comparative study of nonlinear mpc and differential-flatness-based control for quadrotor agile flight,'' 2022.

\bibitem{bhattacharjee2020robust}
D.~Bhattacharjee and K.~Subbarao, ``Robust control strategy for quadcopters using sliding mode control and model predictive control,'' in \emph{AIAA Scitech 2020 Forum}, 2020, p. 2071.

\bibitem{li2022enhanced}
B.~Li and Y.~Wang, ``An enhanced model predictive controller for quadrotor attitude quick adjustment with input constraints and disturbances,'' \emph{International Journal of Control, Automation and Systems}, vol.~20, no.~2, pp. 648--659, 2022.

\bibitem{michel2019design}
N.~Michel, S.~Bertrand, S.~Olaru, G.~Valmorbida, and D.~Dumur, ``Design and flight experiments of a tube-based model predictive controller for the ar. drone 2.0 quadrotor,'' \emph{IFAC-PapersOnLine}, vol.~52, no.~22, pp. 112--117, 2019.

\bibitem{xue2024output}
R.~Xue, L.~Dai, P.~Wang, Z.~Sun, and Y.~Xia, ``Output feedback stochastic mpc for tracking control of quadrotors with disturbances,'' \emph{IET Control Theory \& Applications}, vol.~18, no.~5, pp. 566--580, 2024.

\bibitem{kamel2015fast}
M.~Kamel, K.~Alexis, M.~Achtelik, and R.~Siegwart, ``Fast nonlinear model predictive control for multicopter attitude tracking on so (3),'' in \emph{2015 IEEE conference on control applications (CCA)}.\hskip 1em plus 0.5em minus 0.4em\relax IEEE, 2015, pp. 1160--1166.

\bibitem{tal2020accurate}
E.~Tal and S.~Karaman, ``Accurate tracking of aggressive quadrotor trajectories using incremental nonlinear dynamic inversion and differential flatness,'' \emph{IEEE Transactions on Control Systems Technology}, vol.~29, no.~3, pp. 1203--1218, 2020.

\bibitem{achtelik2013inversion}
M.~W. Achtelik, S.~Lynen, M.~Chli, and R.~Siegwart, ``Inversion based direct position control and trajectory following for micro aerial vehicles,'' in \emph{2013 IEEE/RSJ International Conference on Intelligent Robots and Systems}.\hskip 1em plus 0.5em minus 0.4em\relax IEEE, 2013, pp. 2933--2939.

\bibitem{bouabdallah2007full}
S.~Bouabdallah and R.~Siegwart, ``Full control of a quadrotor,'' in \emph{2007 IEEE/RSJ international conference on intelligent robots and systems}.\hskip 1em plus 0.5em minus 0.4em\relax Ieee, 2007, pp. 153--158.

\bibitem{song2023reaching}
Y.~Song, A.~Romero, M.~M{\"u}ller, V.~Koltun, and D.~Scaramuzza, ``Reaching the limit in autonomous racing: Optimal control versus reinforcement learning,'' \emph{Science Robotics}, vol.~8, no.~82, p. eadg1462, 2023.

\bibitem{romero2024actor}
A.~Romero, Y.~Song, and D.~Scaramuzza, ``Actor-critic model predictive control,'' in \emph{2024 IEEE International Conference on Robotics and Automation (ICRA)}.\hskip 1em plus 0.5em minus 0.4em\relax IEEE, 2024, pp. 14\,777--14\,784.

\bibitem{yogi2024neural}
S.~C. Yogi, L.~Behera, and T.~Tripathy, ``Neural-fxsmc: A robust adaptive neural fixed-time sliding mode control for quadrotors with unknown uncertainties,'' \emph{IEEE Robotics and Automation Letters}, 2024.

\bibitem{salzmann2023real}
T.~Salzmann, E.~Kaufmann, J.~Arrizabalaga, M.~Pavone, D.~Scaramuzza, and M.~Ryll, ``Real-time neural mpc: Deep learning model predictive control for quadrotors and agile robotic platforms,'' \emph{IEEE Robotics and Automation Letters}, vol.~8, no.~4, pp. 2397--2404, 2023.

\bibitem{annaswamy2023integration}
A.~M. Annaswamy, A.~Guha, Y.~Cui, S.~Tang, P.~A. Fisher, and J.~E. Gaudio, ``Integration of adaptive control and reinforcement learning for real-time control and learning,'' \emph{IEEE Transactions on Automatic Control}, vol.~68, no.~12, pp. 7740--7755, 2023.

\bibitem{voos2009nonlinear}
H.~Voos, ``Nonlinear control of a quadrotor micro-uav using feedback-linearization,'' in \emph{2009 IEEE International Conference on Mechatronics}.\hskip 1em plus 0.5em minus 0.4em\relax IEEE, 2009, pp. 1--6.

\bibitem{xu2006sliding}
R.~Xu and U.~Ozguner, ``Sliding mode control of a quadrotor helicopter,'' in \emph{Proceedings of the 45th IEEE Conference on Decision and Control}.\hskip 1em plus 0.5em minus 0.4em\relax IEEE, 2006, pp. 4957--4962.

\bibitem{madani2006backstepping}
T.~Madani and A.~Benallegue, ``Backstepping control for a quadrotor helicopter,'' in \emph{2006 IEEE/RSJ International Conference on Intelligent Robots and Systems}.\hskip 1em plus 0.5em minus 0.4em\relax IEEE, 2006, pp. 3255--3260.

\bibitem{nicol2011robust}
C.~Nicol, C.~Macnab, and A.~Ramirez-Serrano, ``Robust adaptive control of a quadrotor helicopter,'' \emph{Mechatronics}, vol.~21, no.~6, pp. 927--938, 2011.

\bibitem{souza2014passivity}
C.~Souza, G.~V. Raffo, and E.~B. Castelan, ``Passivity based control of a quadrotor uav,'' \emph{IFAC Proceedings Volumes}, vol.~47, no.~3, pp. 3196--3201, 2014.

\bibitem{lee2010geometric}
T.~Lee, M.~Leok, and N.~H. McClamroch, ``Geometric tracking control of a quadrotor uav on se (3),'' in \emph{49th IEEE Conf. Decis. Control (CDC)}.\hskip 1em plus 0.5em minus 0.4em\relax IEEE, 2010, pp. 5420--5425.

\bibitem{garcia2020robust}
O.~Garcia, E.~G. Rojo-Rodriguez, A.~Sanchez, D.~Saucedo, and A.-J. Munoz-Vazquez, ``Robust geometric navigation of a quadrotor uav on se (3),'' \emph{Robotica}, vol.~38, no.~6, pp. 1019--1040, 2020.

\bibitem{lee2013backstepping}
H.~Lee, S.~Kim, T.~Ryan, and H.~J. Kim, ``Backstepping control on se (3) of a micro quadrotor for stable trajectory tracking,'' in \emph{2013 IEEE international conference on systems, man, and cybernetics}.\hskip 1em plus 0.5em minus 0.4em\relax IEEE, 2013, pp. 4522--4527.

\bibitem{lee2012robust}
T.~Lee, ``Robust adaptive attitude tracking on so(3) with an application to a quadrotor uav,'' \emph{IEEE Transactions on Control Systems Technology}, vol.~21, no.~5, pp. 1924--1930, 2012.

\bibitem{lee2013nonlinear}
T.~Lee, M.~Leok, and N.~H. McClamroch, ``Nonlinear robust tracking control of a quadrotor uav on se (3),'' \emph{Asian journal of control}, vol.~15, no.~2, pp. 391--408, 2013.

\bibitem{parra2012toward}
V.~Parra-Vega, A.~Sanchez, and C.~Izaguirre, ``Toward force control of a quadrotor uav in se (3),'' in \emph{2012 IEEE 51st IEEE Conference on Decision and Control (CDC)}.\hskip 1em plus 0.5em minus 0.4em\relax IEEE, 2012, pp. 1802--1809.

\bibitem{gong2020adaptive}
K.~Gong, Y.~Liao, and Y.~Wang, ``Adaptive fixed-time terminal sliding mode control on se (3) for coupled spacecraft tracking maneuver,'' \emph{Int. J. Aerosp. Eng}, vol. 2020, pp. 1--15, 2020.

\bibitem{ren2023adaptive}
J.~Ren, S.~Tang, and T.~Chen, ``Adaptive sliding mode control of spacecraft attitude-orbit dynamics on se (3),'' \emph{Advances in Space Research}, vol.~71, no.~1, pp. 525--538, 2023.

\bibitem{lee2012exponential}
T.~Lee, ``Exponential stability of an attitude tracking control system on so (3) for large-angle rotational maneuvers,'' \emph{Systems \& Control Letters}, vol.~61, no.~1, pp. 231--237, 2012.

\bibitem{teng2022lie}
S.~Teng, W.~Clark, A.~Bloch, R.~Vasudevan, and M.~Ghaffari, ``Lie algebraic cost function design for control on lie groups,'' in \emph{2022 IEEE 61st Conf. Decis. Control (CDC)}.\hskip 1em plus 0.5em minus 0.4em\relax IEEE, 2022, pp. 1867--1874.

\bibitem{lopez2020sliding}
B.~T. Lopez and J.-J.~E. Slotine, ``Sliding on manifolds: Geometric attitude control with quaternions,'' 2020.

\bibitem{bhat2000topological}
S.~P. Bhat and D.~S. Bernstein, ``A topological obstruction to continuous global stabilization of rotational motion and the unwinding phenomenon,'' \emph{Systems \& control letters}, vol.~39, no.~1, pp. 63--70, 2000.

\bibitem{d2024efficient}
S.~D'Angelo, F.~Pagano, F.~Longobardi, F.~Ruggiero, and V.~Lippiello, ``Efficient development of model-based controllers in px4 firmware: A template-based customization approach,'' in \emph{2024 International Conference on Unmanned Aircraft Systems (ICUAS)}.\hskip 1em plus 0.5em minus 0.4em\relax IEEE, 2024, pp. 1155--1162.

\bibitem{crazyflie_lee_controller}
\BIBentryALTinterwordspacing
{Bitcraze Development Team}, \emph{Crazyflie Firmware Documentation: Sensor to Control - Controllers}, 2024, accessed: October 13, 2024. [Online]. Available: \url{https://www.bitcraze.io/documentation/repository/crazyflie-firmware/master/functional-areas/sensor-to-control/controllers/}
\BIBentrySTDinterwordspacing

\bibitem{zheng2014second}
E.-H. Zheng, J.-J. Xiong, and J.-L. Luo, ``Second order sliding mode control for a quadrotor uav,'' \emph{ISA transactions}, vol.~53, no.~4, pp. 1350--1356, 2014.

\bibitem{xiong2017global}
J.-J. Xiong and G.-B. Zhang, ``Global fast dynamic terminal sliding mode control for a quadrotor uav,'' \emph{ISA transactions}, vol.~66, pp. 233--240, 2017.

\bibitem{hou2020nonsingular}
Z.~Hou, P.~Lu, and Z.~Tu, ``Nonsingular terminal sliding mode control for a quadrotor uav with a total rotor failure,'' \emph{Aerospace Science and Technology}, vol.~98, p. 105716, 2020.

\bibitem{hassani2021robust}
H.~Hassani, A.~Mansouri, and A.~Ahaitouf, ``Robust autonomous flight for quadrotor uav based on adaptive nonsingular fast terminal sliding mode control,'' \emph{International Journal of Dynamics and Control}, vol.~9, pp. 619--635, 2021.

\bibitem{mechali2021observer}
O.~Mechali, L.~Xu, Y.~Huang, M.~Shi, and X.~Xie, ``Observer-based fixed-time continuous nonsingular terminal sliding mode control of quadrotor aircraft under uncertainties and disturbances for robust trajectory tracking: Theory and experiment,'' \emph{Control Engineering Practice}, vol. 111, p. 104806, 2021.

\bibitem{wang2020fixed}
J.~Wang, X.~Ma, G.~Zhang, Y.~Zhang, and Q.~Miao, ``Fixed-time terminal sliding mode control for quadrotor aircraft,'' in \emph{Proceedings of the 11th International Conference on Modelling, Identification and Control (ICMIC2019)}.\hskip 1em plus 0.5em minus 0.4em\relax Springer, 2020, pp. 413--421.

\bibitem{ai2019fixed}
X.~Ai and J.~Yu, ``Fixed-time trajectory tracking for a quadrotor with external disturbances: A flatness-based sliding mode control approach,'' \emph{Aerospace Science and Technology}, vol.~89, pp. 58--76, 2019.

\bibitem{yu2021novel}
L.~Yu, G.~He, X.~Wang, and L.~Shen, ``A novel fixed-time sliding mode control of quadrotor with experiments and comparisons,'' \emph{IEEE Control Systems Letters}, vol.~6, pp. 770--775, 2021.

\bibitem{noordin2021sliding}
A.~Noordin, M.~A.~M. Basri, and Z.~Mohamed, ``Sliding mode control with tanh function for quadrotor uav altitude and attitude stabilization,'' in \emph{Intelligent Manufacturing and Mechatronics: Proceedings of SympoSIMM 2020}.\hskip 1em plus 0.5em minus 0.4em\relax Springer, 2021, pp. 471--491.

\bibitem{nguyen2021adaptive}
N.~P. Nguyen, N.~X. Mung, H.~L. N.~N. Thanh, T.~T. Huynh, N.~T. Lam, and S.~K. Hong, ``Adaptive sliding mode control for attitude and altitude system of a quadcopter uav via neural network,'' \emph{IEEE Access}, vol.~9, pp. 40\,076--40\,085, 2021.

\bibitem{noordin2022position}
A.~Noordin, M.~A. Mohd~Basri, Z.~Mohamed, and I.~Mat~Lazim, ``Position and attitude control of quadrotor mav using sliding mode control with tanh function,'' in \emph{Enabling Industry 4.0 through Advances in Mechatronics: Selected Articles from iM3F 2021, Malaysia}.\hskip 1em plus 0.5em minus 0.4em\relax Springer, 2022, pp. 193--204.

\bibitem{zare2022quadrotor}
M.~Zare, F.~Pazooki, and S.~E. Haghighi, ``Quadrotor uav position and altitude tracking using an optimized fuzzy-sliding mode control,'' \emph{IETE Journal of Research}, vol.~68, no.~6, pp. 4406--4420, 2022.

\bibitem{jing2019quadrotor}
Q.~Jing, Z.~Chang, H.~Chu, Y.~Shao, and X.~Zhang, ``Quadrotor attitude control based on fuzzy sliding mode control theory,'' in \emph{2019 Chinese Control Conference (CCC)}.\hskip 1em plus 0.5em minus 0.4em\relax IEEE, 2019, pp. 8360--8364.

\bibitem{munoz2017second}
F.~Mu{\~n}oz, I.~Gonz{\'a}lez-Hern{\'a}ndez, S.~Salazar, E.~S. Espinoza, and R.~Lozano, ``Second order sliding mode controllers for altitude control of a quadrotor uas: Real-time implementation in outdoor environments,'' \emph{Neurocomputing}, vol. 233, pp. 61--71, 2017.

\bibitem{chandra2022higher}
A.~Chandra and P.~P. Lal, ``Higher order sliding mode controller for a quadrotor uav with a suspended load,'' \emph{IFAC-PapersOnLine}, vol.~55, no.~1, pp. 610--615, 2022.

\bibitem{bouadi2011adaptive}
H.~Bouadi, S.~S. Cunha, A.~Drouin, and F.~Mora-Camino, ``Adaptive sliding mode control for quadrotor attitude stabilization and altitude tracking,'' in \emph{2011 IEEE 12th international symposium on computational intelligence and informatics (CINTI)}.\hskip 1em plus 0.5em minus 0.4em\relax IEEE, 2011, pp. 449--455.

\bibitem{huang2019robust}
T.~Huang, D.~Huang, Z.~Wang, and A.~Shah, ``Robust tracking control of a quadrotor uav based on adaptive sliding mode controller,'' \emph{Complexity}, vol. 2019, no.~1, p. 7931632, 2019.

\bibitem{mofid2018adaptive}
O.~Mofid and S.~Mobayen, ``Adaptive sliding mode control for finite-time stability of quad-rotor uavs with parametric uncertainties,'' \emph{ISA transactions}, vol.~72, pp. 1--14, 2018.

\bibitem{nekoukar2021robust}
V.~Nekoukar and N.~M. Dehkordi, ``Robust path tracking of a quadrotor using adaptive fuzzy terminal sliding mode control,'' \emph{Control Engineering Practice}, vol. 110, p. 104763, 2021.

\bibitem{nadda2018adaptive}
S.~Nadda and A.~Swarup, ``On adaptive sliding mode control for improved quadrotor tracking,'' \emph{Journal of Vibration and Control}, vol.~24, no.~14, pp. 3219--3230, 2018.

\bibitem{lian2021adaptive}
S.~Lian, W.~Meng, Z.~Lin, K.~Shao, J.~Zheng, H.~Li, and R.~Lu, ``Adaptive attitude control of a quadrotor using fast nonsingular terminal sliding mode,'' \emph{IEEE Transactions on Industrial Electronics}, vol.~69, no.~2, pp. 1597--1607, 2021.

\bibitem{thanh2018quadcopter}
H.~L. N.~N. Thanh and S.~K. Hong, ``Quadcopter robust adaptive second order sliding mode control based on pid sliding surface,'' \emph{IEEE Access}, vol.~6, pp. 66\,850--66\,860, 2018.

\bibitem{rios2018continuous}
H.~R{\'\i}os, R.~Falc{\'o}n, O.~A. Gonz{\'a}lez, and A.~Dzul, ``Continuous sliding-mode control strategies for quadrotor robust tracking: Real-time application,'' \emph{IEEE Transactions on Industrial Electronics}, vol.~66, no.~2, pp. 1264--1272, 2018.

\bibitem{shao2021adaptive}
X.~Shao, G.~Sun, W.~Yao, J.~Liu, and L.~Wu, ``Adaptive sliding mode control for quadrotor uavs with input saturation,'' \emph{IEEE/ASME Transactions on Mechatronics}, vol.~27, no.~3, pp. 1498--1509, 2021.

\bibitem{sanchez2013time}
A.~Sanchez, V.~Parra-Vega, O.~Garcia, F.~Ruiz-Sanchez, and L.~Ramos-Velasco, ``Time-parametrization control of quadrotors with a robust quaternion-based sliding mode controller for aggressive maneuvering,'' in \emph{2013 Eur. Control Conf. (ECC)}.\hskip 1em plus 0.5em minus 0.4em\relax IEEE, 2013, pp. 3876--3881.

\bibitem{serrano2023terminal}
F.~Serrano, O.~Castillo, M.~Alassafi, F.~Alsaadi, and A.~Ahmad, ``Terminal sliding mode attitude-position quaternion based control of quadrotor unmanned aerial vehicle,'' \emph{Advances in Space Research}, vol.~71, no.~9, pp. 3855--3867, 2023.

\bibitem{arellano2015quaternion}
C.~A. Arellano-Muro, B.~Castillo-Toledo, A.~G. Loukianov, L.~F. Luque-Vega, and L.~E. Gonz{\'a}lez-Jim{\'e}nez, ``Quaternion-based trajectory tracking robust control for a quadrotor,'' in \emph{2015 10th Syst. Syst. Eng. Conf. (SoSE)}.\hskip 1em plus 0.5em minus 0.4em\relax IEEE, 2015, pp. 386--391.

\bibitem{abaunza2019quadrotor}
H.~Abaunza and P.~Castillo, ``Quadrotor aggressive deployment, using a quaternion-based spherical chattering-free sliding-mode controller,'' \emph{IEEE Trans. Aerosp. Electron. Syst.}, vol.~56, no.~3, pp. 1979--1991, 2019.

\bibitem{clarke2008nonsmooth}
F.~H. Clarke, Y.~S. Ledyaev, R.~J. Stern, and P.~R. Wolenski, \emph{Nonsmooth analysis and control theory}.\hskip 1em plus 0.5em minus 0.4em\relax Springer Science \& Business Media, 2008, vol. 178.

\bibitem{shevitz1994lyapunov}
D.~Shevitz and B.~Paden, ``Lyapunov stability theory of nonsmooth systems,'' \emph{IEEE Transactions on automatic control}, vol.~39, no.~9, pp. 1910--1914, 1994.

\bibitem{bacciotti1999stability}
A.~Bacciotti and F.~Ceragioli, ``Stability and stabilization of discontinuous systems and nonsmooth lyapunov functions,'' \emph{ESAIM: Control, Optimisation and Calculus of Variations}, vol.~4, pp. 361--376, 1999.

\bibitem{khalil2002nonlinear}
H.~K. Khalil and J.~W. Grizzle, \emph{Nonlinear systems}.\hskip 1em plus 0.5em minus 0.4em\relax Prentice hall Upper Saddle River, NJ, 2002, vol.~3.

\bibitem{stevens2015aircraft}
B.~L. Stevens, F.~L. Lewis, and E.~N. Johnson, \emph{Aircraft control and simulation: dynamics, controls design, and autonomous systems}.\hskip 1em plus 0.5em minus 0.4em\relax John Wiley \& Sons, 2015.

\bibitem{richter2016polynomial}
C.~Richter, A.~Bry, and N.~Roy, ``Polynomial trajectory planning for aggressive quadrotor flight in dense indoor environments,'' in \emph{Robotics Research: The 16th International Symposium ISRR}.\hskip 1em plus 0.5em minus 0.4em\relax Springer, 2016, pp. 649--666.

\end{thebibliography}
\vspace{-2.5\baselineskip}
\begin{IEEEbiography}[{\includegraphics[width=1in,height=1.25in,clip,keepaspectratio]{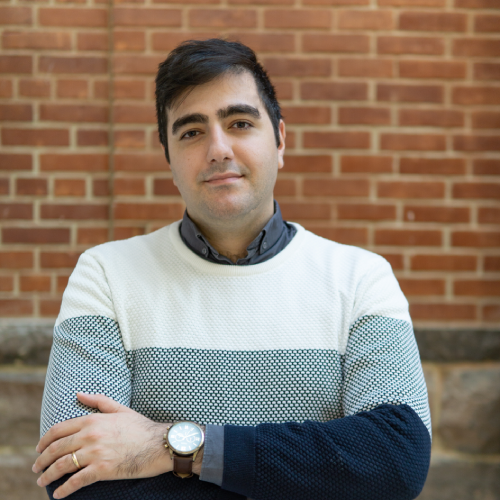}}]{Amin Yazdanshenas}
received the M.Sc. degree in aerospace engineering from Amirkabir University of Technology, Tehran, Iran, in 2022.  
He is currently pursuing the Ph.D. degree in aerospace engineering with the Autonomous Vehicles Laboratory (AVL) at Toronto Metropolitan University, Toronto, Canada. His research focuses on designing advanced control strategies, including sliding mode control, model predictive control, and adaptive control, for aerial and space robotics. 
His work aims to enhance the stability, precision, and performance of autonomous systems operating in dynamic and complex environments.
\end{IEEEbiography}
\vspace{-2\baselineskip}
\begin{IEEEbiography}[{\includegraphics[width=1in,height=1.25in,clip,keepaspectratio]{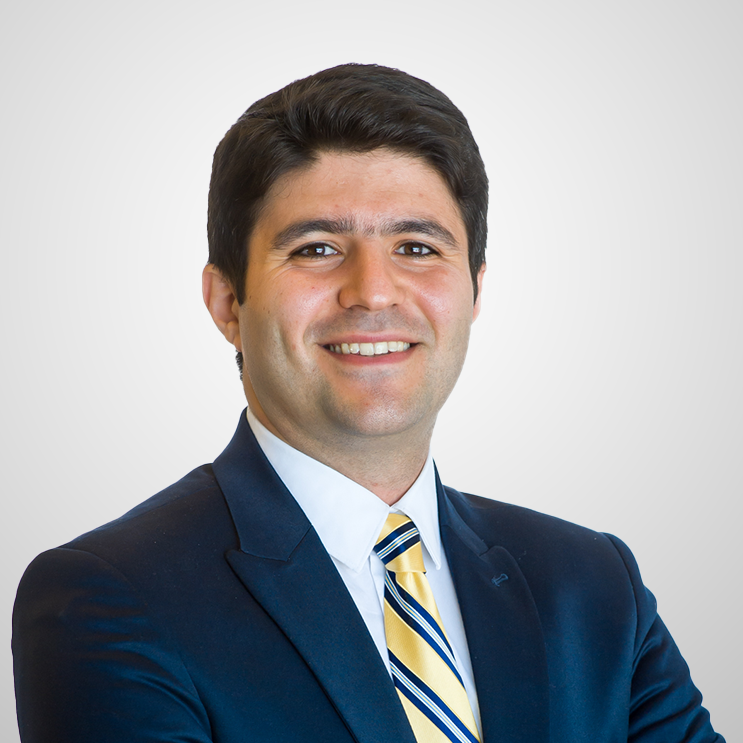}}]{Reza Faieghi}
received his Ph.D. degree from the University of Western Ontario, London, ON, Canada, in 2018. From 2019 to 2020, he was a postdoctoral fellow at the University of Toronto. In 2020, he joined the Department of Aerospace Engineering at Toronto Metropolitan University, where he founded and currently directs the Autonomous Vehicles Laboratory (AVL). His research interests include a broad range of topics in motion planning and control for robotic systems.
\end{IEEEbiography}
\appendices
% \balance
\end{document}